\documentclass[12pt,3p,times]{elsarticle}

\usepackage{setspace}
\usepackage{graphicx}
\usepackage{subfigure}
\usepackage{multirow}
\usepackage[table,xcdraw]{xcolor}
\usepackage{amsmath,amssymb,amsfonts}
\usepackage{array}
\usepackage[linesnumbered,ruled,commentsnumbered]{algorithm2e}
\usepackage[colorlinks,linkcolor=red]{hyperref}
\usepackage{amsthm}
\usepackage{soul}
\usepackage{makecell}
\usepackage{bm}
\newtheorem{lemma}{Lemma}[section]
\newtheorem{theorem}{Theorem}[section]

\newcolumntype{H}{>{\setbox0=\hbox\bgroup}c<{\egroup}@{}}

\journal{Pattern Recognition}

\begin{document}
\doublespacing
\begin{frontmatter}

    % 需要定一个题目
	\title{ Quaternion Nuclear Norm minus Frobenius Norm Minimization for Color Image Reconstruction}

    \author[address1]{Yu Guo}\ead{yuguomath@aliyun.com}
    \author[address1]{Guoqing Chen}\ead{cgq@imu.edu.cn}
    \author[address2]{Tieyong~Zeng}\ead{zeng@math.cuhk.edu.hk}
    \author[address1]{Qiyu~Jin\corref{cor1}}
    \cortext[cor1]{Corresponding author.}
    \ead{qyjin2015@aliyun.com}
    \author[address3]{Michael Kwok-Po Ng}\ead{michael-ng@hkbu.edu.hk}
    
    \address[address1]{School of Mathematical Science, Inner Mongolia University, Hohhot, China }
    \address[address2]{Department of Mathematics, The Chinese University of Hong Kong, Satin, Hong Kong}
    \address[address3]{Department of Mathematics, Hong Kong Baptist University, Kowloon Tong, Hong Kong}

	\begin{abstract}
    Color image restoration methods typically represent images as vectors in Euclidean space or combinations of three monochrome channels. However, they often overlook the correlation between these channels, leading to color distortion and artifacts in the reconstructed image. To address this, we present Quaternion Nuclear Norm Minus Frobenius Norm Minimization (QNMF), a novel approach for color image reconstruction. QNMF utilizes quaternion algebra to capture the relationships among RGB channels comprehensively. By employing a regularization technique that involves nuclear norm minus Frobenius norm, QNMF approximates the underlying low-rank structure of quaternion-encoded color images. Theoretical proofs are provided to ensure the method's mathematical integrity. Demonstrating versatility and efficacy, the QNMF regularizer excels in various color low-level vision tasks, including denoising, deblurring, inpainting, and random impulse noise removal, achieving state-of-the-art results.
	\end{abstract}

	\begin{keyword}
		Quaternion, low rank, color image, low-level vision, nuclear norm, Frobenius norm
	\end{keyword}

\end{frontmatter}

\section{Introduction}

Color images find applications in various fields including video processing, computer vision, medical imaging, remote sensing, and multimedia processing. However, during generation, transmission, and storage, images unavoidably degrade due to factors like noise, blurring, downsampling, and other forms of corruption. This degradation can be represented by the equation:
\begin{equation}
\mathbf{Y}=\mathbf{A}\mathbf{X}+\mathbf{N},
\label{degraded}
\end{equation}
where $\mathbf{X}$ represents the original image, $\mathbf{Y}$ represents the observed degraded image, $\mathbf{A}$ stands for a linear degradation operator, such as a blurring kernel, and $\mathbf{N}$ denotes additive independent Gaussian white noise with a zero mean. Eq. (\ref{degraded}) poses a classical inverse problem due to its ill-posed nature, presenting significant challenges in data science, particularly in the realm of color image restoration methods.

Various research efforts have focused on solving image reconstruction problems using methods such as total variation (TV) \cite{jia2019color}, non-local self-similarity (NSS) \cite{DF07BM3D,dabov2007color}, dictionary learning \cite{huang2021quaternion,zha2020image}, low rank \cite{gu2014weighted,guo2022gaussian}, tensor \cite{jiang2023robust,yang2020low}, and deep learning \cite{Zhang2017DnCNN,guo2021fast}. While these algorithms have shown significant success with grayscale images, their application to color images has posed significant challenges. Grayscale images have a single channel, while color images have three channels: R, G, and B, exhibiting color correlation. Applying grayscale algorithms independently to each channel disrupts this correlation, leading to color artifacts and distortions in the final restored result.

Two successful strategies have been proposed to address the challenge of channel correlation in color images. The first strategy involves converting color images from the standard RGB space to a less correlated color space, such as YCbCr or HSV, and then independently denoising each channel of the converted image. Pioneering work in this approach extended block matching and 3D filtering (BM3D) \cite{DF07BM3D} to color block matching and 3D filtering (CBM3D) \cite{dabov2007color}. CBM3D converts an sRGB image to a luminance-chrominance space, performs block matching in the luminance space, and then filters each channel independently. This approach has been followed by a substantial body of work \cite{hurault2018epll}. In addition to the YCbCr space, other researchers have explored the opponent transformation \cite{hurault2018epll}. Recently, Jia et al. \cite{jia2019color} proposed the saturation-value total variation (SV-TV) for the HSV color space. These studies effectively circumvent the challenges posed by channel correlation through color space conversion. However, this conversion leads to a complex noise structure and can even result in the loss of noise independence, transforming it into correlated noise \cite{guo2023best}. Furthermore, the conversion of color space does not effectively exploit the correlation between colors.

An alternative solution considers coupling between channels. This can be achieved by vectorizing the three channels, fully accounting for cross-channel information. Kong et al. \cite{kong2019color} have further extended this approach by transforming color image patches into vectors, treating them as a unified entity. These methods allow for simultaneous processing of all three channels, unlike independent processing. They can also be generalized to a tensor form \cite{jiang2023robust,miao2020low}, interpreting the color image as a multidimensional tensor with dimensions corresponding to the spatial and spectral channels. Xu et al. \cite{xu2017multi} noted that noise intensity is not always uniform across the three channels of a color image. To address this, they proposed a diagonal form of the weight matrix to model noise differences between the channels, an idea that was subsequently expanded upon \cite{shan2023multi}. However, this strategy reverts to independent channel processing when the noise intensity approaches uniformity, thereby losing the capacity for information exchange between channels.

Additionally, data-driven methods have demonstrated exceptional performance in image reconstruction. Deep learning, unlike traditional methods, primarily focuses on learning the mapping between degraded images and their corresponding clean images using pairwise data. Consequently, these methods circumvent the challenge of extending traditional algorithms from grayscale to color images. In 2016, Zhang et al. proposed DnCNN \cite{Zhang2017DnCNN}, a creative approach based on residual learning that significantly enhanced reconstruction performance. Since then, advancements in the deep learning community \cite{tian2023multi} have consistently surpassed state-of-the-art results for image reconstruction. Currently, the most advanced image reconstruction algorithms are based on Transformer \cite{shamsolmoali2024distance} and diffusion models \cite{xia2023diffir}. Despite the excellent performance of these deep learning algorithms, their computational complexity and high computational cost pose challenges for real-world applications.

The main challenge in color image restoration is effectively utilizing inter-channel correlation to minimize color distortion and artifacts while preserving image structure. Given the human visual system's sensitivity to color, degradation from color anomalies significantly impacts subjective evaluations of image quality. Recent advances have been made in traditional algorithms for processing color images. Quaternions, generalizations of complex numbers, consist of a real part and three imaginary parts, precisely encoding the three channels of a color image. This captures the coupling relations between channels, allowing the color image to be processed as a whole, leveraging non-trivial algebraic structure. Quaternion multiplication facilitates complete information exchange between channels. Recent studies \cite{chen2019low,yu2019quaternion} demonstrate that the color image processing method based on quaternion representation delivers excellent color image performance. Figure \ref{patch_1} succinctly illustrates the advantages of the quaternion method in color processing. In the smoothing region, even subtle processing errors between channels can lead to significant color artifacts, whereas the quaternion-based QNMF yields excellent visual results.

\begin{figure}[t]
	\centering
	\addtolength{\tabcolsep}{-5.5pt}
    \renewcommand\arraystretch{0.6}
	{\fontsize{10pt}{\baselineskip}\selectfont 
	\begin{tabular}{ccccc}
        \includegraphics[width=0.15\textwidth]{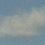} &
        \includegraphics[width=0.15\textwidth]{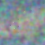} &
        \includegraphics[width=0.15\textwidth]{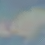} &
        \includegraphics[width=0.15\textwidth]{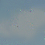} &
        \includegraphics[width=0.15\textwidth]{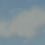} \\

        \includegraphics[width=0.15\textwidth]{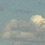} &
        \includegraphics[width=0.15\textwidth]{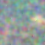} &
        \includegraphics[width=0.15\textwidth]{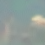} &
        \includegraphics[width=0.15\textwidth]{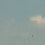} &
        \includegraphics[width=0.15\textwidth]{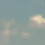} \\

        GT & SV-TV & CBM3D & McWNNM & QNMF \\

	\end{tabular}
	}
	\caption{ Comparison of color image denoising algorithms with different strategies in terms of color distortion and artifacts. The images are $45\times45$ patches, sourced from the Kodak dataset, and are corrupted by noise with $\sigma=60$.
 }
	\label{patch_1}
\end{figure}

To better address color distortion and artifacts while preserving more image details, we propose a new nonconvex quaternion hybrid norm, termed the Quaternion Nuclear Norm Minus Frobenius Norm (QNMF). The QNMF effectively handles various rank components, enabling improved approximation of the rank norm. We employ QNMF as a quaternionic low-rank regularizer for a wide range of low-level vision tasks, thus enhancing algorithm performance and robustness. Additionally, we offer a rigorous mathematical proof to ensure the theoretical soundness of QNMF.
The main contributions of this paper are as follows:
\begin{itemize}

    \item To enhance the approximation of the rank norm, this paper introduces QNMF as a quaternionic low-rank regularizer, applicable to various color image tasks.

    \item We prove the global optimality of our proposed QNMF model and demonstrate the convergence of the algorithm.

    \item     Numerous experiments were conducted, demonstrating that the proposed QNMF is capable of addressing various tasks including color image denoising, deblurring, matrix completion, and robust principal component analysis (RPCA), and achieved state-of-the-art results.
    
\end{itemize}

The paper is structured as follows: Section \ref{sec:2} covers the fundamentals of quaternion algebra and quaternion low-rank correlation methods. Our proposed nuclear norm minus Frobenius norm and its application to color image processing are introduced in Section \ref{sec:3}. Section \ref{sec:4} presents simulation experiments comparing our approach with other state-of-the-art methods. Conclusions are outlined in Section \ref{sec:5}.

\section{Related work}
\label{sec:2}

\subsection{Quaternion Basics}
As a generalization of the real space $\mathbb{R}$ and the complex space $\mathbb{C}$, the quaternion space $\mathbb{ Q }$ is defined as $\mathbb{ Q } = \{a_0 + a_1\mathbf{i} + a_2\mathbf{j} +a_3\mathbf{k}|a_0,a_1,a_2,a_3 \in \mathbb{R}\}$, where $\{1,\mathbf{i},\mathbf{j},\mathbf{k} \}$ forms a basis of $\mathbb{ Q }$, and the imaginary units $\mathbf{i},\mathbf{j},\mathbf{k}$ satisfy $\mathbf{i}^2=\mathbf{j}^2=\mathbf{k}^2=\mathbf{i}\mathbf{j}\mathbf{k}=-1$, $\mathbf{i}\mathbf{j}=\mathbf{k}=-\mathbf{j}\mathbf{i}$, $\mathbf{j}\mathbf{k}=\mathbf{i}=-\mathbf{k}\mathbf{j}$, $\mathbf{k}\mathbf{i}=\mathbf{j}=-\mathbf{i}\mathbf{k}$. Quaternion multiplication, one of the most important properties of the quaternion space $\mathbb{ Q }$, does not satisfy the law of exchange. For example, a quaternion variable is denoted by having a dot on the variable, such as $\dot{\mathbf{a}}$ and $\dot{\mathbf{A}}$.

Let $\dot{\mathbf{a}} = a_0 + a_1\mathbf{i} + a_2\mathbf{j} +a_3\mathbf{k} \in \mathbb{ Q }$, $\dot{\mathbf{b}} = b_0 + b_1\mathbf{i} + b_2\mathbf{j} +b_3\mathbf{k} \in \mathbb{ Q }$ and $\lambda \in \mathbb{R}$,  their operations are
\begin{equation*}
\begin{aligned}
&\dot{\mathbf{a}} + \dot{\mathbf{b}} =   (a_0 + b_0) + (a_1+b_1)\mathbf{i} + (a_2+b_2)\mathbf{j} + (a_3+b_3)\mathbf{k}, \\
&\lambda\dot{\mathbf{a}} =  (\lambda a_0) + (\lambda a_1)\mathbf{i} + (\lambda a_2)\mathbf{j} + (\lambda a_3)\mathbf{k}, \\
&\dot{\mathbf{a}}\dot{\mathbf{b}} = (a_0b_0-a_1b_1-a_2b_2-a_3b_3) + (a_0b_1+a_1b_0+a_2b_3-a_3b_2)\mathbf{i} \\
&~~+ (a_0b_2-a_1b_3+a_2b_0+a_3b_1)\mathbf{j}  + (a_0b_3+a_1b_2-a_2b_1-a_3b_0)\mathbf{k}.
\end{aligned}
\end{equation*}
The conjugate and modulus of $\dot{\mathbf{a}}$ are defined as
\begin{equation}
\begin{aligned}
\dot{\mathbf{a}}^{*} &= a_0 - a_1\mathbf{i} - a_2\mathbf{j} - a_3\mathbf{k}, \\
|\dot{\mathbf{a}}| &= \sqrt{a_{0}^{2}+a_{1}^{2}+a_{2}^{2}+a_{3}^{2}}.
\end{aligned}
\end{equation}
If $a_0 = 0$, $\dot{\mathbf{a}}$ is called pure quaternion. Each quaternion $\dot{\mathbf{a}}$ can be uniquely represented as 
$$\dot{\mathbf{a}} = a_0 + a_1\mathbf{i} + (a_2 +a_3\mathbf{i})\mathbf{j} = c_{1} + c_{2}\mathbf{j},$$
where $c_{1}=a_0 + a_1\mathbf{i}$ and $c_{2}=a_2 +a_3\mathbf{i}$ are complex numbers. 

For a quaternion matrix, defined as $\dot{\mathbf{X}} = (\dot{x}_{ij}) \in \mathbb{ Q }^{m\times n}$, where $\dot{\mathbf{X}} = \mathbf{X}_0 + \mathbf{X}_1\mathbf{i} + \mathbf{X}_2\mathbf{j}+ \mathbf{X}_3\mathbf{k}$ and $\mathbf{X}_l\in \mathbb{R}^{m\times n}(l=0,1,2,3)$ are real matrices. For a color image, it can be represented as a pure quaternion matrix with real part 0, i.e., each channel is represented as an imaginary part
\begin{equation}
\begin{aligned}
\dot{\mathbf{X}}_{RGB} =  \mathbf{X}_{R}\mathbf{i} + \mathbf{X}_{G}\mathbf{j}+ \mathbf{X}_{B}\mathbf{k},
\end{aligned}
\end{equation}
where $\mathbf{X}_R$, $\mathbf{X}_G$ and $\mathbf{X}_B$ are the R, G, and B channels of the color image.

The norms of quaternion vectors and matrices are defined as follows. The $l_1$ and $l_2$ norms of the quaternion vector are $||\dot{\mathbf{a}}||_1:=\sum_{i=0}^{3}|a_i|$ and $||\dot{\mathbf{a}}||_2:=\sqrt{\sum_{i=0}^{3}|a_i|^2}$, respectively. The Frobenius norm of a quaternionic matrix is $||\dot{\mathbf{X}}||_{F} := \sqrt{\sum_{i,j}| \dot{x}_{ij}|^{2}} = \sqrt{\sum_{k}^{\min{\{i,j\}}}  \sigma_{\dot{\mathbf{X}},k}^{2} }$, where $\sigma_{\dot{\mathbf{X}},k}$ denotes the $k$-th singular value. The singular value decomposition of a quaternion matrix $\dot{\mathbf{X}}$ is given by the following lemma.

\begin{lemma}(QSVD \cite{zhang1997quaternions})
Given a quaternion matrix $\dot{\mathbf{X}} \in \mathbb{ Q }^{m\times n}$ with rank $r$.  There are two unitary quaternion matrices $\dot{\mathbf{U}} \in \mathbb{ Q }^{m\times m}$ and $\dot{\mathbf{V}} \in \mathbb{ Q }^{n\times n}$ satisfying $\dot{\mathbf{X}}  = \dot{\mathbf{U}}\left(\begin{array}{cc} \mathbf{\Sigma_{r}} & 0\\ 0&0 \end{array}\right)\dot{\mathbf{V}}^{\top}$, where $\mathbf{\Sigma_{r}} = \mathrm{diag}(\sigma_1,\dots,\sigma_r)\in \mathbb{R}^{r\times r}$, and all singular values $\sigma_i>0, i=1,\dots,r$.
\end{lemma}

\subsection{Quaternion Low Rank Methods}

Real-valued low-rank methods excel in grayscale image processing but face hurdles with color images, often resulting in color distortion and artifacts due to disregarding RGB channel correlation. Recognizing the significance of inter-channel correlation, quaternion-based color image processing methods have garnered increasing interest in recent years.

Yu et al. \cite{yu2019quaternion} introduced quaternion-weighted nuclear norm minimization (QWNNM) for color image denoising, drawing inspiration from \cite{gu2014weighted,gu2017weighted}. Extending this approach, Huang et al. \cite{huang2022quaternion} adapted QWNNM for processing color image deblurring. Mathematically, QWNNM is defined as:
\begin{equation}
\begin{aligned}
\min_{\dot{\mathbf{X}}} \frac{\gamma}{2}|| \dot{\mathbf{A}}\dot{\mathbf{X}} - \dot{\mathbf{Y}} ||^{2}_F + ||\dot{\mathbf{X}}||_{w,*} 
\label{deblur-QWNNM}
\end{aligned}
\end{equation}
where  $ \dot{\mathbf{A}} \in \mathbb{ Q }^{m\times m}$ represents a 2D quaternion degradation matrix, while $\dot{\mathbf{X}}$ and $\dot{\mathbf{Y}}$ stand for the potential image and the observed degraded image, respectively. 
The term $||\cdot||_{w,*} $ is defined as 
$$ ||\dot{\mathbf{X}}||_{w,*} = \sum_{i} |w_{i}\sigma_{\dot{\mathbf{X}},i} |,$$
where $\sigma_{\dot{\mathbf{X}},i}$ is the $i$-th singular value of $\dot{\mathbf{X}}$ and the weight $w_{i}$ is computed as   $w_{i}=\frac{c}{\sigma_{\dot{\mathbf{X}},i}+\epsilon}$, where $c$ is a constant and $\epsilon $ is a small number.

The studies by Yu et al. \cite{yu2019quaternion} and Huang et al. \cite{huang2022quaternion} illustrate the effectiveness of quaternion representation in color image processing by leveraging inter-channel correlation. Combining the weighted Schatten p-norm minimization \cite{xie2016weighted} with quaternion representation, Zhang et al. \cite{zhang2024quaternion} proposed the QWSNM model, defined as:
\begin{equation}
\begin{aligned}
||\dot{\mathbf{X}}||_{w,S_{p}} = \left( \sum_{i=1}^{\min{\{n,m\}}} w_{i}\sigma_{\dot{\mathbf{X}},i}^{p} \right)^{\frac{1}{p}},
\end{aligned}
\end{equation}
where $p>0$. The QWSNM offers precise control over each singular value. Notably, when $p=1$, QWNNM emerges as a special instance of QWSNM. However, the inclusion of the Schatten p-norm increases the computational complexity of QWSNM relative to QWNNM.

Several recent works \cite{wang2021scalable,yu2023low} have highlighted that convex alternatives to rank norms can yield suboptimal solutions. In response, various nonconvex alternatives to rank norms have emerged \cite{chen2019low,yang2022quaternion}. Building upon these insights, we introduce a novel quaternion nonconvex hybrid norm, termed quaternion nuclear norm minus Frobenius norm, tailored for color image processing.

\section{Quaternion Nuclear Norm minus Frobenius Norm Minimization}
\label{sec:3}

In this section, we introduce the quaternion nuclear norm minus Frobenius norm minimization (QNMF) problem, extending it to the image linear inverse problem.

\subsection{QNMF model}
\label{3.1}
To address the color image denoising problem, we propose a novel approach by combining quaternion representation with the nuclear-Frobenius hybrid norm. This leads to the formulation of the quaternion nuclear norm minus Frobenius norm minimization problem:
\begin{equation}
\dot{\mathbf{X}} = \arg\min_{\dot{\mathbf{X}}} \frac{1}{2}|| \dot{\mathbf{Y}}-\dot{\mathbf{X}} ||^{2}_F + \lambda(||\dot{\mathbf{X}}||_{*} - \alpha||\dot{\mathbf{X}}||_{F} ),
\label{NNFN}
\end{equation}
where $\dot{\mathbf{Y}}$ represents the observed noisy color image, and $\lambda>0$ and $\alpha>0$ are regularization parameters. The global optimal solution of model (\ref{NNFN}) can be obtained according to the following theorem.

\begin{theorem}
Given $\dot{\mathbf{Y}} \in \mathbb{ Q }^{m\times n}$, without loss of generality, we assume that $m \geq n$. Let $\dot{\mathbf{Y}}= \dot{\mathbf{U}} \mathbf{\Sigma_{\dot{Y}}} \dot{\mathbf{V}}^{\top}$ be the QSVD of $\dot{\mathbf{Y}}$, where $\mathbf{\Sigma_{\dot{Y}}}=\mathrm{diag}(\sigma_{\dot{\mathbf{Y}},1},\sigma_{\dot{\mathbf{Y}},2},\dots,\sigma_{\dot{\mathbf{Y}},m})$. The global optimum of model (\ref{NNFN}) can be expressed as $\dot{\mathbf{X}}=\hat{\dot{\mathbf{U}}} \mathbf{\Sigma_{\dot{X}}} \hat{\dot{\mathbf{V}}}^{\top}$, where $\mathbf{\Sigma_{\dot{X}}}=\mathrm{diag}(\sigma_{\dot{\mathbf{X}},1},\sigma_{\dot{\mathbf{X}},2},\dots,\sigma_{\dot{\mathbf{X}},m})$ is a diagonal nonnegative matrix and $(\sigma_{\dot{\mathbf{X}},1},\sigma_{\dot{\mathbf{X}},2},\dots,\sigma_{\dot{\mathbf{X}},m})$ is the solution to the following optimization problem:
\begin{equation}
\begin{aligned}
\arg&\min_{\mathbf{\Sigma_{\dot{X}}}} \frac{1}{2} ||\mathbf{\Sigma_{\dot{Y}}} - \mathbf{\Sigma_{\dot{X}}} ||^{2}_{2} + \lambda(||\mathbf{\Sigma_{\dot{X}}}||_{1}-\alpha ||\mathbf{\Sigma_{\dot{X}}}||_{2}) \\
& \mathrm{s.t.} \sigma_{\dot{\mathbf{X}},1} \geq \sigma_{\dot{\mathbf{X}},2} \geq \dots \geq \sigma_{\dot{\mathbf{X}},m} \geq 0, 
\end{aligned}
\end{equation}
and the globally optimal solution $\mathbf{\tilde{\Sigma}_{\dot{X}} }$ is $\mathbf{\tilde{\Sigma}_{\dot{X}} } = \frac{|| \mathbf{z} ||_{2} + \alpha\lambda}{|| \mathbf{z} ||_{2}}\mathbf{z}$, where $\mathbf{z} = \max(\mathbf{\Sigma_{\dot{Y}}}-\lambda,0)$. 
\label{theorem01}
\end{theorem}

\begin{proof}[Proof]
Let $\dot{\mathbf{X}}=\hat{\dot{\mathbf{U}}} \mathbf{\Sigma_{\dot{X}}} \hat{\dot{\mathbf{V}}}^{\top}$ and  $\dot{\mathbf{Y}}= \dot{\mathbf{U}} \mathbf{\Sigma_{\dot{Y}}} \dot{\mathbf{V}}^{\top}$ be the QSVD decomposition of $\dot{\mathbf{X}}$ and $\dot{\mathbf{Y}}$.
\begin{equation}
\begin{aligned}
& \frac{1}{2}|| \dot{\mathbf{Y}}-\dot{\mathbf{X}} ||^{2}_F + \lambda(||\dot{\mathbf{X}}||_{*} - \alpha||\dot{\mathbf{X}}||_{F} ) \\
&= \frac{1}{2}\mathrm{tr}(\dot{\mathbf{Y}}^{\top}\dot{\mathbf{Y}} + \dot{\mathbf{X}}^{\top}\dot{\mathbf{X}} - 2\dot{\mathbf{X}}^{\top}\dot{\mathbf{Y}}) + \lambda(||\dot{\mathbf{X}}||_{*} - \alpha||\dot{\mathbf{X}}||_{F} ) \\
&= \frac{1}{2}( || \mathbf{\Sigma_{\dot{Y}}} ||^{2}_{2} + || \mathbf{\Sigma_{\dot{X}}} ||^{2}_{2}) - \mathrm{tr}(\dot{\mathbf{X}}^{\top}\dot{\mathbf{Y}}) + \lambda|| \mathbf{\Sigma_{\dot{X}}}||_{1} - \lambda \alpha ||\mathbf{\Sigma_{\dot{X}}}||_{2}.
\end{aligned}
\end{equation}
According to the Von Neumann's trace inequality, when $\dot{\mathbf{U}}=\hat{\dot{\mathbf{U}}}$ and $\dot{\mathbf{V}}=\hat{\dot{\mathbf{V}}}$,   $\mathrm{tr}(\dot{\mathbf{X}}^{\top}\dot{\mathbf{Y}}) \leq \mathbf{\Sigma_{\dot{X}}}^{\top}\mathbf{\Sigma_{\dot{Y}}}$.
Then solving (\ref{NNFN}) can be instead computed by solving $\mathbf{\Sigma_{\dot{X}}}$ as
\begin{equation}
\begin{aligned}
\arg&\min_{\mathbf{\Sigma_{\dot{X}}}} \frac{1}{2} ||\mathbf{\Sigma_{\dot{Y}}} - \mathbf{\Sigma_{\dot{X}}} ||^{2}_{2} + \lambda(||\mathbf{\Sigma_{\dot{X}}}||_{1}-\alpha ||\mathbf{\Sigma_{\dot{X}}}||_{2}) \\
& \mathrm{s.t.} \sigma_{\dot{\mathbf{X}},1} \geq \sigma_{\dot{\mathbf{X}},2} \geq \dots \geq \sigma_{\dot{\mathbf{X}},m} \geq 0,
\end{aligned}
\label{l1-l2}
\end{equation}
where $\mathbf{\Sigma_{\dot{X}}}=\mathrm{diag}(\sigma_{\dot{\mathbf{X}},1},\sigma_{\dot{\mathbf{X}},2},\dots,\sigma_{\dot{\mathbf{X}},m})$. According to \cite{lou2018fast}, the global optimal solution of model (\ref{l1-l2}) can be obtained by
\begin{equation}
\mathbf{\tilde{\Sigma}_{\dot{X}} } = \frac{|| \mathbf{z} ||_{2}+\alpha\lambda}{|| \mathbf{z} ||_{2}}\mathbf{z}, 
\end{equation}
where $\mathbf{z} = \max(\mathbf{\Sigma_{\dot{Y}}}-\lambda,0)$.
\end{proof}

According to Theorem \ref{theorem01}, the globally optimal solution for the model (\ref{NNFN}) can be expressed as $\tilde{\dot{\mathbf{X}}} = \dot{\mathbf{U}} \mathbf{\tilde{\Sigma}_{\dot{X}}} \dot{\mathbf{V}}^{\top}$. Larger singular values hold more information and should be shrunk less. Therefore, we adjust the shrinkage threshold $\lambda$ accordingly.
If $\sigma_{\dot{\mathbf{Y}},i} \leq \lambda$, we keep the current shrinkage threshold: $ z_i = \max(\sigma_{\dot{\mathbf{Y}},i} - \lambda, 0) = 0$. However, if $\sigma_{\dot{\mathbf{Y}},i} > \lambda$, we reduce the shrinkage threshold to $\lambda/2$: $ z_i = \sigma_{\dot{\mathbf{Y}},i} - \lambda/2$.
We observe that $\frac{|| \mathbf{z} ||_{2} + \alpha\lambda/2}{|| \mathbf{z} ||_{2}} > 1$, which could lead to $ \tilde{\sigma}_{\dot{\mathbf{X}},i} = \frac{|| \mathbf{z} ||_{2} + \alpha\lambda/2}{|| \mathbf{z} ||_{2}} z_i > \sigma_{\dot{\mathbf{Y}},i}$. This is not allowed in singular value shrinkage. To prevent this, when $\frac{|| \mathbf{z} ||_{2} + \alpha\lambda/2}{|| \mathbf{z} ||_{2}} z_i > \sigma_{\dot{\mathbf{Y}},i}$, we set $\tilde{\sigma}_{\dot{\mathbf{X}},i} = \sigma_{\dot{\mathbf{Y}},i}$.
In conclusion, the final solution for $\mathbf{\tilde{\Sigma}_{\dot{X}} }$ is presented as follows:
\begin{equation}  
\tilde{\sigma}_{\dot{\mathbf{X}},i} = 
\left\{  
     \begin{array}{cl}  
     \sigma_{\dot{\mathbf{Y}},i}, & \mathrm{if} \enspace \sigma_{\dot{\mathbf{Y}},i} > \frac{K\lambda}{2(K-1)};\\
     K(\sigma_{\dot{\mathbf{Y}},i}-\lambda/2), & \mathrm{if} \enspace \lambda < \sigma_{\dot{\mathbf{Y}},i} \leq \frac{K\lambda}{2(K-1)}; \enspace \\
     0, & \mathrm{if} \enspace \sigma_{\dot{\mathbf{Y}},i} \leq \lambda. \\    
     \end{array}  
\right.  
\label{solution}
\end{equation}
where $K = 1 + \frac{\alpha\lambda}{2|| \mathbf{z} ||_{2}}$, $\mathbf{z} = \mathrm{diag}(z_1,z_2,\dots,z_m)$, and 
$$ z_i = \left\{  
     \begin{array}{cl} \sigma_{\dot{\mathbf{Y}},i}-\lambda/2, & \mathrm{if} \enspace \sigma_{\dot{\mathbf{Y}},i} > \lambda;\\ 
     0, & \mathrm{if} \enspace \sigma_{\dot{\mathbf{Y}},i} \leq \lambda. 
     \end{array}  
\right. $$

Figure \ref{singular}(a) illustrates the shrinkage process described in Eq. (\ref{solution}) on a $35 \times 35$ patch with $\sigma=60$. We compare our proposed QNMF method with previous studies such as QWNNM \cite{yu2019quaternion} and QWSNM \cite{zhang2024quaternion}. Similar to QWNNM and QWSNM, QNMF applies less shrinkage to larger singular values, but it avoids complex weight calculations. A key distinction is that QNMF offers a closed-form solution, eliminating the need for iterative shrinkage and simplifying computations.
Figure \ref{singular}(b) shows the variation in shrinkage among the three methods. QNMF incorporates truncation to prevent excessive increases in singular values, setting $\tilde{\sigma}_{\dot{\mathbf{X}},i}=\sigma_{\dot{\mathbf{Y}},i}$ when $\sigma_{\dot{\mathbf{Y}},i} > \frac{K\lambda}{2(K-1)}$. This approach is similar to truncated nuclear norms (TNN) \cite{yang2021weighted}, but with a significant difference: while TNN uses a fixed truncation threshold, QNMF adjusts dynamically.

\begin{figure}[t]
	\centering
	% \addtolength{\tabcolsep}{-5.5pt}
    \renewcommand\arraystretch{0.4}
	{\fontsize{8pt}{\baselineskip}\selectfont 
	\begin{tabular}{cc}
        \includegraphics[width=0.4\textwidth, trim={14 19 39 21}, clip]{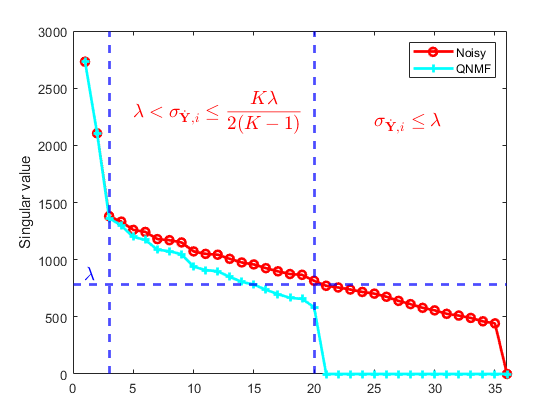} &
        \includegraphics[width=0.4\textwidth, trim={14 19 39 21}, clip]{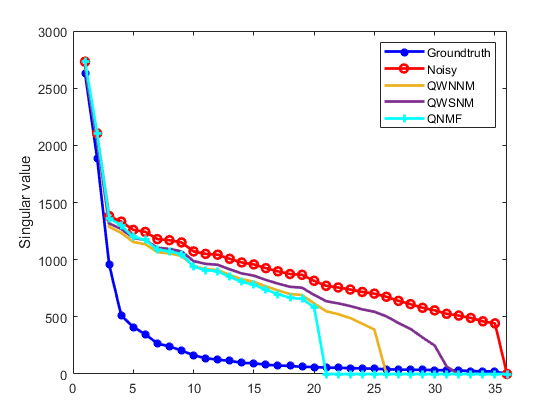} \\
        (a) & (b) \\
	\end{tabular}
	}
	\caption{Singular value shrinkage diagrammed.}
	\label{singular}
\end{figure}

\subsection{QNMF for image linear inverse problem}
We extend QNMF to the image linear inverse problem given by:
\begin{equation}
\begin{aligned}
\dot{\mathbf{X}} = \arg\min_{\dot{\mathbf{X}}} \frac{\gamma}{2}|| \dot{\mathbf{A}}\dot{\mathbf{X}} - \dot{\mathbf{Y}} ||^{2}_F + \lambda(||\dot{\mathbf{X}}||_{*} - \alpha||\dot{\mathbf{X}}||_{F} ) 
\label{deblur-NNFN}
\end{aligned}
\end{equation}
where $ \dot{\mathbf{A}} \in \mathbb{ Q }^{m\times m}$ represents a 2D quaternion degradation matrix.For instance, if $\dot{\mathbf{A}}=\dot{\mathbf{I}}$, it pertains to the denoising problem discussed in subsection \ref{3.1}. Alternatively, if $\dot{\mathbf{A}}$ denotes a blur kernel, it corresponds to color image deblurring. $\dot{\mathbf{X}} \in \mathbb{ Q }^{m\times n} $ and $\dot{\mathbf{Y}} \in \mathbb{ Q }^{m\times n} $ represent the original and degraded images, respectively.

We utilize the alternating direction multiplier method (ADMM) to solve the model (\ref{deblur-NNFN}). This involves introducing an auxiliary variable $\dot{\mathbf{Z}}$ with the constraint $\dot{\mathbf{X}}=\dot{\mathbf{Z}}$, and the Lagrange multiplier $\dot{\mathbf{\eta}}$. The augmented Lagrange function of (\ref{deblur-NNFN}) is then expressed as
\begin{equation}
\begin{aligned}
\mathcal{L}(\dot{\mathbf{X}}, \dot{\mathbf{Z}}, \dot{\mathbf{\eta}}, \beta) &= \frac{\gamma}{2}|| \dot{\mathbf{A}}\dot{\mathbf{X}} - \dot{\mathbf{Y}} ||^{2}_F + \lambda(||\dot{\mathbf{Z}}||_{*} - \alpha||\dot{\mathbf{Z}}||_{F} ) + \frac{\beta}{2}||\dot{\mathbf{X}}-\dot{\mathbf{Z}}||_{F}^{2} + \langle{ \dot{\bm{\eta}}, \dot{\mathbf{X}}-\dot{\mathbf{Z}} }\rangle. 
\label{en-NNFN}
\end{aligned}
\end{equation}
where $\beta > 0$ is the penalty parameter. The update strategy for (\ref{en-NNFN}) is as follows:
\begin{align}\left\{\begin{aligned}
\dot{\mathbf{X}}^{(k+1)} &= \arg\min_{\dot{\mathbf{X}}} \mathcal{L}(\dot{\mathbf{X}}, \dot{\mathbf{Z}}^{(k)}, \dot{\mathbf{\eta}}^{(k)}, \beta^{(k)}), \\
\dot{\mathbf{Z}}^{(k+1)} &= \arg\min_{\dot{\mathbf{Z}}} \mathcal{L}(\dot{\mathbf{X}}^{(k+1)}, \dot{\mathbf{Z}}, \dot{\mathbf{\eta}}^{(k)}, \beta^{(k)}), \\
\dot{\mathbf{\eta}}^{(k+1)} &= \dot{\mathbf{\eta}}^{(k)} + \beta^{(k)}(\dot{\mathbf{X}}^{(k+1)} - \dot{\mathbf{Z}}^{(k+1)}), \\
\beta^{(k+1)} &= \mu\beta^{(k)},\quad(\mu>1). \\
\end{aligned}\right.
\end{align}
The minimization (\ref{en-NNFN}) readily splits into two subproblems: $\dot{\mathbf{X}}$ and $\dot{\mathbf{Z}}$.

The $\dot{\mathbf{X}}$ sub-problem is written as
\begin{equation}
\begin{aligned}
\dot{\mathbf{X}} = \arg\min_{\dot{\mathbf{X}}} \frac{\gamma}{2}|| \dot{\mathbf{A}}\dot{\mathbf{X}} - \dot{\mathbf{Y}} ||^{2}_F +  \frac{\beta}{2}||\dot{\mathbf{X}}-\dot{\mathbf{Z}}||_{F}^{2} + \langle{ \dot{\mathbf{\eta}}, \dot{\mathbf{X}}-\dot{\mathbf{Z}} }\rangle. 
\end{aligned}
\end{equation}
By quaternion theory \cite{Xu2015Theory}, the closed solution of the problem is
\begin{equation}
\begin{aligned}
\dot{\mathbf{X}} = (\gamma\dot{\mathbf{A}}^{*}\dot{\mathbf{A}}+\beta\dot{\mathbf{I}})^{-1}(\gamma\dot{\mathbf{A}}^{*}\dot{\mathbf{Y}}+\beta\dot{\mathbf{Z}}-\dot{\mathbf{\eta}}).
\end{aligned}
\label{qfft}
\end{equation}
The left inverse of a quaternion matrix is defined as follows: if a quaternion matrix $\dot{\mathbf{U}}$ is known and there exists a quaternion matrix $\dot{\mathbf{V}}$ such that $\dot{\mathbf{V}}\dot{\mathbf{U}}=\mathbf{I}$, then $\dot{\mathbf{V}}$ is the left inverse of $\dot{\mathbf{U}}$, denoted as $\dot{\mathbf{U}}^{-1}$. To avoid complex inverse operations, this can be efficiently solved using the quaternion fast Fourier transform \cite{sangwine1996fourier}.

According to \eqref{en-NNFN}, the $\dot{\mathbf{Z}}$ subproblem can be reformulated as
\begin{equation}
\begin{aligned}
\dot{\mathbf{Z}} = \arg\min_{\dot{\mathbf{Z}}} \, \lambda(||\dot{\mathbf{Z}}||_{*} - \alpha||\dot{\mathbf{Z}}||_{F} ) + \frac{\beta}{2}||\dot{\mathbf{Z}} - ( \dot{\mathbf{X}} + \frac{ \dot{\mathbf{\eta}}}{\beta})||_{F}^{2}.
\end{aligned}
\label{zz}
\end{equation}
Therefore, the $\dot{\mathbf{Z}}$ subproblem can be solved using \eqref{solution}. The entire optimization process is summarized in Algorithm \ref{ag1}. Subsequently, we provide weak convergence results for Algorithm \ref{ag1}.

\begin{theorem}
Suppose that the parameter sequence $\{\beta^{(k)}\}$ is unbounded and the sequences $\{\dot{\mathbf{X}}^{(k)}\}$ and $\{\dot{\mathbf{Z}}^{(k)}\}$ are generated by Algorithm \ref{ag1}, which satisfies: 
\begin{align}
&\lim_{k\rightarrow +\infty} || \dot{\mathbf{X}}^{(k+1)}-\dot{\mathbf{Z}}^{(k+1)} ||_{F} = 0, \label{converge1}\\
&\lim_{k\rightarrow +\infty} || \dot{\mathbf{X}}^{(k+1)}-\dot{\mathbf{X}}^{(k)} ||_{F} = 0, \label{converge2}\\
&\lim_{k\rightarrow +\infty} || \dot{\mathbf{Z}}^{(k+1)}-\dot{\mathbf{Z}}^{(k)} ||_{F} = 0. \label{converge3}
\end{align}
\label{theorem02}
\end{theorem}

% Appendix A supplementary material
The detailed proof of Theorem \ref{theorem02} is available in Appendix A. It's noteworthy that the proof hinges on the unboundedness of $\beta^{(k)}$. Rapid increases in $\beta^{(k)}$ can prematurely halt iterations, hindering the attainment of a satisfactory solution. Thus, smaller $\mu$ values are employed in the experiments to restrain the rapid growth of $\beta^{(k)}$.

% % 算法
\begin{algorithm}[t]
\SetAlgoNoLine
\SetKwInOut{Input}{\textbf{Input}}
\SetKwInOut{Output}{\textbf{Output}}

\Input{
    Observation $\dot{\mathbf{Y}}$; Initialize $\dot{\mathbf{X}}^{(0)}=\dot{\mathbf{Y}}$, $\dot{\mathbf{Z}}^{(0)}=\dot{\mathbf{X}}^{(0)}$, $\dot{\mathbf{\eta}}^{(0)}=0$; Parameters $\gamma$, $\lambda$, $\beta$, $\mu$, $\alpha$, $K$;\\
    }

\Output{
    The reconstructed image $\dot{\mathbf{X}}^{(K)}$; \\
    }

    \For{$k=0,1,\dots,K-1$}{
    Update $\dot{\mathbf{X}}^{(k+1)}$ by (\ref{qfft}) \;
    Update $\dot{\mathbf{Z}}^{(k+1)}$ by (\ref{zz}) \;
    Update $\dot{\mathbf{\eta}}^{(k+1)} = \dot{\mathbf{\eta}}^{(k)} + \beta^{(k)}(\dot{\mathbf{X}}^{(k+1)} - \dot{\mathbf{Z}}^{(k+1)})$ \;
    Update $\beta^{(k+1)} = \mu\beta^{(k)}$ \;
    }
    \caption{QNMF Algorithm}
    \label{ag1}
\end{algorithm}

\subsection{QNMF for matrix completion}
In this section, we apply QNMF to tackle the matrix completion problem and introduce a model based on QNMF:
\begin{equation}
\begin{aligned}
\min_{\dot{\mathbf{X}}} ||\dot{\mathbf{X}}||_{*} - \alpha||\dot{\mathbf{X}}||_{F} \quad s.t \quad \mathcal{P}_{\Omega}(\dot{\mathbf{X}}) = \mathcal{P}_{\Omega}(\dot{\mathbf{Y}}),
\label{mc-NNFN}
\end{aligned}
\end{equation}
where $\Omega$ is a binary mask matrix of the same size as $\dot{\mathbf{Y}}$. Zeroes in $\Omega$ indicate unobserved values. $\mathcal{P}_{\Omega}(\dot{\mathbf{Y}})=\Omega\odot\dot{\mathbf{Y}}$, where $\odot$ denotes the Hadamard product. This constraint ensures consistency between the estimated $\dot{\mathbf{X}}$ and the observed $\dot{\mathbf{Y}}$ at the observed entries.

By introducing the auxiliary variable $\dot{\mathbf{Z}}$, and the Lagrange multiplier $\dot{\mathbf{\eta}}$,  the augmented Lagrange function for (\ref{mc-NNFN}) is expressed as:
\begin{equation}
\begin{aligned}
\mathcal{L}(\dot{\mathbf{X}}, \dot{\mathbf{Z}}, \dot{\mathbf{\eta}}, \beta) &=  \lambda(||\dot{\mathbf{X}}||_{*} - \alpha||\dot{\mathbf{X}}||_{F} ) + \frac{\beta}{2}||\dot{\mathbf{Y}}-\dot{\mathbf{X}} - \dot{\mathbf{Z}} ||_{F}^{2} \\
&~~~+ \langle{ \dot{\mathbf{\eta}},\dot{\mathbf{Y}}-\dot{\mathbf{X}} - \dot{\mathbf{Z}} }\rangle \quad s.t \quad \mathcal{P}_{\Omega}(\dot{\mathbf{Z}}) = 0. 
\label{en-mc}
\end{aligned}
\end{equation}
The update strategy for (\ref{en-mc}) is as follows:
\begin{align}\left\{\begin{aligned}
\dot{\mathbf{Z}}^{(k+1)} &= \arg\min_{\dot{\mathbf{Z}}} \mathcal{L}(\dot{\mathbf{X}}^{(k+1)}, \dot{\mathbf{Z}}, \dot{\mathbf{\eta}}^{(k)}, \beta^{(k)}),  \quad s.t \quad ||\mathcal{P}_{\Omega}(\dot{\mathbf{Z}})||_{F}^{2} = 0,  \\
\dot{\mathbf{X}}^{(k+1)} &= \arg\min_{\dot{\mathbf{X}}} \mathcal{L}(\dot{\mathbf{X}}, \dot{\mathbf{Z}}^{(k)}, \dot{\mathbf{\eta}}^{(k)}, \beta^{(k)}), \\
\dot{\mathbf{\eta}}^{(k+1)} &= \dot{\mathbf{\eta}}^{(k)} + \beta^{(k)}(\dot{\mathbf{Y}}- \dot{\mathbf{X}}^{(k+1)} - \dot{\mathbf{Z}}^{(k+1)}), \\
\beta^{(k+1)} &= \mu\beta^{(k)},\quad(\mu>1). \\
\end{aligned}\right.
\end{align}
The $\dot{\mathbf{Z}}$ subproblem is
\begin{equation}
\begin{aligned}
\dot{\mathbf{Z}} = \arg\min_{\dot{\mathbf{Z}}} \, ||\dot{\mathbf{Y}} -  \dot{\mathbf{Z}} + \frac{ \dot{\mathbf{\eta}}}{\beta} - \dot{\mathbf{X}}||_{F}^{2} \; s.t \; ||\mathcal{P}_{\Omega}(\dot{\mathbf{Z}})||_{F}^{2} = 0.
\end{aligned}
\label{z-mc}
\end{equation}
This quadratic problem can be easily solved. The $\dot{\mathbf{X}}$ subproblem is
\begin{equation}
\begin{aligned}
\dot{\mathbf{X}} = \arg\min_{\dot{\mathbf{X}}} \, \lambda(||\dot{\mathbf{X}}||_{*} - \alpha||\dot{\mathbf{X}}||_{F} ) + \frac{\beta}{2}||(\dot{\mathbf{Y}} -  \dot{\mathbf{Z}} + \frac{ \dot{\mathbf{\eta}}}{\beta}) - \dot{\mathbf{X}}||_{F}^{2}.
\end{aligned}
\label{x-mc}
\end{equation}
The solution of this subproblem is therefore given in (\ref{solution}).  The entire algorithmic procedure is described in Algorithm \ref{ag2}.

\begin{theorem}
Suppose that the parameter sequence $\{\beta^{(k)}\}$ is unbounded and the sequences $\{\dot{\mathbf{X}}^{(k)}\}$ and $\{\dot{\mathbf{Z}}^{(k)}\}$ are generated by Algorithm \ref{ag2}, which satisfies: 
\begin{align}
&\lim_{k\rightarrow +\infty} || \dot{\mathbf{X}}^{(k+1)}-\dot{\mathbf{X}}^{(k)} ||_{F} = 0, \label{converge21}\\
&\lim_{k\rightarrow +\infty} || \dot{\mathbf{Y}} - \dot{\mathbf{X}}^{(k+1)}-\dot{\mathbf{Z}}^{(k+1)} ||_{F} = 0. \label{converge22}
\end{align}
\label{theorem03}
\end{theorem}

The proof of Theorem \ref{theorem03} mirrors that of Theorem \ref{theorem04} discussed subsequently; refer to the proof of Theorem \ref{theorem04} for specifics. It is not reiterated in this paper.

\begin{algorithm}[t]
\SetAlgoNoLine
\SetKwInOut{Input}{\textbf{Input}}
\SetKwInOut{Output}{\textbf{Output}}

\Input{
    Observation $\dot{\mathbf{Y}}$; Mask matrix $\Omega$; 
    Initialize $\dot{\mathbf{X}}^{(0)}=\dot{\mathbf{Y}}$, $\dot{\mathbf{\eta}}^{(0)}=0$; Parameters $\lambda$, $\beta$, $\mu$, $\alpha$, $K$;\\
    }

\Output{
    The reconstructed image $\dot{\mathbf{X}}^{(K)}$; \\
    }

    \For{$k=0,1,\dots,K-1$}{
    Update $\dot{\mathbf{Z}}^{(k+1)}$ by (\ref{z-mc}) \;
    Update $\dot{\mathbf{X}}^{(k+1)}$ by (\ref{x-mc}) \;
    Update $\dot{\mathbf{\eta}}^{(k+1)} = \dot{\mathbf{\eta}}^{(k)} + \beta^{(k)}(\dot{\mathbf{Y}}- \dot{\mathbf{X}}^{(k+1)} - \dot{\mathbf{Z}}^{(k+1)})$ \;
    Update $\beta^{(k+1)} = \mu\beta^{(k)}$ \;
    }
    \caption{QNMF-MC Algorithm}
    \label{ag2}
\end{algorithm}

\subsection{QNMF for robust PCA}

In practice, observed data are often corrupted by outliers or sparse noise, significantly impacting the Frobenius norm data fidelity term used in low-rank estimation. To address this issue,  Cand\'{e}s et al. \cite{candes2011robust} introduced robust principal component analysis (RPCA) based on the nuclear norm, utilizing the $\mathcal{L}1$-norm for modeling errors. Building upon this, we extend QNMF to RPCA, resulting in a QNMF-based RPCA model:
\begin{equation}
\begin{aligned}
\min_{\dot{\mathbf{X}},\dot{\mathbf{Z}}} \lambda(||\dot{\mathbf{X}}||_{*} - \alpha||\dot{\mathbf{X}}||_{F}) + \rho||\dot{\mathbf{Z}}||_{1} \; s.t \; \dot{\mathbf{Y}} = \dot{\mathbf{X}} + \dot{\mathbf{Z}}.
\label{rpca-NNFN}
\end{aligned}
\end{equation}
where $\dot{\mathbf{Z}}$ is sparse error. According to ADMM, its augmented Lagrange function is
\begin{equation}
\begin{aligned}
\mathcal{L}(\dot{\mathbf{X}}, \dot{\mathbf{Z}}, \dot{\mathbf{\eta}}, \beta) &=  \lambda(||\dot{\mathbf{X}}||_{*} - \alpha||\dot{\mathbf{X}}||_{F} ) + \rho||\dot{\mathbf{Z}}||_{1} + \frac{\beta}{2}||\dot{\mathbf{Y}}-\dot{\mathbf{X}} - \dot{\mathbf{Z}} ||_{F}^{2}
 + \langle{ \dot{\mathbf{\eta}},\dot{\mathbf{Y}}-\dot{\mathbf{X}} - \dot{\mathbf{Z}} }\rangle.
\label{en-rpca}
\end{aligned}
\end{equation}

The update strategy for (\ref{en-rpca}) is as follows:
\begin{align}\left\{\begin{aligned}
\dot{\mathbf{Z}}^{(k+1)} &= \arg\min_{\dot{\mathbf{Z}}} \mathcal{L}(\dot{\mathbf{X}}^{(k+1)}, \dot{\mathbf{Z}}, \dot{\mathbf{\eta}}^{(k)}, \beta^{(k)}),   \\
\dot{\mathbf{X}}^{(k+1)} &= \arg\min_{\dot{\mathbf{X}}} \mathcal{L}(\dot{\mathbf{X}}, \dot{\mathbf{Z}}^{(k)}, \dot{\mathbf{\eta}}^{(k)}, \beta^{(k)}), \\
\dot{\mathbf{\eta}}^{(k+1)} &= \dot{\mathbf{\eta}}^{(k)} + \beta^{(k)}(\dot{\mathbf{Y}}- \dot{\mathbf{X}}^{(k+1)} - \dot{\mathbf{Z}}^{(k+1)}), \\
\beta^{(k+1)} &= \mu\beta^{(k)},\quad(\mu>1). \\
\end{aligned}\right.
\end{align}
The $\dot{\mathbf{Z}}$ subproblem is
\begin{equation}
\begin{aligned}
\dot{\mathbf{Z}} = \arg\min_{\dot{\mathbf{Z}}} \, \rho||\dot{\mathbf{Z}}||_1 + \frac{\beta}{2} ||(\dot{\mathbf{Y}} -  \dot{\mathbf{X}} + \frac{ \dot{\mathbf{\eta}}}{\beta}) - \dot{\mathbf{Z}}||_{F}^{2}.
\end{aligned}
\label{z-rpca}
\end{equation}
According to the soft-threshold shrinkage, the solution for $\dot{\mathbf{Z}}$  is
\begin{equation}
\begin{aligned}
\dot{\mathbf{Z}} = \frac{\dot{\mathbf{Y}} -  \dot{\mathbf{X}} + \frac{ \dot{\mathbf{\eta}}}{\beta}}{\left|\dot{\mathbf{Y}} -  \dot{\mathbf{X}} + \frac{ \dot{\mathbf{\eta}}}{\beta}\right| + \epsilon }\max\left(\left|\dot{\mathbf{Y}} -  \dot{\mathbf{X}} + \frac{ \dot{\mathbf{\eta}}}{\beta}\right| - \frac{\rho}{\beta}, 0\right),
\end{aligned}
\label{z-s}
\end{equation}
where $|\cdot|$ is  the modulus of the quaternion and $\epsilon$ is a small number. The $\dot{\mathbf{X}}$ subproblem is consistent with (\ref{x-mc}). The whole algorithmic procedure is described in Algorithm \ref{ag3}.

\begin{algorithm}[t]
\SetAlgoNoLine
\SetKwInOut{Input}{\textbf{Input}}
\SetKwInOut{Output}{\textbf{Output}}

\Input{
    Observation $\dot{\mathbf{Y}}$; Initialize $\dot{\mathbf{X}}^{(0)}=\dot{\mathbf{Y}}$, $\dot{\mathbf{Z}}^{(0)}=0$, $\dot{\mathbf{\eta}}^{(0)}=0$; Parameters $\rho$, $\lambda$, $\beta$, $\mu$, $\alpha$, $K$;\\
    }

\Output{
    The reconstructed image $\dot{\mathbf{X}}^{(K)}$; \\
    }

    \For{$k=0,1,\dots,K-1$}{
    Update $\dot{\mathbf{Z}}^{(k+1)}$ by (\ref{z-s}) \;
    Update $\dot{\mathbf{X}}^{(k+1)}$ by (\ref{x-mc}) \;
    Update $\dot{\mathbf{\eta}}^{(k+1)} = \dot{\mathbf{\eta}}^{(k)} + \beta^{(k)}(\dot{\mathbf{Y}}- \dot{\mathbf{X}}^{(k+1)} - \dot{\mathbf{Z}}^{(k+1)})$ \;
    Update $\beta^{(k+1)} = \mu\beta^{(k)}$ \;
    }
    \caption{QNMF-RPCA Algorithm}
    \label{ag3}
\end{algorithm}

\begin{theorem}
Suppose that the parameter sequence $\{\beta^{(k)}\}$ is unbounded and the sequences $\{\dot{\mathbf{X}}^{(k)}\}$ and $\{\dot{\mathbf{Z}}^{(k)}\}$ are generated by Algorithm \ref{ag3}, which satisfies: 
\begin{align}
&\lim_{k\rightarrow +\infty} || \dot{\mathbf{X}}^{(k+1)}-\dot{\mathbf{X}}^{(k)} ||_{F} = 0, \label{converge31}\\
&\lim_{k\rightarrow +\infty} || \dot{\mathbf{Z}}^{(k+1)}-\dot{\mathbf{Z}}^{(k)} ||_{F} = 0. \label{converge32}\\
&\lim_{k\rightarrow +\infty} || \dot{\mathbf{Y}} - \dot{\mathbf{X}}^{(k+1)}-\dot{\mathbf{Z}}^{(k+1)} ||_{F} = 0. \label{converge33}
\end{align}
\label{theorem04}
\end{theorem}

The proof details of Theorem \ref{theorem04} are provided in the supplementary material.

\begin{figure}[t]
	\centering
	\addtolength{\tabcolsep}{-5.5pt}
    \renewcommand\arraystretch{0.4}
	{\fontsize{8pt}{\baselineskip}\selectfont 
	\begin{tabular}{cccccc}
        \includegraphics[width=0.15\textwidth]{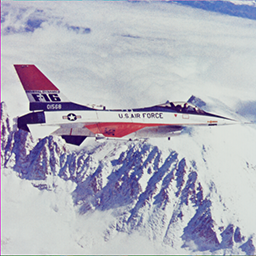} &
        \includegraphics[width=0.15\textwidth]{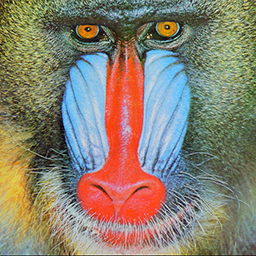} &
        \includegraphics[width=0.15\textwidth]{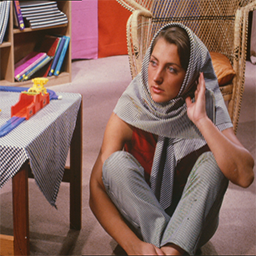} &
        \includegraphics[width=0.15\textwidth]{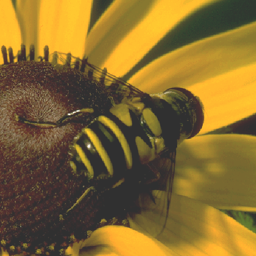} &
        \includegraphics[width=0.15\textwidth]{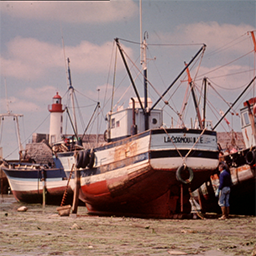} &
        \includegraphics[width=0.15\textwidth]{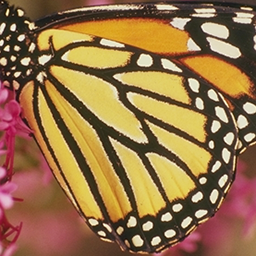} \\
        \includegraphics[width=0.15\textwidth]{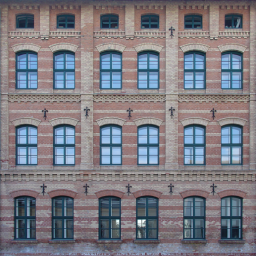} &
        \includegraphics[width=0.15\textwidth]{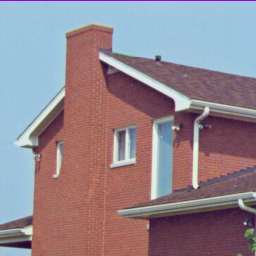} &
        \includegraphics[width=0.15\textwidth]{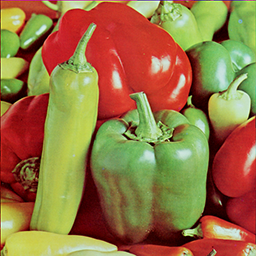} &
        \includegraphics[width=0.15\textwidth]{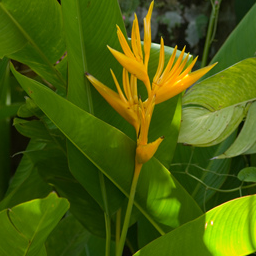} & \includegraphics[width=0.15\textwidth]{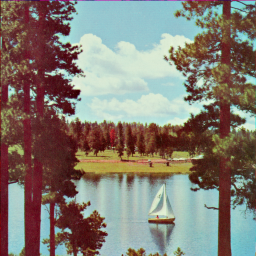} &
        \includegraphics[width=0.15\textwidth]{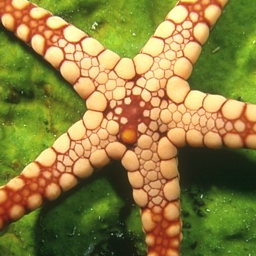} \\

	\end{tabular}
	}
	\caption{ CSet12. }
	\label{CSet12}
\end{figure}

\section{Experiments}
\label{sec:4}
\subsection{Experiment Setting}

In this section, we conducted experiments to verify the effectiveness and potential of the proposed QNMF in color image processing. We utilized 12 representative color images, each with dimensions of $256\times256$, collectively referred to as CSet12 (Figure \ref{CSet12}). To ensure the generalizability of our approach, we included the McMaster \cite{Dubois2005Frequency} and Kodak \cite{zhang2011color} datasets in our denoising experiments alongside CSet12. The McMaster dataset comprises 18 color images of size $500\times500$, while the Kodak dataset contains 24 color images of size $768\times512$. For real image denoising, we utilized three datasets: cc \cite{nam2016holistic}, PolyU \cite{xu2018real}, and SIDD \cite{Abdelhamed2018}, extracting 15 patches of dimensions $512\times512$ from each.

We assessed the restoration quality using peak signal-to-noise ratio (PSNR) and structural similarity (SSIM). Subsequently, we present the outcomes of our proposed algorithm across various tasks, including color image denoising, deblurring, matrix completion, and RPCA. All experiments were executed in MATLAB on a PC with 3.40 GHz and 64 GB RAM.

% 高斯
\subsection{Color image denoising} 
\label{Exp:Gaussian}

This subsection provides a comparative analysis of QNMF with other state-of-the-art color image denoising algorithms. The comparison includes six algorithms: CBM3D \cite{dabov2007color} and SV-TV \cite{jia2019color} for color space transformation; McWNNM \cite{xu2017multi} for vectorization; and QLRMA \cite{chen2019low}, QWNNM \cite{yu2019quaternion}, and QWSNM \cite{zhang2024quaternion} based on quaternionic representation. We utilized three datasets, CSet12, McMaster, and Kodak, where the images were corrupted by Gaussian noise with mean zero and standard deviation $\sigma = 5, 10, 20, 30, 40, 50, 60$.

The rank of an image is typically not exactly low but exhibits a long-tailed distribution. Consequently, directly applying low-rank regularization may not yield optimal results. According to the non-local self-similarity (NNS) prior for images, images can be represented by several sets of similar patches, each tending to be low-rank. Using the NNS prior for image denoising involves processing image patches instead of the entire image, as detailed in \cite{gu2014weighted} and \cite{guo2022gaussian}.
For various noise levels $\sigma$, we set the parameters of QNMF as follows: for $\sigma \leq 5$, $m=4$, $n=80$, $K=4$; for $5 < \sigma \leq 20$, $m=4$, $n=80$, $K=6$; for $20 < \sigma \leq 40$, $m=5$, $n=90$, $K=8$; for $40 < \sigma \leq 60$, $m=5$, $n=120$, $K=9$; and for $\sigma > 60$, $m=5$, $n=140$, $K=10$. Similar patches are sought within a $30 \times 30$ search window. The crucial parameters $\lambda$ and $\alpha$ are set to $\lambda=2c$, where $c=\sqrt{5\sqrt{2n}\sigma}$, and $\alpha=4$.

%%% GT
\begin{table*}[t]	
	\begin{center}
	\caption{
 Comparison of denoising PSNR/SSIM values in CSet12, McMaster and Kodak datasets. The best results are shown in \textbf{bold}.}
    \label{denoising}
    \resizebox{\textwidth}{!}{
    \addtolength{\tabcolsep}{ -2pt}
    \begin{tabular}{c|ccccccc}
    \hline
    \multicolumn{8}{c}{CSet12 (12)} \\ \hline
    
    $\sigma$ & CBM3D \cite{dabov2007color} & McWNNM \cite{xu2017multi} & SV-TV \cite{jia2019color} & QLRMA \cite{chen2019low} & QWSNM \cite{zhang2024quaternion} & QWNNM \cite{yu2019quaternion} & Ours \\ \hline
    5 & 38.51/0.9679 & 38.34/0.9656 & 34.89/0.9477 & 37.73/0.9624  & \textbf{39.15}/0.9719 & 39.13/\textbf{0.9726} & 39.14/0.9725 \\  
    10 & 34.74/0.9398 & 35.12/0.9428 & 32.17/0.9064 & 34.10/0.9310 & 35.44/0.9460 & 35.39/0.9477  & \textbf{35.50}/\textbf{0.9482} \\  
    20 & 31.31/0.8950 & 30.01/0.8685 & 28.93/0.8133 & 29.42/0.8443 & 31.89/0.9017 & 31.87/0.9054  & \textbf{32.00}/\textbf{0.9065} \\  
    30 & 29.38/0.8567 & 27.32/0.8026 & 26.64/0.7180 & 28.43/0.8218 & 30.00/0.8671 & 29.77/0.8671  & \textbf{30.02}/\textbf{0.8705} \\  
    40 & 28.06/0.8231 & 25.77/0.7498 & 24.87/0.6335 & 27.26/0.7916 & 28.58/0.8339 & 28.42/0.8358  & \textbf{28.61}/\textbf{0.8386} \\  
    50 & 27.01/0.7920 & 24.62/0.7040 & 23.43/0.5623 & 26.08/0.7546 & 27.44/0.8038 & 27.50/0.8106  & \textbf{27.54}/\textbf{0.8126} \\  
    60 & 26.18/0.7647 & 23.74/0.6668 & 22.20/0.5024 & 25.36/0.7282 & 26.59/0.7761 & 26.67/0.7845  & \textbf{26.71}/\textbf{0.7861} \\ 
    Av.& 30.74/0.8627 & 29.27/0.8143 & 27.59/0.7262 & 29.77/0.8334 & 31.30/0.8715 & 31.25/0.8748 & \textbf{31.36}/\textbf{0.8764} \\ \hline

    \multicolumn{8}{c}{McMaster (18)} \\ \hline
    
    $\sigma$ & CBM3D \cite{dabov2007color} & McWNNM \cite{xu2017multi} & SV-TV \cite{jia2019color} & QLRMA \cite{chen2019low} & QWSNM \cite{zhang2024quaternion} & QWNNM \cite{yu2019quaternion} & Ours \\ \hline
    5 & 39.18/0.9606 & 39.09/0.9587 & 35.14/0.9345 & 38.56/0.9558 & 39.90/0.9660 & 39.89/0.9671  & \textbf{39.96}/\textbf{0.9675} \\  
    10 & 35.91/0.9336 & 36.29/0.9376 & 33.05/0.8988 & 35.36/0.9263 & 36.62/0.9422 & 36.56/0.9436  & \textbf{36.70}/\textbf{0.9445} \\ 
    20 & 32.71/0.8896 & 31.65/0.8712 & 30.26/0.8209 & 31.59/0.8720 & 33.24/0.8995 & 33.26/0.9038  & \textbf{33.38}/\textbf{0.9051} \\
    30 & 30.85/0.8516 & 29.15/0.8146 & 28.32/0.7419 & 30.02/0.8240 & 31.46/0.8672 & 31.29/0.8684  & \textbf{31.48}/\textbf{0.8701} \\ 
    40 & 29.53/0.8190 & 27.60/0.7689 & 26.79/0.6683 & 28.93/0.7998 & 30.05/0.8360 & 29.95/0.8385  & \textbf{30.08}/\textbf{0.8388} \\ 
    50 & 28.54/0.7903 & 26.43/0.7296 & 25.54/0.6026 & 27.85/0.7683 & 28.99/0.8100 & 29.09/0.8161  & \textbf{29.14}/\textbf{0.8171} \\  
    60 & 27.73/0.7636 & 25.54/0.6964 & 24.45/0.5441 & 27.17/0.7478 & 28.11/0.7847 & 28.21/0.7910  & \textbf{28.28}/\textbf{0.7916} \\ 
    Av. & 32.06/0.8583 & 30.82/0.8253 & 29.08/0.7444 & 31.35/0.8420 & 32.62/0.8722 & 32.61/0.8755 & \textbf{32.72}/\textbf{0.8764} \\ \hline

    \multicolumn{8}{c}{Kodak (24)} \\ \hline

    $\sigma$ & CBM3D \cite{dabov2007color} & McWNNM \cite{xu2017multi} & SV-TV \cite{jia2019color} & QLRMA \cite{chen2019low} & QWSNM \cite{zhang2024quaternion} & QWNNM \cite{yu2019quaternion} & Ours \\ \hline
    5 & 40.30/0.9704 & 39.26/0.9616 & 36.68/0.9440 & 38.30/0.9567 & 40.38/0.9695 & 40.45/0.9711  & \textbf{40.51}/\textbf{0.9714} \\  
    10 & 36.57/0.9433 & 36.22/0.9357 & 33.18/0.8896 & 34.83/0.9194 & 36.61/0.9409 & 36.66/0.9443  & \textbf{36.77}/\textbf{0.9452} \\ 
    20 & 33.00/0.8919 & 31.07/0.8414 & 29.75/0.7935 & 30.11/0.8308 & 32.90/0.8846 & 33.05/0.8938  & \textbf{33.13}/\textbf{0.8959} \\
    30 & 31.02/0.8471 & 28.72/0.7736 & 27.70/0.7076 & 29.57/0.8036 & 31.02/0.8393 & 31.05/0.8509  & \textbf{31.16}/\textbf{0.8525} \\ 
    40 & 29.69/0.8099 & 27.35/0.7243 & 26.18/0.6297 & 28.49/0.7662 & 29.64/0.7993 & 29.70/0.8121  & \textbf{29.79}/\textbf{0.8136} \\ 
    50 & 28.68/0.7785 & 26.33/0.6855 & 24.94/0.5607 & 27.54/0.7356 & 28.60/0.7667 & 28.77/0.7789  & \textbf{28.85}/\textbf{0.7847} \\ 
    60 & 27.90/0.7521 & 25.54/0.6552 & 23.89/0.5009 & 26.90/0.7096 & 27.79/0.7400 & 27.95/0.7510  & \textbf{28.06}/\textbf{0.7564} \\ 
    Av. & 32.45/0.8562 & 30.64/0.7968 & 28.90/0.7180 & 30.82/0.8174 & 32.42/0.8486 & 32.52/0.8574 & \textbf{32.61}/\textbf{0.8600} \\ \hline
    
    \end{tabular}}
	\end{center}

\end{table*} 

\begin{figure*}[!t]
	\centering
	\addtolength{\tabcolsep}{-5.5pt}
    \renewcommand\arraystretch{0.7}
	{\fontsize{10pt}{\baselineskip}\selectfont 
	\begin{tabular}{ccccccccc}
        \includegraphics[width=0.17\textwidth]{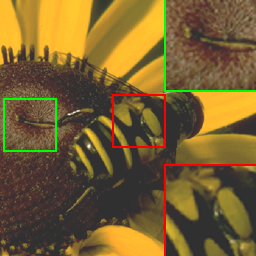} &
        \includegraphics[width=0.17\textwidth]{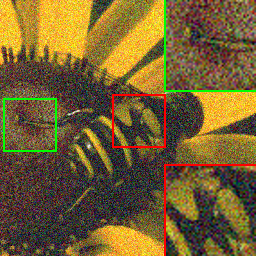} &
        \includegraphics[width=0.17\textwidth]{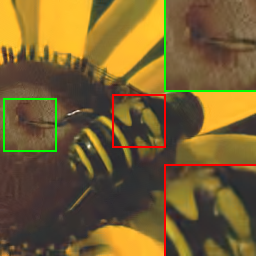} &
        \includegraphics[width=0.17\textwidth]{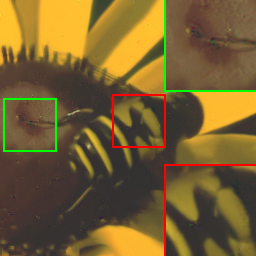} &
        \includegraphics[width=0.17\textwidth]{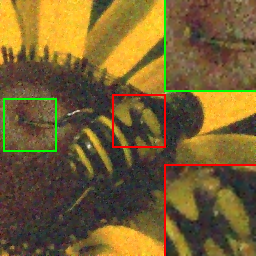} \\
        PSNR/SSIM & $\sigma=30$ & 31.65/0.8521 & 30.76/0.8496 & 27.86/0.6347 \\
        (a) GT & (b) NY & (c) CBM3D & (d) McWNNM & (e) SV-TV \\

        &
        \includegraphics[width=0.17\textwidth]{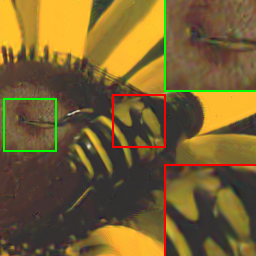} &
        \includegraphics[width=0.17\textwidth]{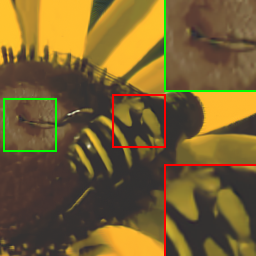} &
        \includegraphics[width=0.17\textwidth]{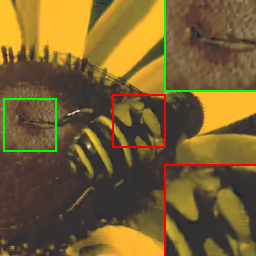} &
        \includegraphics[width=0.17\textwidth]{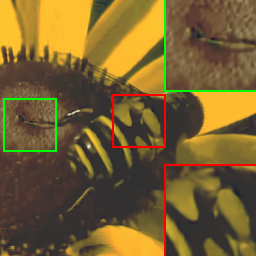} \\
        & 31.02/0.8213 & 32.47/0.8804 & 32.35/0.8809 & 32.57/0.8853 \\
        & (f) QLRMA & (g) QWSNM & (h) QWNNM & (i) Ours \\

        \includegraphics[width=0.17\textwidth]{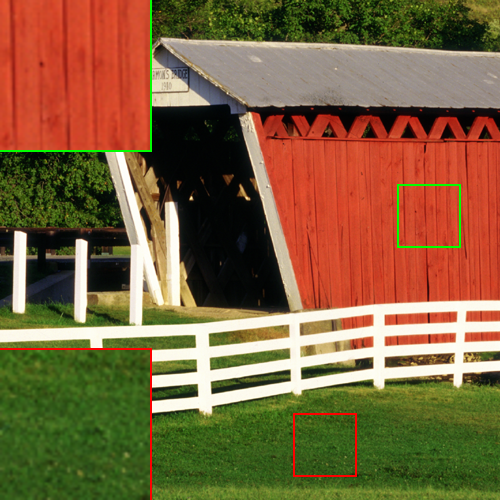} &
        \includegraphics[width=0.17\textwidth]{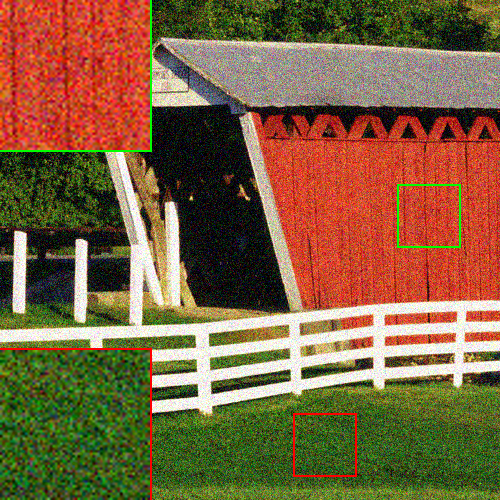} &
        \includegraphics[width=0.17\textwidth]{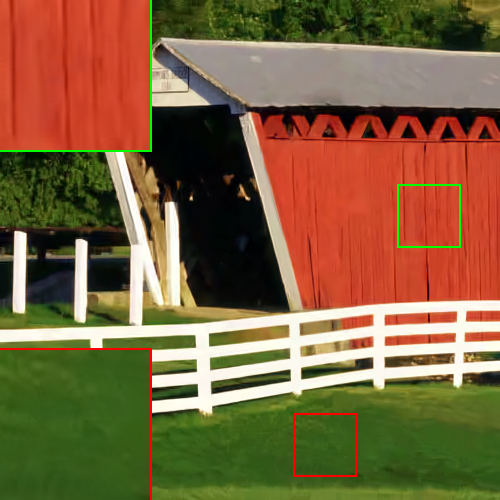} &
        \includegraphics[width=0.17\textwidth]{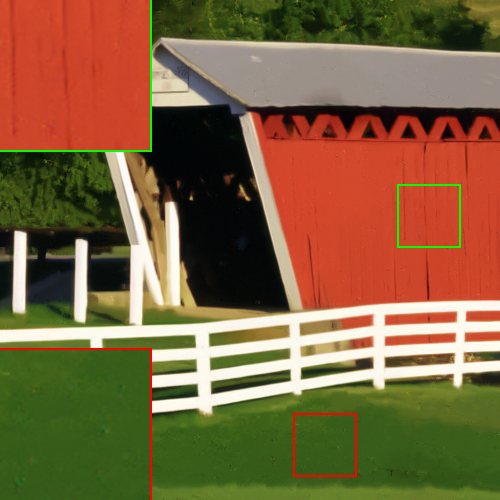} &
        \includegraphics[width=0.17\textwidth]{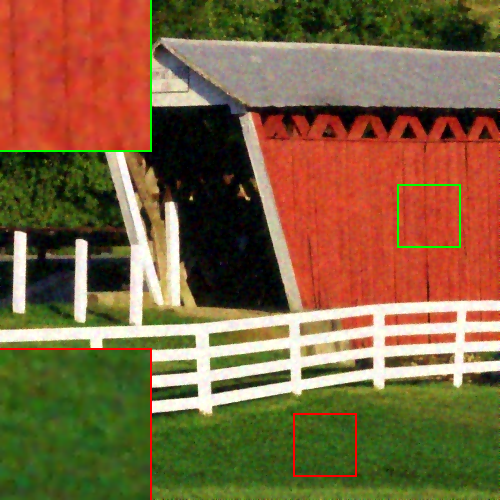} \\
        PSNR/SSIM & $\sigma=30$ & 29.86/0.8252 & 28.36/0.7768 & 27.77/0.7403 \\
        (a) GT & (b) NY & (c) CBM3D & (d) McWNNM & (e) SV-TV \\
        
        &
        \includegraphics[width=0.17\textwidth]{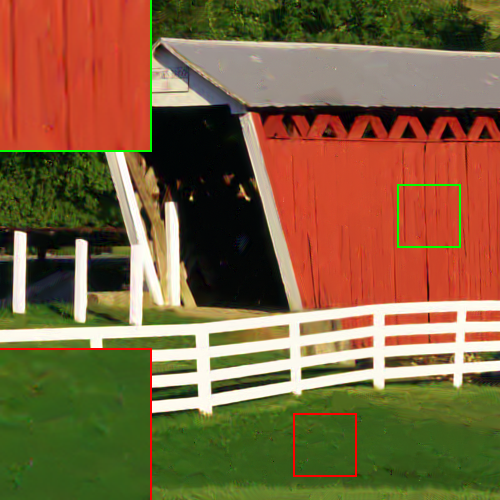} &
        \includegraphics[width=0.17\textwidth]{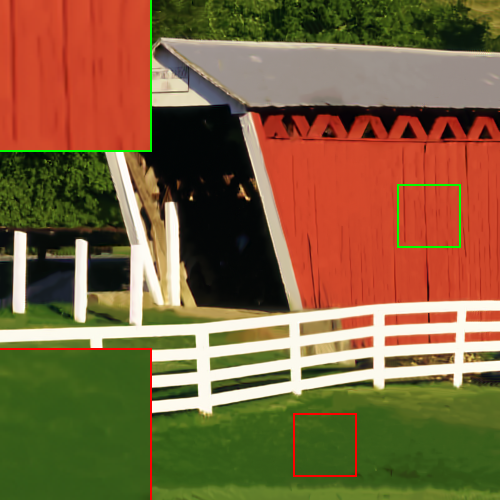} &
        \includegraphics[width=0.17\textwidth]{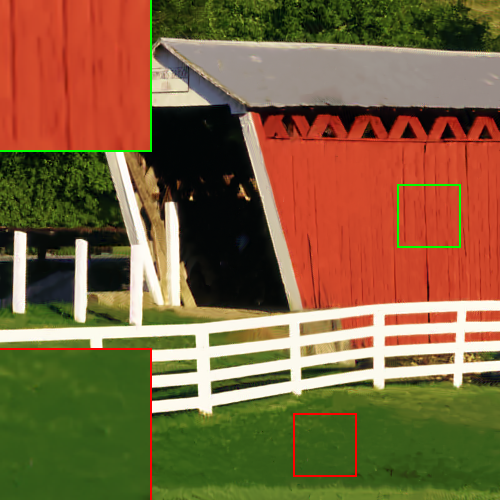} &
        \includegraphics[width=0.17\textwidth]{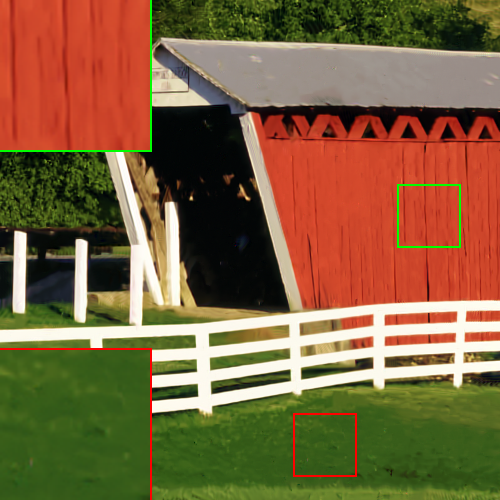} \\
        & 29.10/0.7946 & 30.37/0.8326 & 30.33/0.8417 & 30.54/0.8469 \\ 
        & (f) QLRMA & (g) QWSNM & (h) QWNNM & (i) Ours \\

    \end{tabular}
	}
	\caption{
    Comparison of visual results for color image denoising. The first two rows of images are from CSet12. The last two rows of images are from the McMaster dataset. 
 }
	\label{denoise_fig_1}
\end{figure*}

\begin{figure*}[!t]
	\centering
	\addtolength{\tabcolsep}{-5.5pt}
    \renewcommand\arraystretch{0.7}
	{\fontsize{10pt}{\baselineskip}\selectfont 
	\begin{tabular}{ccccccccc}
 
        \includegraphics[width=0.18\textwidth]{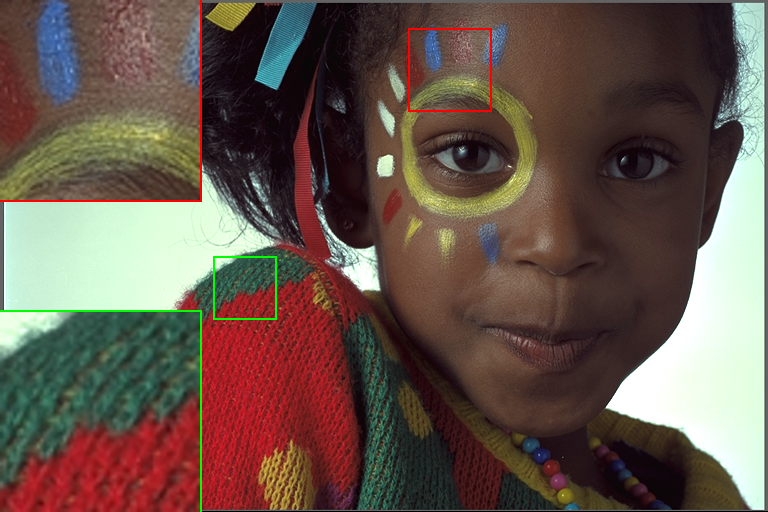} &
        \includegraphics[width=0.18\textwidth]{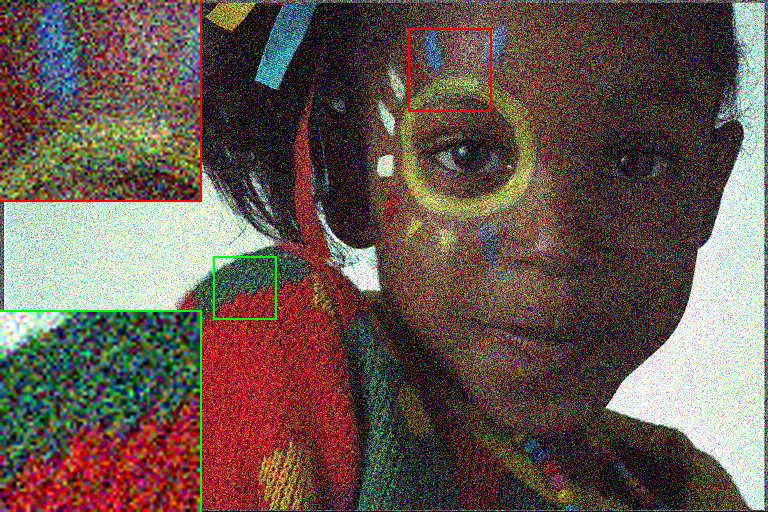} &
        \includegraphics[width=0.18\textwidth]{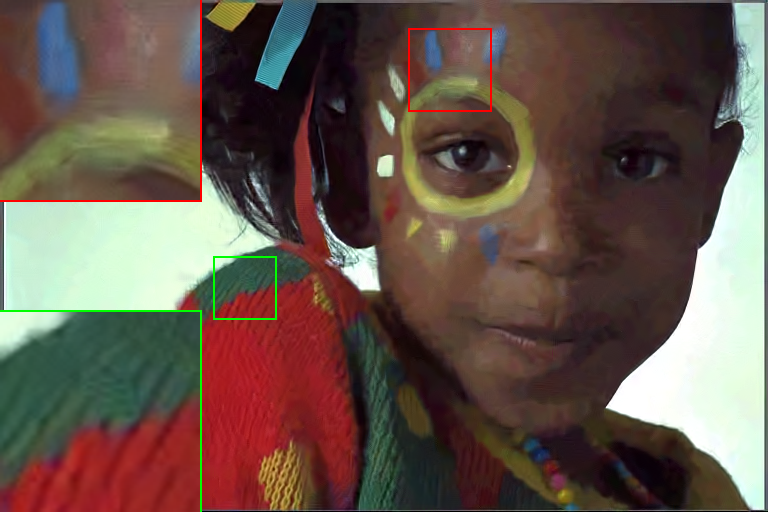} &
        \includegraphics[width=0.18\textwidth]{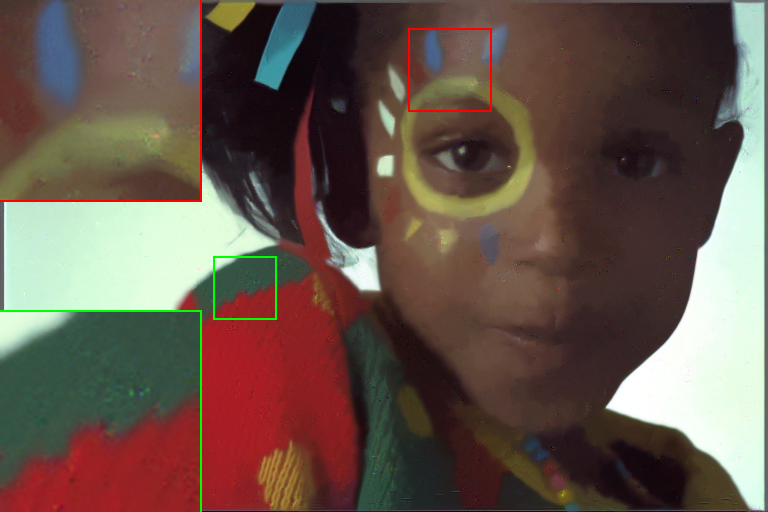} &
        \includegraphics[width=0.18\textwidth]{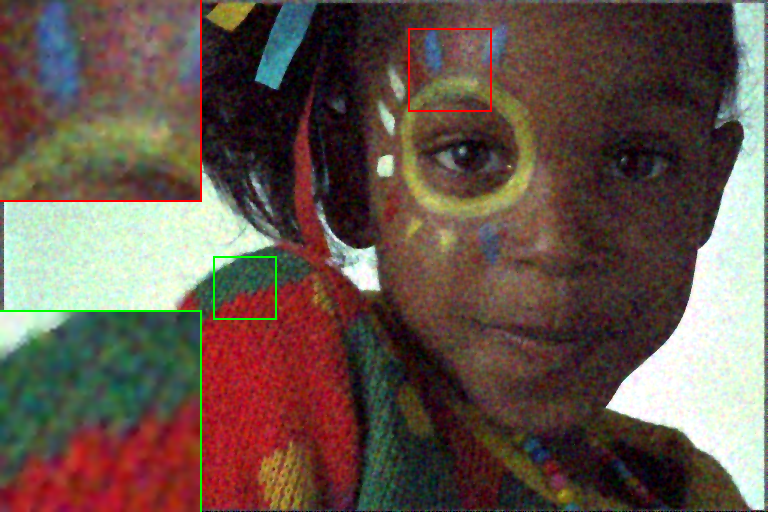} \\
        PSNR/SSIM & $\sigma=50$ & 30.22/0.8037 & 28.26/0.7495 & 26.02/0.5511  \\
        (a) GT & (b) NY & (c) CBM3D & (d) McWNNM & (e) SV-TV \\
        
        &
        \includegraphics[width=0.18\textwidth]{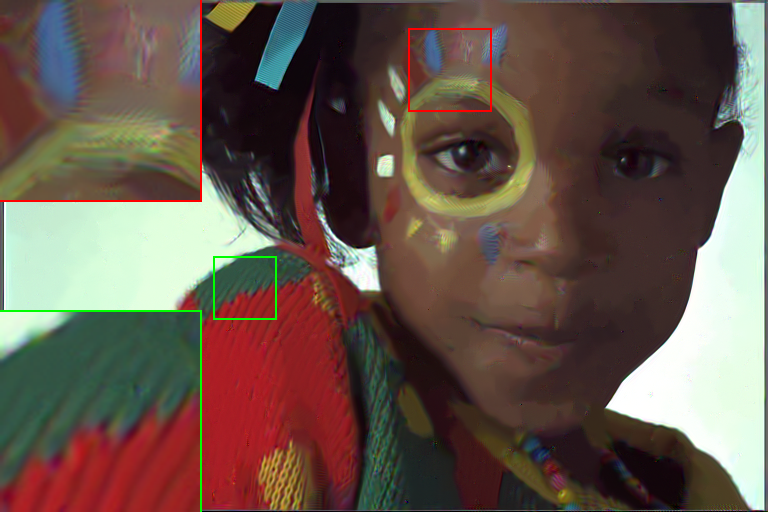} &
        \includegraphics[width=0.18\textwidth]{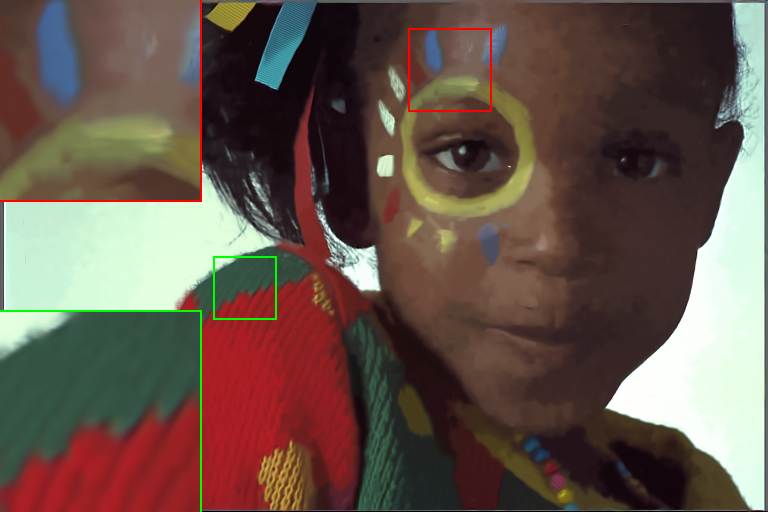} &
        \includegraphics[width=0.18\textwidth]{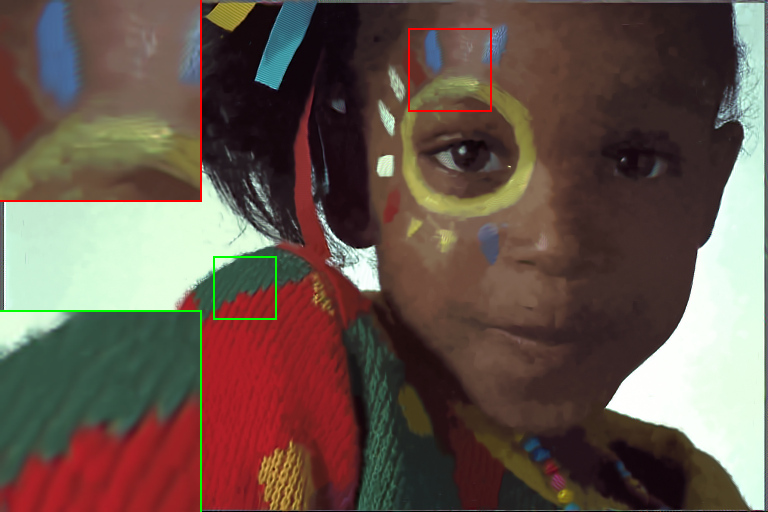} &
        \includegraphics[width=0.18\textwidth]{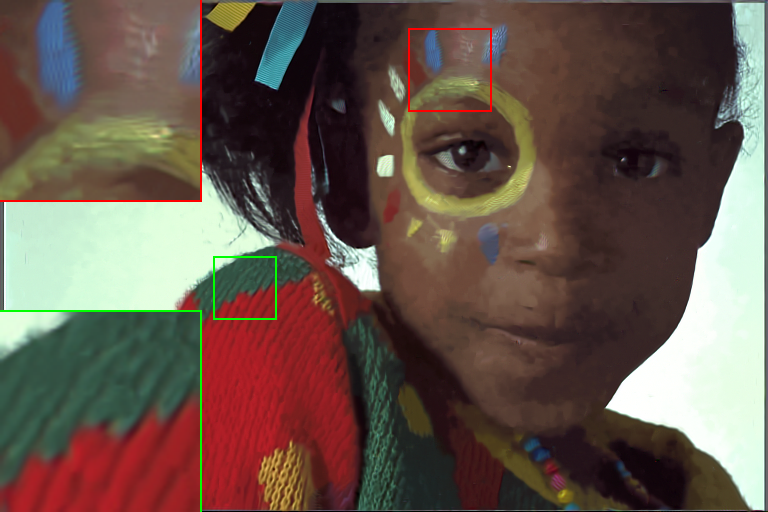} \\
        & 29.44/0.7819 & 30.25/0.8014 & 30.34/0.8075 & 30.48/0.8091 \\ 
        & (f) QLRMA & (g) QWSNM & (h) QWNNM & (i) Ours \\
        
        \includegraphics[width=0.18\textwidth]{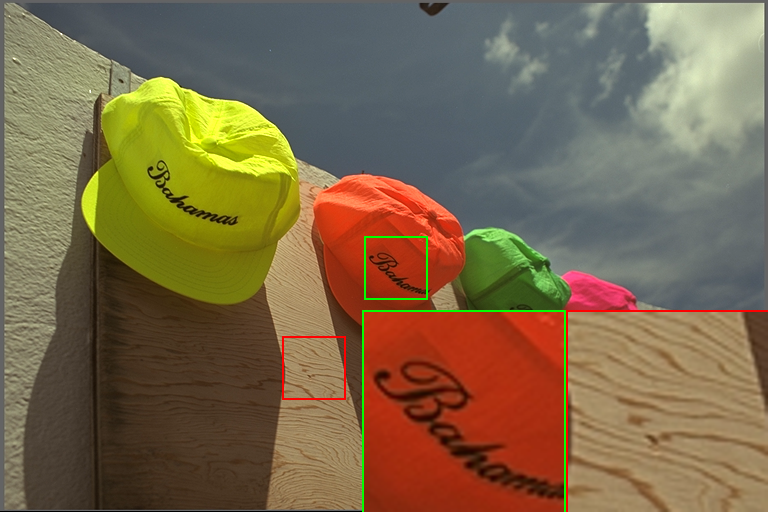} &
        \includegraphics[width=0.18\textwidth]{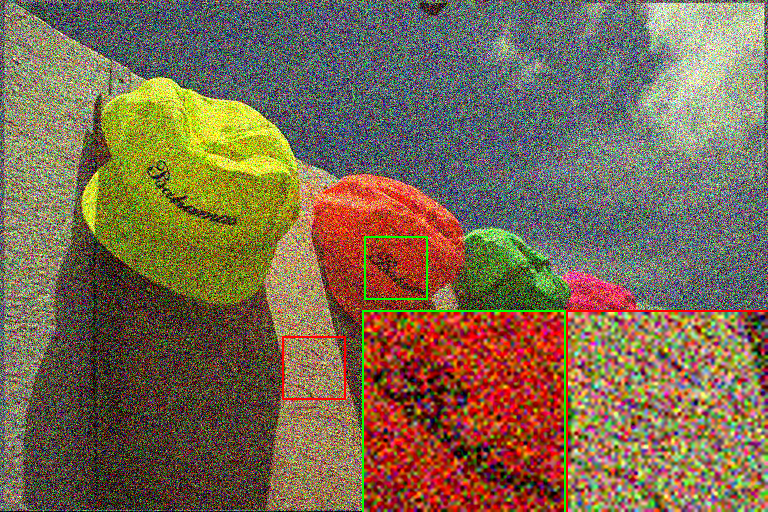} &
        \includegraphics[width=0.18\textwidth]{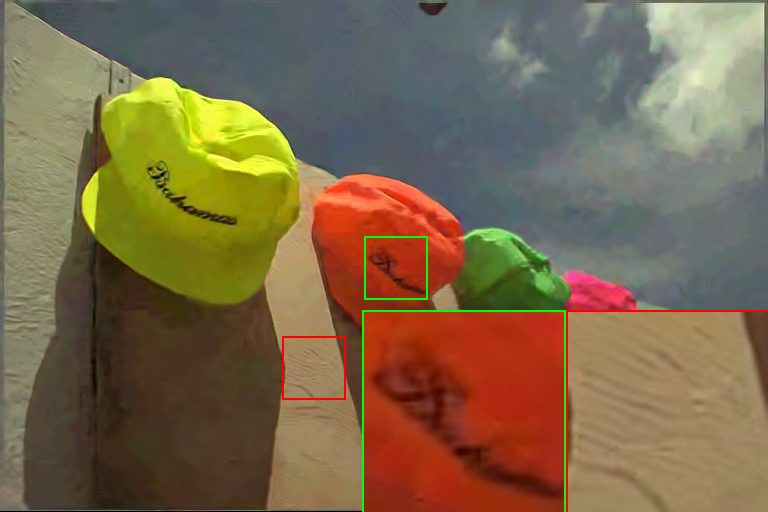} &
        \includegraphics[width=0.18\textwidth]{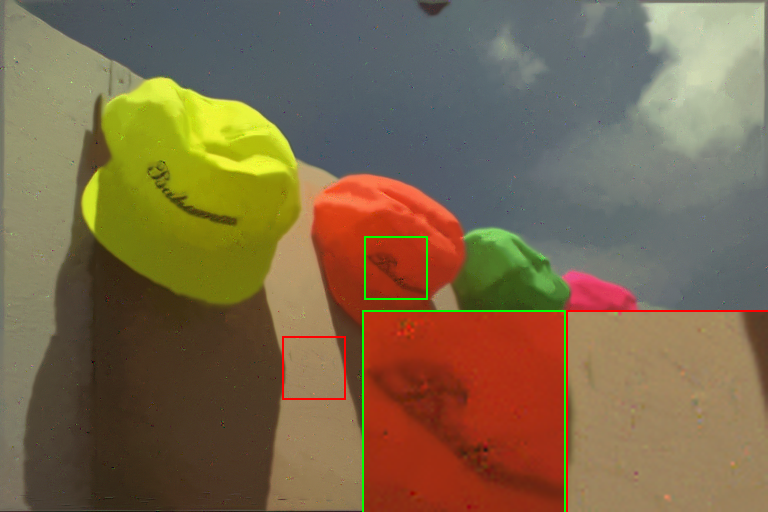} &
        \includegraphics[width=0.18\textwidth]{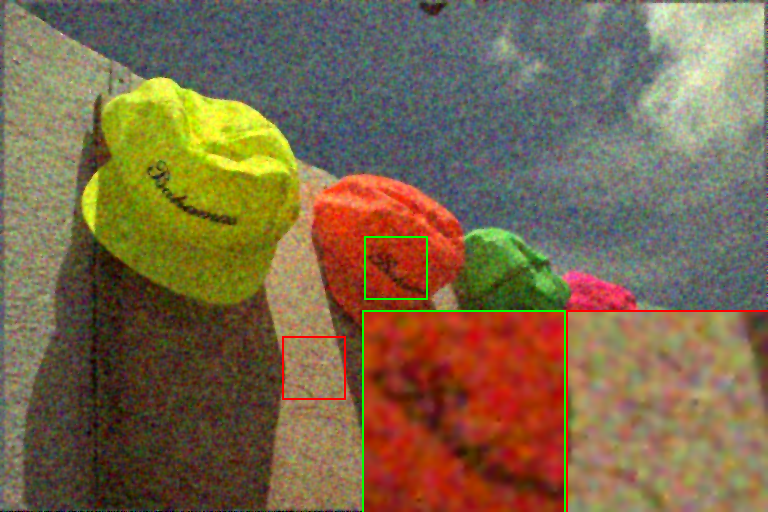} \\
        PSNR/SSIM & $\sigma=60$ & 30.52/0.8243 & 28.60/0.7630 & 25.07/0.4679 \\
        (a) GT & (b) NY & (c) CBM3D & (d) McWNNM & (e) SV-TV \\
        
        & 
        \includegraphics[width=0.18\textwidth]{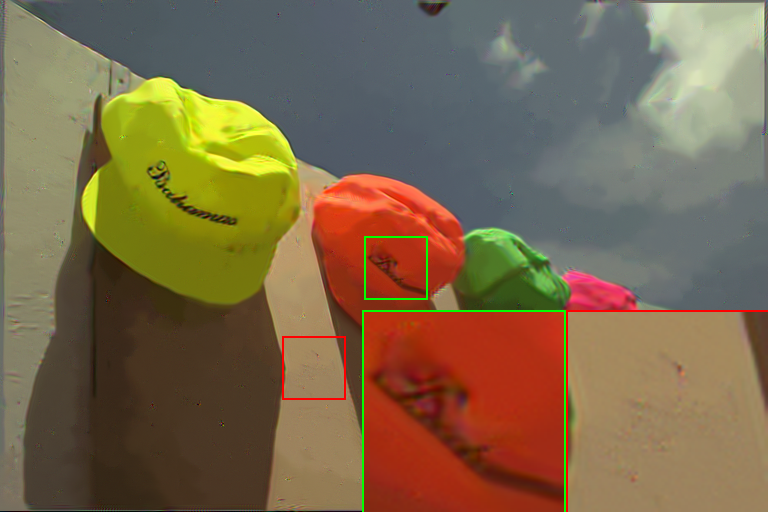} &
        \includegraphics[width=0.18\textwidth]{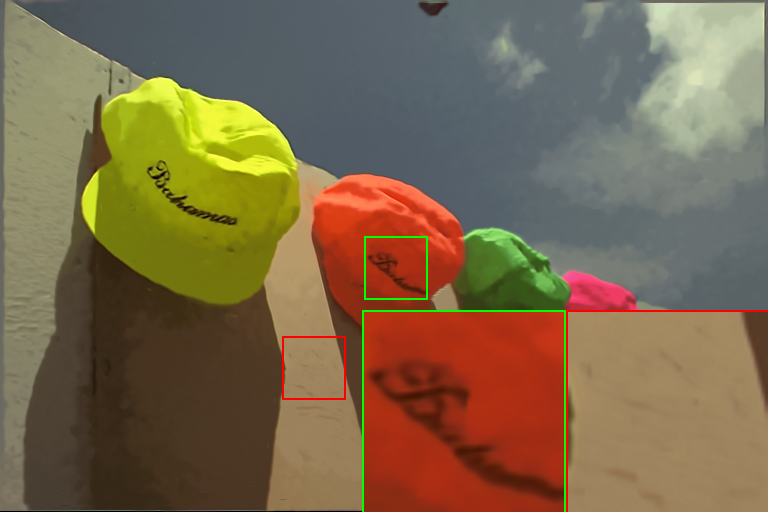} &
        \includegraphics[width=0.18\textwidth]{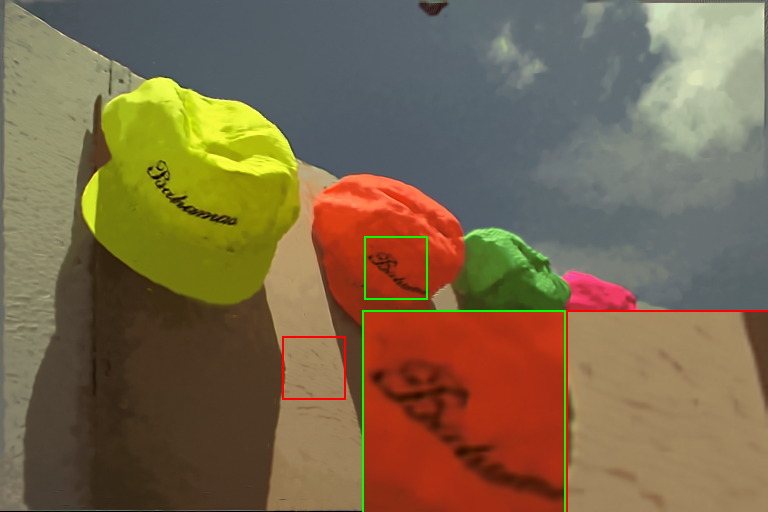} &
        \includegraphics[width=0.18\textwidth]{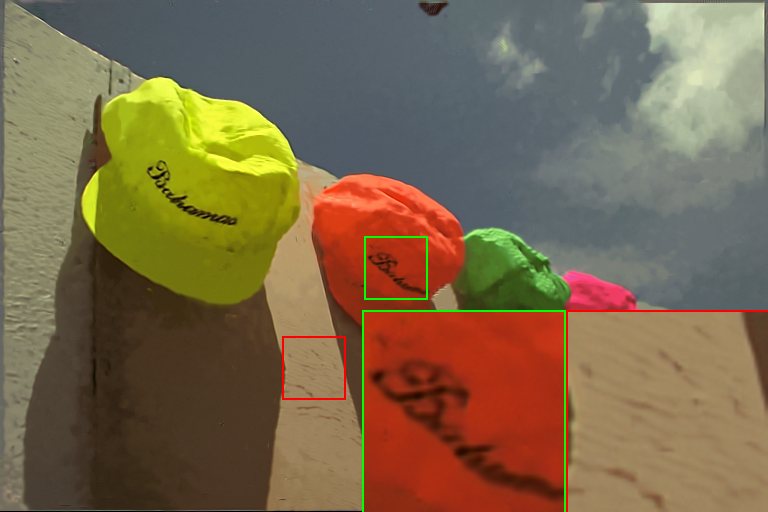} \\
        & 29.84/0.7994 & 30.52/0.8226 & 30.68/0.8278 & 30.82/0.8262 \\
        & (f) QLRMA & (g) QWSNM & (h) QWNNM & (i) Ours \\

	\end{tabular}
	}
	\caption{ 
    Comparison of visual results for color image denoising. All images are from the Kodak dataset.
 }
	\label{denoise_fig_2}
\end{figure*}

Table \ref{denoising} illustrates the comparative results of color image denoising, showcasing the consistent state-of-the-art performance of the proposed QNMF across various datasets and noise levels. In general, the algorithm based on quaternion representation outperforms methods focusing on vectorization or color space transformation. Specifically, among quaternion representation-based approaches, QWNNM excels in handling high noise levels ($\sigma \geq 50$), while QWSNM is particularly effective for low to medium noise levels ($\sigma < 50$). However, the proposed QNMF demonstrates robust performance across all noise levels.

The visual comparison of reconstructed images is depicted in Figures \ref{denoise_fig_1} and \ref{denoise_fig_2}. SV-TV exhibits significant noise residuals. CBM3D, another color space transformation algorithm, generates numerous color artifacts, particularly noticeable in smooth image regions like the sky in the hat image (in Figure \ref{denoise_fig_2}). These abrupt color changes stem from processing errors in smooth areas, which are highly discernible to the human eye, significantly influencing the perceived image quality. McWNNM, utilizing vectorization, produces fewer color artifacts but tends to oversmooth the image, leading to a loss of texture detail. Algorithms based on quaternion representation offer notable advantages in image color, attributed to the transformation and computation of inter-channel information facilitated by quaternion operations. Among these quaternion-based approaches, QNMF stands out for superior color and detail preservation, whereas both QWNNM and QWSNM suffer from excessive smoothing.

\begin{table*}[t]	
	\begin{center}
	\caption{Comparison of real denoising PSNR/SSIM values in cc, PolyU and SIDD datasets. The best results are shown in \textbf{bold}.}
    \label{real}
    \resizebox{\textwidth}{!}{
    \renewcommand\arraystretch{1}
    % \addtolength{\tabcolsep}{ -2pt}
    \begin{tabular}{l|ccccccc}
    \hline

    Real dataset & CBM3D \cite{dabov2007color} & McWNNM \cite{xu2017multi} & SV-TV \cite{jia2019color} & QLRMA \cite{chen2019low} & QWSNM \cite{zhang2024quaternion} & QWNNM \cite{yu2019quaternion} & Ours \\ \hline
    cc \cite{nam2016holistic}   & 36.62/0.9379 & 35.99/0.9247 & 35.72/0.9348 & 36.29/0.9346 & 37.75/0.9555 & 37.69/0.9542 & \textbf{37.83}/\textbf{0.9566} \\ 
    PolyU \cite{xu2018real}     & 37.93/0.9630 & 36.26/0.9513 & 37.15/0.9536 & 32.79/0.9115 & 37.88/0.9612 & 38.06/0.9622 & \textbf{38.10}/\textbf{0.9625} \\ 
    SIDD \cite{Abdelhamed2018}   & 36.10/0.9093 & 34.40/0.8720 & 32.84/0.8147 & 32.49/0.7971 & 36.44/0.9137  & 36.49/0.9155 & \textbf{36.53}/\textbf{0.9166} \\ 
    Av. & 36.88/0.9367 & 35.55/0.9160 & 35.24/0.9010 & 33.86/0.8811 & 37.36/0.9435 & 37.41/0.9440 & \textbf{37.49}/\textbf{0.9452} \\ \hline

    \end{tabular}
    }
	\end{center}

\end{table*} 
\begin{figure*}[!t]
	\centering
	\addtolength{\tabcolsep}{-5.5pt}
    \renewcommand\arraystretch{0.7}
	{\fontsize{10pt}{\baselineskip}\selectfont 
	\begin{tabular}{ccccccccc}
        \includegraphics[width=0.17\textwidth]{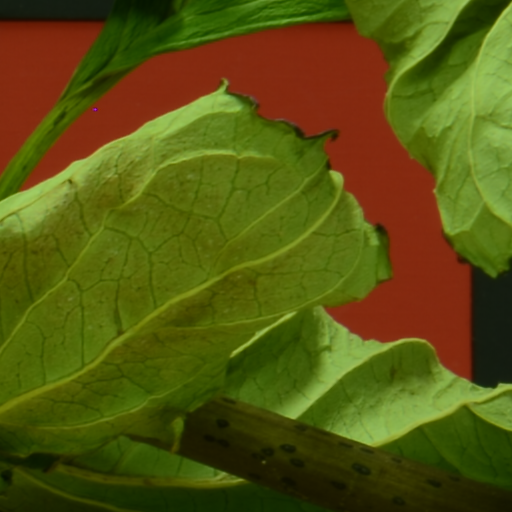} &
        \includegraphics[width=0.17\textwidth]{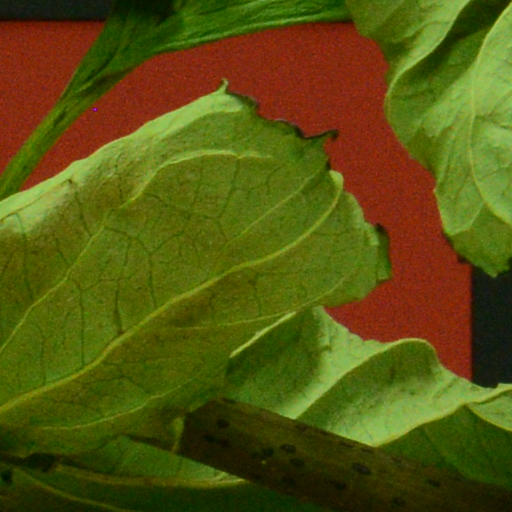} &
        \includegraphics[width=0.17\textwidth]{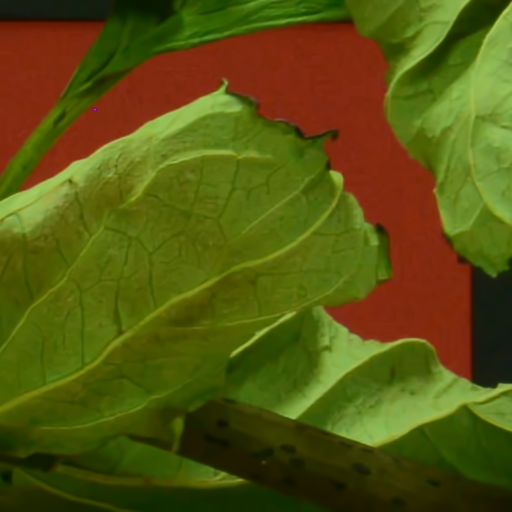} &
        \includegraphics[width=0.17\textwidth]{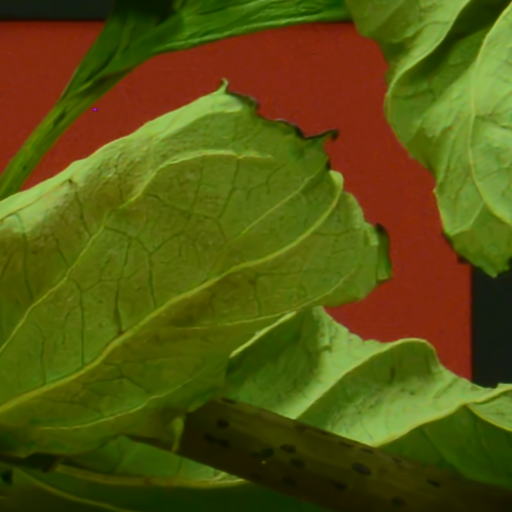} &
        \includegraphics[width=0.17\textwidth]{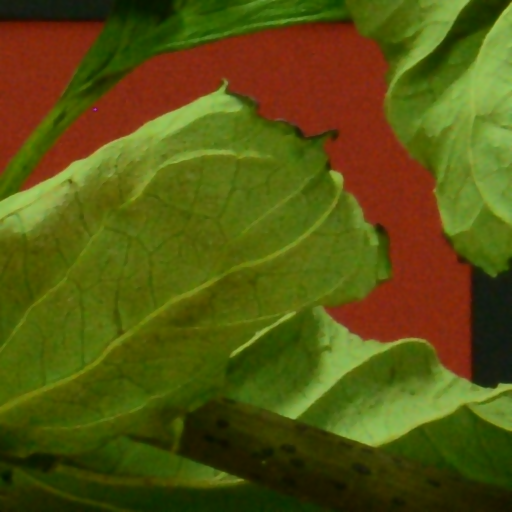} \\
        (a) GT & (b) NY & (c) CBM3D & (d) McWNNM & (e) SV-TV \\

        &
        \includegraphics[width=0.17\textwidth]{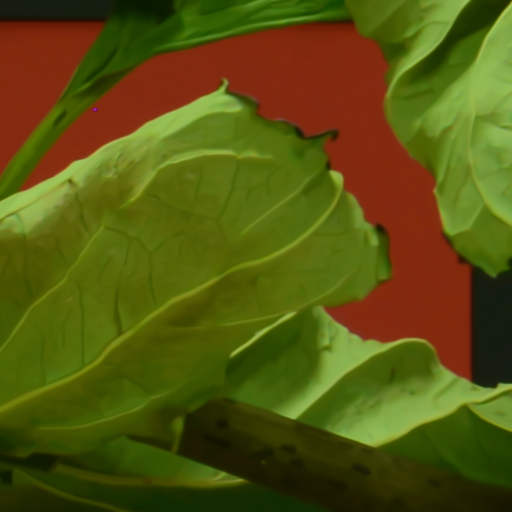} &
        \includegraphics[width=0.17\textwidth]{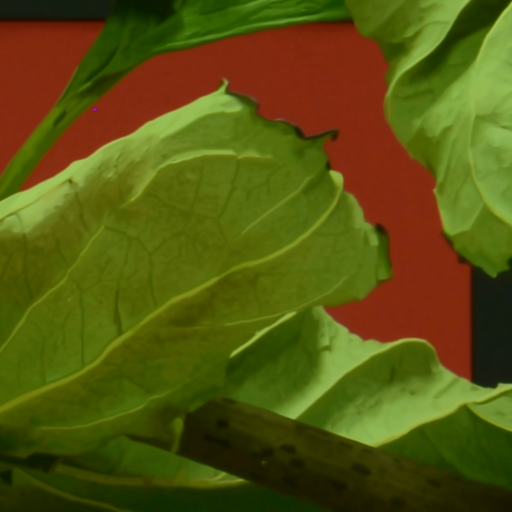} &
        \includegraphics[width=0.17\textwidth]{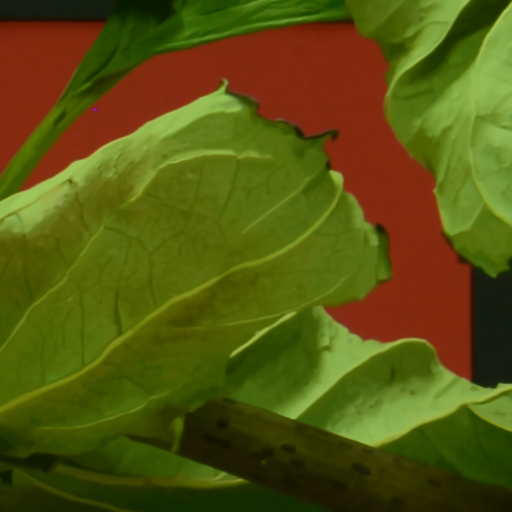} &
        \includegraphics[width=0.17\textwidth]{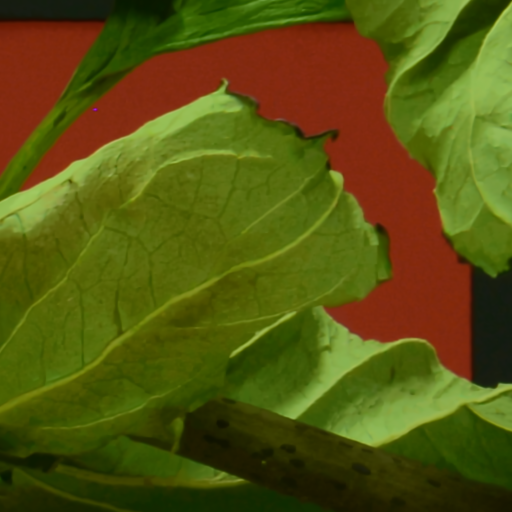} \\
        & (f) QLRMA & (g) QWSNM & (h) QWNNM & (i) Ours \\

        \includegraphics[width=0.17\textwidth]{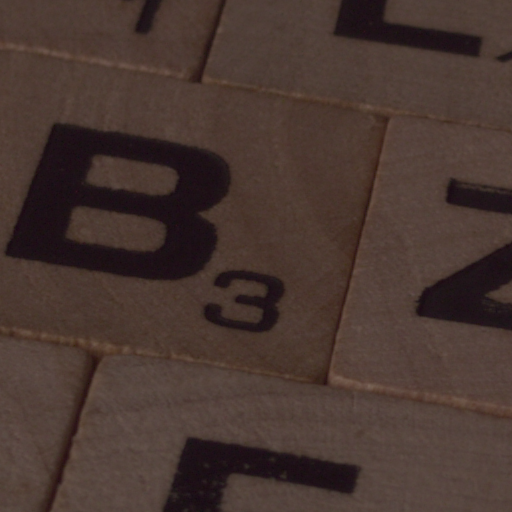} &
        \includegraphics[width=0.17\textwidth]{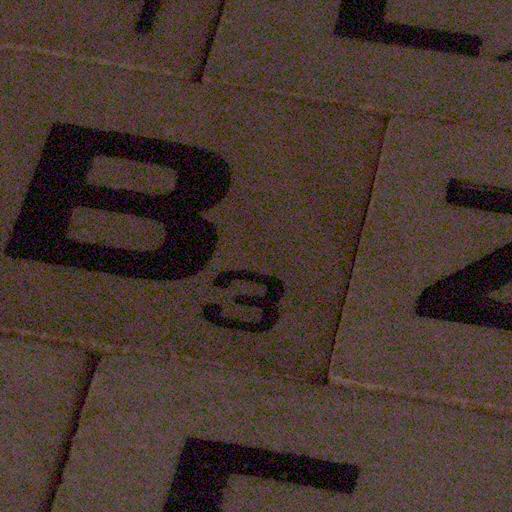} &
        \includegraphics[width=0.17\textwidth]{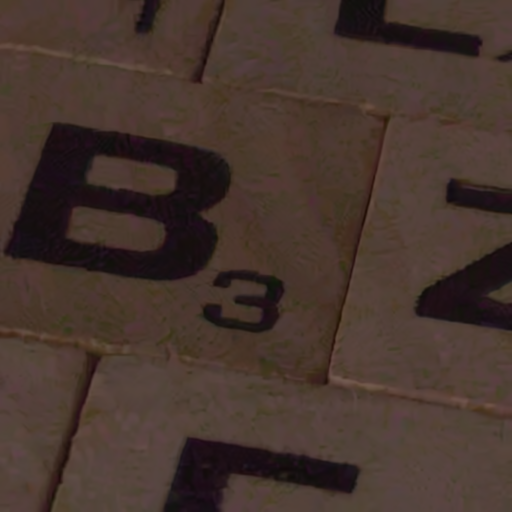} &
        \includegraphics[width=0.17\textwidth]{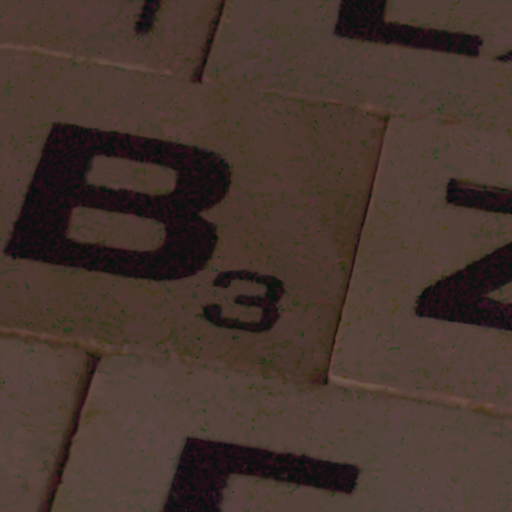} &
        \includegraphics[width=0.17\textwidth]{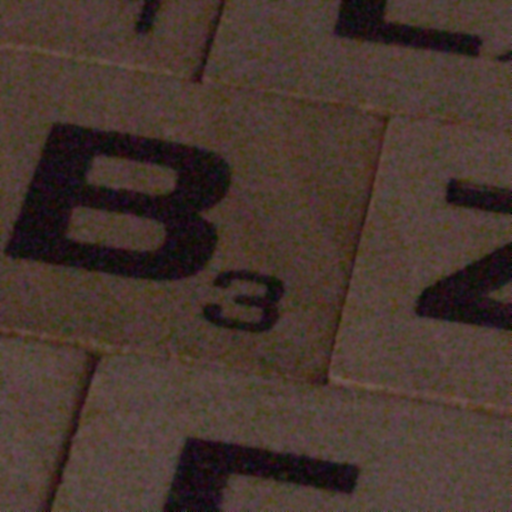} \\
        (a) GT & (b) NY & (c) CBM3D & (d) McWNNM & (e) SV-TV \\

        &
        \includegraphics[width=0.17\textwidth]{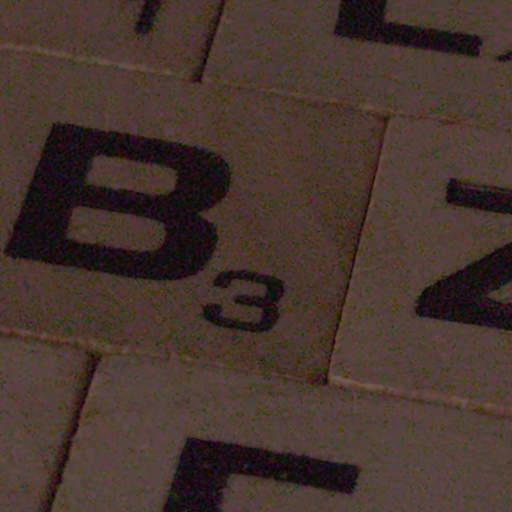} &
        \includegraphics[width=0.17\textwidth]{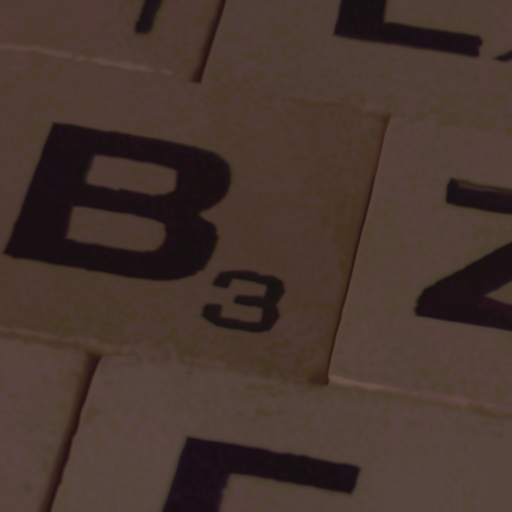} &
        \includegraphics[width=0.17\textwidth]{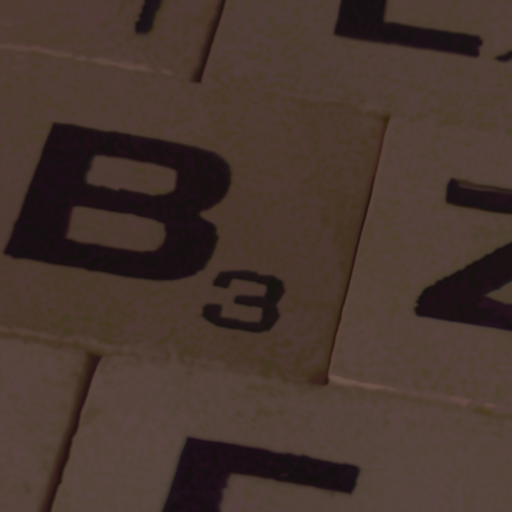} &
        \includegraphics[width=0.17\textwidth]{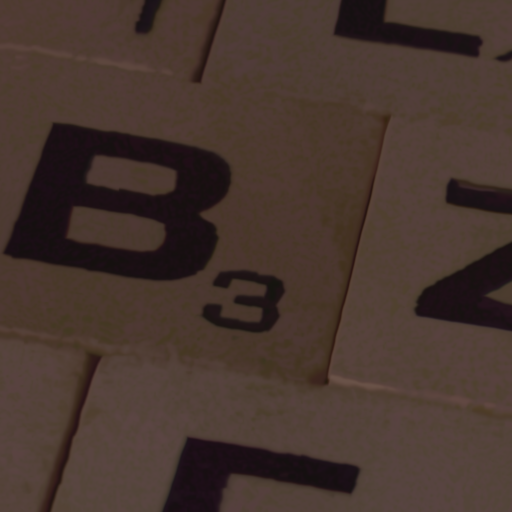} \\
        & (f) QLRMA & (g) QWSNM & (h) QWNNM & (i) Ours \\

	\end{tabular}
	}
	\caption{ Comparison of visual results for real color image denoising. The first two rows of images are from the cc dataset. The last two rows of images are from the SIDD dataset. }
	\label{Real_fig}
\end{figure*}

% 真实
\subsection{Real image denoising} 
In real image denoising, the algorithm setup remains consistent with the previous subsection. To address blind denoising, where the noise intensity level is unknown, we employed the noise intensity estimator proposed by \cite{Chen2015Efficient} to estimate the $\sigma$ parameter for each algorithm.

\begin{table}[t]	
	\caption{Comparison of debluring PSNR/SSIM values in CSet12 datasets. The best results are shown in \textbf{bold}.}
    \label{deblur}
    
	\begin{center}
    % \resizebox{\textwidth}{!}{
    \renewcommand\arraystretch{1}
    \footnotesize
    \addtolength{\tabcolsep}{ -4pt}
    \begin{tabular}{l|ccc}
    \hline

     Blur kernel & UB(9,9)/$\sigma=15$ & GB(25,1.6)/$\sigma=15$ & MB(20,60)/$\sigma=15$ \\ \hline
     SNSS \cite{zha2020image} & 20.02/0.5511 & 20.74/0.6438 & 19.95/0.5571 \\ 
     SV-TV \cite{jia2019color} & 22.57/0.6030 & 16.09/0.3634 & 19.77/0.4870 \\ 
     QD-SV-TV \cite{huang2021quaternion} & 21.37/0.5551 & 22.75/0.6110 & 19.79/0.4944 \\ 
     QWSNM \cite{zhang2024quaternion} & 24.42/0.6879 & 26.15/0.7642 & 23.83/0.6687 \\ 
     QWNNM \cite{huang2022quaternion} & 24.39/0.6862 & 26.02/0.7572 & 24.07/0.6819 \\ 
     Ours  & \textbf{24.45}/\textbf{0.6908} & \textbf{26.20}/\textbf{0.7665} & \textbf{24.19}/\textbf{0.6857} \\ \hline

    \end{tabular}
    % }
	\end{center}
\end{table}

\begin{figure}[!t]
	\centering
	\addtolength{\tabcolsep}{-5.5pt}
    \renewcommand\arraystretch{0.7}
	{\fontsize{10pt}{\baselineskip}\selectfont 
	\begin{tabular}{cccc}
        \includegraphics[width=0.17\textwidth]{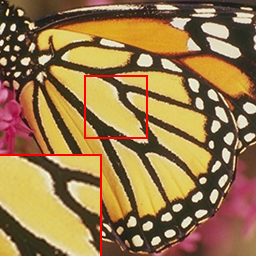} &
        \includegraphics[width=0.17\textwidth]{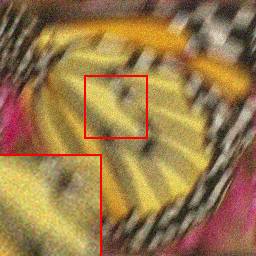} &
        \includegraphics[width=0.17\textwidth]{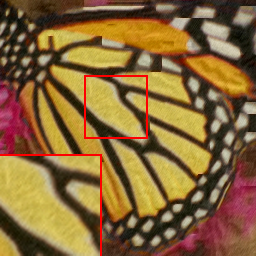} &
        \includegraphics[width=0.17\textwidth]{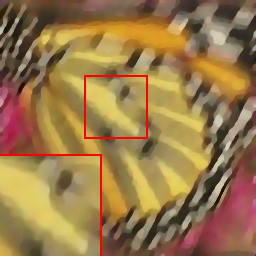}  \\

        (a) GT & (b) Blur & (c) SNSS & (d) SV-TV   \\
        
         \includegraphics[width=0.17\textwidth]{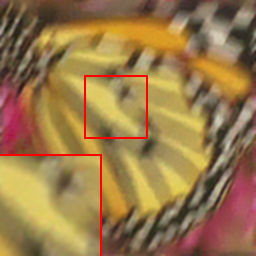} &
         \includegraphics[width=0.17\textwidth]{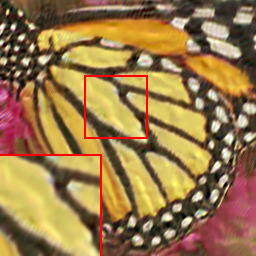} &
        \includegraphics[width=0.17\textwidth]{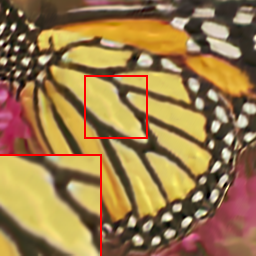} &
        \includegraphics[width=0.17\textwidth]{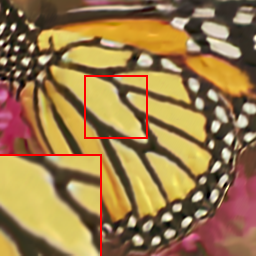} \\
       (e) QD-SV-TV & (f) QWSNM & (g) QWNNM & (h) Ours \\

	\end{tabular}
	}
	\caption{ 
     Comparison of visual results of color image deblurring. The blurring kernel is MB(20,60)/$\sigma=15$.
 }
	\label{Deblur_fig}
\end{figure}

Table \ref{real} illustrates the results on three real image datasets, indicating that the proposed QNMF consistently performs well and achieves state-of-the-art results across all three datasets. This aligns with its performance in the denoising experiments discussed earlier. The visual comparison is presented in Figure \ref{Real_fig}. QNMF maintains its superiority, preserving more image details compared to other methods. For example, our algorithm better preserves leaf texture while effectively eliminating noise (see the first two rows of Figure \ref{Real_fig}).

\subsection{Color image debluring} 

Taking color image deblurring as an example, we explored the potential of the proposed QNMF in low-level vision tasks. We considered three common blur kernels: a $9\times9$ uniform kernel, a Gaussian kernel with a standard deviation of 1.6, and the motion kernel with a motion length of 20 and a motion angle of 60. All of these blurred images were affected by additive Gaussian noise with $\sigma=15$. Four algorithms, namely SV-TV \cite{jia2019color}, QWNNM \cite{huang2022quaternion}, QWSNM \cite{zhang2024quaternion}, QD-SV-TV based on quaternion dictionary learning \cite{huang2021quaternion}, and SNSS \cite{zha2020image} based on gaussian mixture model, were used for comparison.

The parameters of the deblurring algorithm were set as follows: $\beta=8.5$, $\gamma=115$ for uniform blur; $\beta=7.5$, $\gamma=65$ for Gaussian blur; $\beta=7.5$, $\gamma=115$ for motion blur. 155 similar patches of $6\times6$ were searched in the search window of 30. $\alpha$ was still set to 4.

Table \ref{deblur} displays the reconstruction performance of the comparison algorithms for the three types of blurs. QWSNM shows superior performance under uniform and Gaussian blurs, while QWNNM excels at handling motion blur. The proposed QNMF consistently achieves state-of-the-art results across all three blurs. In Figure \ref{Deblur_fig}, a visual comparison between the proposed QNMF and state-of-the-art algorithms under motion blur is presented. The reconstruction outcomes of SV-TV and QD-SV-TV are unsatisfactory. Both QWNNM and QWSNM exhibit reconstruction artifacts, whereas the proposed QNMF demonstrates smooth reconstruction results. These findings underscore the remarkable potential of QNMF in addressing the color image linear inverse problem.
 
\begin{table}[t]	
	\caption{Comparison of PSNR/SSIM values for color image matrix completion. The best results are marked in \textbf{bold}. The second best results are marked in {\color{red} red}. Ours-G indicates processing on the global image and Ours-NNS indicates processing using image patches.}
    \label{mc}
    
	\begin{center}
    \footnotesize
    \renewcommand\arraystretch{1}
    \addtolength{\tabcolsep}{ -2pt}
    \begin{tabular}{l|cccc}
    \hline

    Miss rate & 50\% & 60\% & 75\% & 80\%  \\ \hline
    TRPCA \cite{yang2020low} & 24.26/{\color{red} 0.7636} & 22.97/0.6901 & 20.70/0.5397 & 19.71/0.4745 \\ 
    QMC \cite{jia2019robust} & 23.76/0.7594 & 22.75/{\color{red} 0.6940} & 20.73/{\color{red} 0.5504} & 19.82/{\color{red} 0.4866} \\ 
    QWSNM \cite{zhang2024quaternion} & 25.47/0.6993 & 23.62/0.6124 & 19.14/0.3827 & 14.42/0.2097 \\ 
    QWNNM \cite{huang2022quaternion} & 24.63/0.6577 & 22.80/0.5680 & 19.60/0.3942 & 17.04/0.2769 \\ 
    Ours-G & {\color{red} 26.27}/0.7594 & {\color{red} 24.36}/0.6717 & {\color{red} 21.52}/0.5083 & {\color{red} 20.37}/0.4415 \\ 
    Ours-NNS & \textbf{31.80}/\textbf{0.9397} & \textbf{29.85}/\textbf{0.9114} & \textbf{27.05}/\textbf{0.8512} & \textbf{23.81}/\textbf{0.7714} \\ \hline

    \end{tabular}
	\end{center}

\end{table} 
\begin{table}[!t]	
	\caption{Comparison of PSNR/SSIM values for color image random impulse noise removal. The best results are marked in \textbf{bold}. The second best results are marked in {\color{red} red}. Ours-G indicates processing on the global image and Ours-NNS indicates processing using image patches.}
    \label{rpca}
    
	\begin{center}
    \footnotesize
    \renewcommand\arraystretch{1}
    % \addtolength{\tabcolsep}{ -4pt}
    \begin{tabular}{l|ccccH}
    \hline

    Random rate & 3\% & 5\% & 7\% & 10\% & 20\%  \\ \hline
    TRPCA \cite{yang2020low} & 29.56/{\color{red} 0.9398} & 29.38/{\color{red} 0.9343} & 29.14/\textbf{0.9279} & 28.80/\textbf{0.9170} & {\color{red} 27.51}/\textbf{0.8677} \\ 
    QMC \cite{jia2019robust} & 27.53/0.9119 & 27.37/0.9062 & 27.23/0.9004 & 26.99/{\color{red} 0.8904} & 26.10/{\color{red} 0.8444} \\ 
    QWSNM \cite{zhang2024quaternion} & 24.32/0.7513 & 24.29/0.7489 & 24.25/0.7468 & 24.18/0.7438 & 23.91/0.7294 \\ 
    QWNNM \cite{huang2022quaternion} & 24.40/0.7582 & 24.39/0.7558 & 24.33/0.7534 & 24.28/0.7498 & 23.97/0.7338 \\ 
    Ours-G & \textbf{33.21}/\textbf{0.9462} & \textbf{32.07}/\textbf{0.9362} & \textbf{30.86}/{\color{red} 0.9164} & {\color{red} 29.34}/0.8895 & 25.85/0.7704 \\ 
    Ours-NNS & {\color{red} 32.70}/0.9093 & {\color{red} 30.98}/0.8815 & {\color{red} 30.32}/0.8617 & \textbf{29.58}/0.8582 & \textbf{27.74}/0.8194 \\ \hline

    \end{tabular}
	\end{center}

\end{table} 

\subsection{Matrix completion and Robust PCA}

In this subsection, we evaluated the performance of QNMF in matrix completion and RPCA. For matrix completion, we assessed observation missing rates of 50\%, 60\%, 75\%, and 80\%. For RPCA, we added 3\%, 5\%, 7\%, and 10\%  random impulse noise to the color images. We compared QNMF with tensor-based TRPCA \cite{yang2020low}, quaternion matrix completion (QMC) \cite{jia2019robust}, QWNNM, and QWSNM.  TRPCA, QMC, QWNNM, and QWSNM used global image computation. In contrast, QNMF utilized both NNS-based patch computation (Ours-NNS) and global image computation (Ours-G).
For parameter settings, refer to \cite{gu2017weighted,xie2016weighted}, with $\alpha$ set to 4.

\begin{figure}[t]
	\centering
	\addtolength{\tabcolsep}{-5.5pt}
    \renewcommand\arraystretch{0.6}
	{\fontsize{10pt}{\baselineskip}\selectfont 
	\begin{tabular}{cccc}
        \includegraphics[width=0.17\textwidth]{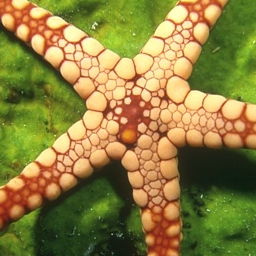} &
        \includegraphics[width=0.17\textwidth]{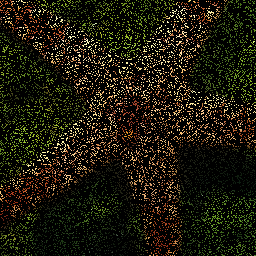} &
        \includegraphics[width=0.17\textwidth]{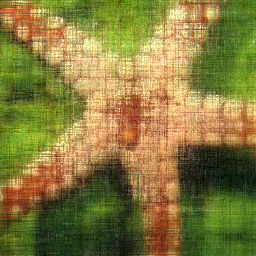} &
        \includegraphics[width=0.17\textwidth]{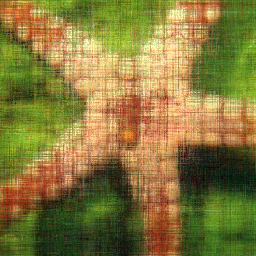} \\

        (a) GT & (b) Miss (80\%) & (c) TRPCA & (d) QMC \\
        
        \includegraphics[width=0.17\textwidth]{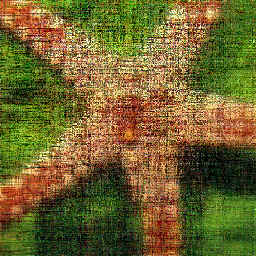} &
        \includegraphics[width=0.17\textwidth]{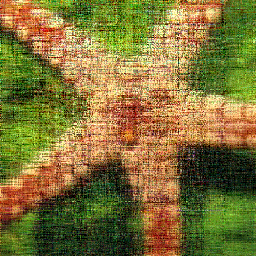} &
        \includegraphics[width=0.17\textwidth]{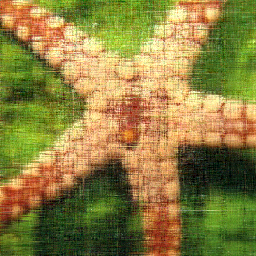} & 
        \includegraphics[width=0.17\textwidth]{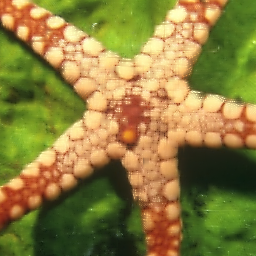} \\
         (e) QWSNM & (f) QWNNM & (g) Ours-G & (f) Ours-NNS \\

	\end{tabular}
	}
	\caption{   
    Comparison of visual results for color image matrix completion. The missing rate is 80\%. }
	\label{MC_fig}
\end{figure}
\begin{figure}[!t]
	\centering
	\addtolength{\tabcolsep}{-5.5pt}
    \renewcommand\arraystretch{0.6}
	{\fontsize{10pt}{\baselineskip}\selectfont 
	\begin{tabular}{cccc}
        \includegraphics[width=0.17\textwidth]{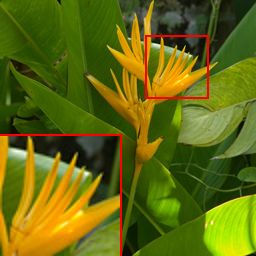} &
        \includegraphics[width=0.17\textwidth]{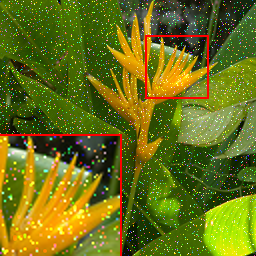} &
        \includegraphics[width=0.17\textwidth]{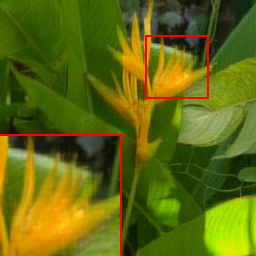} &
        \includegraphics[width=0.17\textwidth]{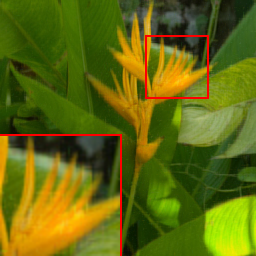} \\

        (a) GT & (b) NY (10\%) & (c) TRPCA & (d) QMC \\
        
        \includegraphics[width=0.17\textwidth]{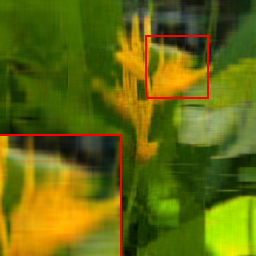} &
        \includegraphics[width=0.17\textwidth]{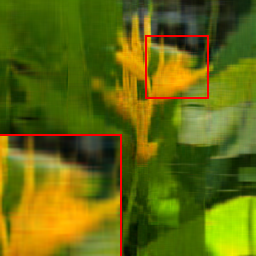} &
        \includegraphics[width=0.17\textwidth]{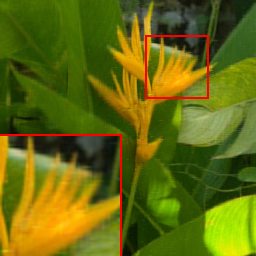} & 
        \includegraphics[width=0.17\textwidth]{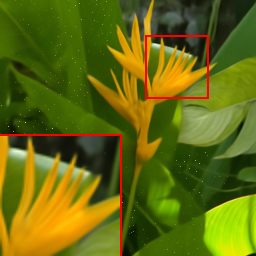} \\
         (e) QWSNM & (f) QWNNM & (g) Ours-G & (f) Ours-NNS \\

	\end{tabular}
	}
	\caption{ 
    Comparison of visual results for random impulse noise removal from color images. The random rate is 10\%. }
	\label{RPCA_fig}
\end{figure}

Table \ref{mc} presents the numerical results of matrix completion, where QNMF-NNS achieves the best performance. QNMF-G ranks second in PSNR, while TRPCA and QMC perform well in SSIM. However, QWNNM and QWSNM, which excel in denoising and deblurring, exhibit subpar performance in matrix completion. Figure \ref{MC_fig} illustrates the visual comparison of completion results for 80\% missing data. Notably, the reconstructions by QWNNM and QWSNM exhibit significant noise, while those of TRPCA and QMC are overly smooth. QNMF-G stands out visually among global algorithms by preserving more details. QNMF-NNS leverages non-local self-similarity prior for superior results. Table \ref{rpca} and Figure \ref{RPCA_fig} present quantitative and qualitative RPCA results, respectively. The performance comparisons are similar to those in matrix completion, with QNMF-NNS not consistently optimal. QNMF-G outperforms QNMF-NNS particularly when random noise is minimal, as fewer pixel mutations interfere with the image, potentially leading to smoother outcomes with NNS utilization.

\subsection{Ablation experiment}

In the QNMF model (\ref{NNFN}), the parameter $\alpha$ is crucial. When $\alpha=0$, the QNMF model simplifies to the QNNM model. Setting $\alpha$ is critical because a value that is too large can truncate the shrinkage of singular values, making it ineffective. Therefore, $\alpha$ should be within a reasonable range to ensure effective singular value shrinkage.
We evaluated the impact of $\alpha$ on the performance of the Gaussian denoising algorithm for various noise levels ($\sigma=5, 20, 40, 60$). We tested $\alpha$ values of 0.5, 1, 2, 3, 4, 5, and 6, and plotted the performance curves in Figure \ref{alpha_fig}. The trends for both PSNR and SSIM are similar with respect to $\alpha$. For $\alpha$ values between 0.5 and 3, the algorithm’s performance improves as $\alpha$ increases. However, for $\alpha$ values greater than 4, the performance begins to decline due to insufficient singular value shrinkage. Based on these observations, we typically set $\alpha=4$ in practical applications of the QNMF algorithm to balance performance and effective singular value shrinkage.

\begin{figure}[!t]
	\centering
	\addtolength{\tabcolsep}{-5.5pt}
    \renewcommand\arraystretch{0.6}
	{\fontsize{10pt}{\baselineskip}\selectfont 
	\begin{tabular}{cccc}
        \includegraphics[width=0.4\textwidth]{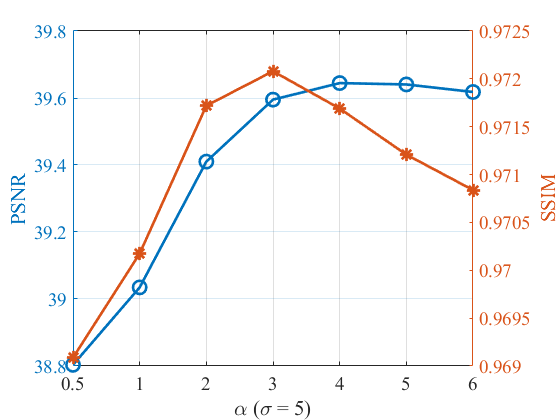} &
        \includegraphics[width=0.4\textwidth]{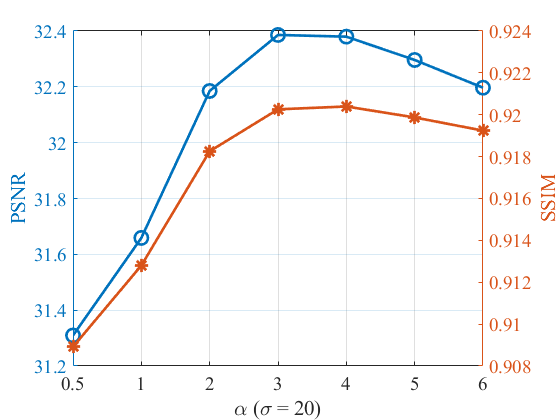} \\
        \includegraphics[width=0.4\textwidth]{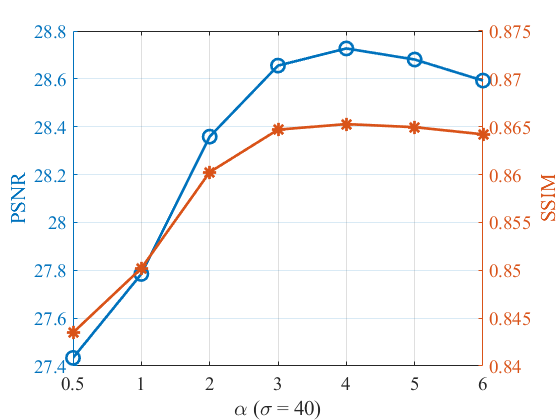} &
        \includegraphics[width=0.4\textwidth]{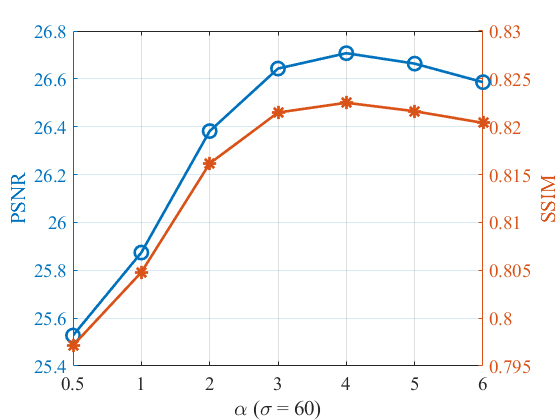} \\

	\end{tabular}
	}
 '
	\caption{ The effect of the value $\alpha$ on the denoising performance.} 
	\label{alpha_fig}
\end{figure}

%%% GT
\begin{table*}[t]	
	\begin{center}
	\caption{The running time (s) of all algorithms on a PC with Intel (R) Core (TM) i7-13700KF 3.40 GHz and 64G RAM.} 
 
    \label{time}
    \resizebox{\textwidth}{!}{
    \addtolength{\tabcolsep}{ -2pt}
    \begin{tabular}{c|ccccccc}
    \hline
    \multicolumn{8}{c}{Denoising, $\sigma=40$} \\ \hline
    
    Size & CBM3D \cite{dabov2007color} & McWNNM \cite{xu2017multi} & SV-TV \cite{jia2019color} & QLRMA \cite{chen2019low} & QWSNM \cite{zhang2024quaternion} & QWNNM \cite{yu2019quaternion} & Ours \\ \hline
    256$\times$256 & 2.02 & 111.57 & 3.01  & 577.46  & 955.33  & 580.60  &  620.86 \\  
    500$\times$500 & 7.38 & 460.51 & 11.53 & 2289.23 & 3379.72 & 2289.69 & 2423.98 \\ \hline

    \multicolumn{8}{c}{Debluring, GB(25,1.6)/$\sigma=15$} \\ \hline
    Size & -- & SNSS \cite{zha2020image} & SV-TV \cite{jia2019color} & QD-SV-TV \cite{huang2021quaternion} & QWSNM \cite{zhang2024quaternion} & QWNNM \cite{huang2022quaternion} & Ours \\ \hline
    256$\times$256 & -- & 324.42 & 3.09  & 22.82  & 3959.28  & 3831.23  &  2154.88 \\  \hline

    \multicolumn{8}{c}{Matrix completion, miss=80\% } \\ \hline
    Size & -- & -- & TRPCA \cite{yang2020low} & QMC \cite{jia2019robust} & QWSNM \cite{zhang2024quaternion} & QWNNM \cite{huang2022quaternion} & Ours \\ \hline
    256$\times$256 & -- & -- & 15.02 & 168.13  & 21.07  & 20.85  &  10.83 \\  \hline

    \end{tabular}}
	\end{center}

\end{table*}

To demonstrate the computational complexity of different algorithms, we tested their runtimes on a PC with an Intel (R) Core (TM) i7-13700KF 3.40 GHz processor and 64GB RAM. The runtimes are shown in Table \ref{time}. CBM3D and SV-TV exhibit rapid runtimes, while low-rank class algorithms generally require longer runtimes. Notably, the proposed QNMF algorithm has an advantageous runtime among quaternion-based algorithms.

\section{Limitation and Conclusion}
\label{sec:5}

Our comparison of runtimes reveals that algorithms utilizing the low-rank nature of images generally have longer runtimes. This is due to the fact that image rank often follows a long-tailed distribution, making direct use of low-rank priors less effective. To better exploit low-rank priors, similar image patches are typically combined with the self-similarity of images to form low-rank matrices. While this method improves results, it also increases computational complexity.
Furthermore, algorithms based on quaternion algebra typically rely on QSVD, which is slower than SVD in the real number domain. These challenges highlight the need for future optimization efforts to enhance computational speed. Currently, deep learning algorithms are at the forefront of image reconstruction. Incorporating a priori knowledge from traditional algorithms into these learnable models will be a key area of focus moving forward.

In this paper, we investigate the quaternion nucleus norm minus Frobenius norm minimization (QNMF) problem and apply it to various image processing tasks. Initially, we introduce a strategy for solving the QNMF problem while minimizing the loss of F-norm fidelity, offering a closed-form solution. Subsequently, we extend the application of QNMF to address the color image linear inverse problem, matrix completion (MC), and robust principal component analysis (RPCA) using the ADMM framework. We provide rigorous proofs for all these methodologies to ensure their theoretical soundness. We validate the effectiveness of QNMF across various tasks, such as color image denoising, deblurring, inpainting, and random impulse noise removal. Experimental results confirm that the proposed method outperforms existing techniques, especially in preserving color fidelity and reducing artifacts. The ability to handle various rank components and the availability of closed-form solutions make QNMF a promising tool for future research in color image reconstruction and related applications.  In the future, we will consider whether other quaternion hybrid norms can approximate the rank norm more accurately, such as the quaternion nuclear norm over Frobenius norm. Meanwhile, considering that QSVD has high computational complexity, QSVD-free quaternion methods will also be the focus of our future work.

\section*{Acknowledgment}
This work was supported by National Natural Science Foundation of China (No. 12061052), Young Talents of Science and Technology in Universities of Inner Mongolia Autonomous Region (No. NJYT22090), Natural Science Foundation of Inner Mongolia Autonomous Region (No. 2024LHMS01006), Innovative Research Team in Universities of Inner Mongolia Autonomous Region (No. NMGIRT2207),  Special Funds for Graduate Innovation and Entrepreneurship of Inner Mongolia University (No.~11200-121024), Prof. Guoqing Chen's “111 project” of higher education talent training in Inner Mongolia Autonomous Region, Inner Mongolia University Independent Research Project (No. 2022-ZZ004), Inner Mongolia Science and Technology Achievement Transfer and Transformation Demonstration Zone, University Collaborative Innovation Base, and University Entrepreneurship Training Base" Construction Project (Supercomputing Power Project) (No. 21300-231510) and the network information center of Inner Mongolia University. M. Ng’s research is funded by HKRGC GRF 17201020 and 17300021, HKRGC CRF C7004-21GF, and Joint NSFC and RGC N-HKU769/21.
The authors are also grateful to the reviewers for their valuable comments and remarks.

\section*{Appendix}
\subsection*{A. Proof of Theorem \ref{theorem02}}
\label{Appendix 2}

\begin{proof}[Proof]
First, we prove that the sequence $\{\dot{\mathbf{\eta}}^{(k)}\}$ generated by Algorithm \ref{ag1} is upper bounded.
\begin{equation}
\begin{aligned}
||\dot{\mathbf{\eta}}^{(k+1)} ||_{F}^{2} &= || \dot{\mathbf{\eta}}^{(k)} + \beta^{(k)}(\dot{\mathbf{X}}^{(k+1)}-\dot{\mathbf{Z}}^{(k+1)})  ||_{F}^{2} \\
&= {\beta^{(k)}}^{2} || ({\beta^{(k)}}^{-1}\dot{\mathbf{\eta}}^{(k)} + \dot{\mathbf{X}}^{(k+1)}) -\dot{\mathbf{Z}}^{(k+1)}  ||_{F}^{2} \\
&= {\beta^{(k)}}^{2} ||\dot{\mathbf{U}}^{(k)}\Sigma^{(k)} \mathbf{V}^{(k)\top} -\dot{\mathbf{U}}^{(k)}\Sigma^{(k)}_{\dot{\mathbf{Z}}} \mathbf{V}^{(k)\top} ||_{F}^{2} \\
&= {\beta^{(k)}}^{2} \sum^{M}_{i=1}(\sigma^{(k)}_{i} - \sigma^{(k)}_{\dot{\mathbf{Z}},i})^2 \\
&\leq {\beta^{(k)}}^{2} \sum_{i=1}^{M} (\lambda/\beta^{(k)})^{2} \\
&= \lambda^{2}M.
\label{eta}
\end{aligned}
\end{equation}
For the last inequality, if $\sigma^{(k)}_{i} \geq \lambda/\beta^{(k)}$, then $\sigma^{(k)}_{\dot{\mathbf{Z}},i} = \frac{||z||_{2}+\alpha \lambda/\beta^{(k)}}{||z||_{2}}\max(\sigma^{(k)}_{i}-\lambda/\beta^{(k)}, 0)$.  Obviously $\sigma^{(k)}_{i} - \sigma^{(k)}_{\dot{\mathbf{Z}},i} \leq \lambda/\beta^{(k)}$ holds. If $\sigma^{(k)}_{i} < \lambda/\beta^{(k)}$, $\sigma^{(k)}_{i} - \sigma^{(k)}_{\dot{\mathbf{Z}},i} = \sigma^{(k)}_{i} < \lambda/\beta^{(k)}$. Hence, the sequence $\{\dot{\mathbf{\eta}}^{(k)}\}$ generated by Algorithm \ref{ag1} is upper bounded.

Then, we prove that the sequence of Lagrange function $\{ \mathcal{L}(\dot{\mathbf{X}}^{(k+1)}, \dot{\mathbf{Z}}^{(k+1)}, \dot{\mathbf{\eta}}^{(k+1)}, \beta^{(k+1)}) \}$ is also upper bounded. According to the update rule for $\dot{\mathbf{\eta}}^{(k)}$, it can be obtained 
\begin{equation}
\begin{aligned}
&\mathcal{L}(\dot{\mathbf{X}}^{(k+1)}, \dot{\mathbf{Z}}^{(k+1)}, \dot{\mathbf{\eta}}^{(k+1)}, \beta^{(k+1)}) \\
&~=\mathcal{L}(\dot{\mathbf{X}}^{(k+1)}, \dot{\mathbf{Z}}^{(k+1)}, \dot{\mathbf{\eta}}^{(k)}, \beta^{(k)}) + \langle{\dot{\mathbf{\eta}}^{(k+1)} - \dot{\mathbf{\eta}}^{(k)}, \dot{\mathbf{X}}^{(k+1)}-\dot{\mathbf{Z}}^{(k+1)} }\rangle  +\frac{\beta^{(k+1)}-\beta^{(k)}}{2}||\dot{\mathbf{X}}^{(k+1)}-\dot{\mathbf{Z}}^{(k+1)} ||_{F}^{2} \\
&~=\mathcal{L}(\dot{\mathbf{X}}^{(k+1)}, \dot{\mathbf{Z}}^{(k+1)}, \dot{\mathbf{\eta}}^{(k)}, \beta^{(k)})  + \langle{\dot{\mathbf{\eta}}^{(k+1)} - \dot{\mathbf{\eta}}^{(k)}, \frac{\dot{\mathbf{\eta}}^{(k+1)} - \dot{\mathbf{\eta}}^{(k)}}{\beta^{(k)}} }\rangle  +\frac{\beta^{(k+1)}-\beta^{(k)}}{2}||\frac{\dot{\mathbf{\eta}}^{(k+1)} - \dot{\mathbf{\eta}}^{(k)}}{\beta^{(k)}} ||_{F}^{2} \\
&~= \mathcal{L}(\dot{\mathbf{X}}^{(k+1)}, \dot{\mathbf{Z}}^{(k+1)}, \dot{\mathbf{\eta}}^{(k)}, \beta^{(k)})  +\frac{\beta^{(k+1)} +\beta^{(k)}}{2{\beta^{(k)}}^2}||\dot{\mathbf{\eta}}^{(k+1)} - \dot{\mathbf{\eta}}^{(k)} ||_{F}^{2}. 
\label{L}
\end{aligned}
\end{equation}
Since $\{\dot{\mathbf{\eta}}^{(k)}\}$ is bounded, the sequence $\{\dot{\mathbf{\eta}}^{(k+1)} - \dot{\mathbf{\eta}}^{(k)}\}$ is also bounded. Suppose that the upper bound of the sequence $\{\dot{\mathbf{\eta}}^{(k+1)} - \dot{\mathbf{\eta}}^{(k)}\}$ is $M_0$, i.e., $ \forall k \geq 0, || \dot{\mathbf{\eta}}^{(k+1)} - \dot{\mathbf{\eta}}^{(k)} ||_F \leq M_0$. Meanwhile, the inequality $\mathcal{L}(\dot{\mathbf{X}}^{(k+1)}, \dot{\mathbf{Z}}^{(k+1)}, \dot{\mathbf{\eta}}^{(k)}, \beta^{(k)}) \leq \mathcal{L}(\dot{\mathbf{X}}^{(k)}, \dot{\mathbf{Z}}^{(k)}, \dot{\mathbf{\eta}}^{(k)}, \beta^{(k)})$ always holds because $\dot{\mathbf{X}}$ and $\dot{\mathbf{Z}}$ are globally optimal solutions to the corresponding subproblems. Therefore, we have
\begin{equation}
\begin{aligned}
&\mathcal{L}(\dot{\mathbf{X}}^{(k+1)}, \dot{\mathbf{Z}}^{(k+1)}, \dot{\mathbf{\eta}}^{(k+1)}, \beta^{(k+1)}) \\
&~\leq \mathcal{L}(\dot{\mathbf{X}}^{(k+1)}, \dot{\mathbf{Z}}^{(k+1)}, \dot{\mathbf{\eta}}^{(k)}, \beta^{(k)}) + \frac{\beta^{(k+1)}+\beta^{(k)}}{2{\beta^{(k)}}^2}M_{0}^{2} \\
&~\leq \mathcal{L}(\dot{\mathbf{X}}^{(1)}, \dot{\mathbf{Z}}^{(1)}, \dot{\mathbf{\eta}}^{(0)}, \beta^{(0)}) + M_{0}^{2}\sum_{k=0}^{\infty}\frac{1+\mu}{2\beta^{(0)}\mu^{k}} \\
&~\leq \mathcal{L}(\dot{\mathbf{X}}^{(1)}, \dot{\mathbf{Z}}^{(1)}, \dot{\mathbf{\eta}}^{(0)}, \beta^{(0)}) + \frac{M_{0}^{2}}{\beta^{(0)}}\sum_{k=0}^{\infty}\frac{1}{\mu^{k-1}} \\
&~< -\infty.
\end{aligned}
\label{L2}
\end{equation}
Hence, $\{\mathcal{L}(\dot{\mathbf{X}}^{(k+1)}, \dot{\mathbf{Z}}^{(k+1)}, \dot{\mathbf{\eta}}^{(k+1)}, \beta^{(k+1)}) \}$ is upper bounded.

Next, we prove that the sequences $\{\dot{\mathbf{X}}^{(k)}$\} and $\{\dot{\mathbf{Z}}^{(k)}\}$ are upper bounded. From formula (\ref{en-NNFN}), we have
\begin{equation}
\begin{aligned}
&\frac{\gamma}{2}|| \dot{\mathbf{A}}\dot{\mathbf{X}}^{(k+1)} - \dot{\mathbf{Y}} ||^{2}_F + \lambda(||\dot{\mathbf{Z}}^{(k+1)}||_{*} - \alpha||\dot{\mathbf{Z}}^{(k+1)}||_{F} ) \\
&=\mathcal{L}(\dot{\mathbf{X}}^{(k+1)}, \dot{\mathbf{Z}}^{(k+1)}, \dot{\mathbf{\eta}}^{(k)}, \beta^{(k)}) - \langle{\dot{\mathbf{\eta}}^{(k)}, \dot{\mathbf{X}}^{(k+1)} - \dot{\mathbf{Z}}^{(k+1)}}\rangle - \frac{\beta^{(k)}}{2}|| \dot{\mathbf{X}}^{(k+1)} - \dot{\mathbf{Z}}^{(k+1)} ||_{F}^{2} \\
&=\mathcal{L}(\dot{\mathbf{X}}^{(k+1)}, \dot{\mathbf{Z}}^{(k+1)}, \dot{\mathbf{\eta}}^{(k)}, \beta^{(k)}) - \langle{\dot{\mathbf{\eta}}^{(k)}, \frac{\dot{\mathbf{\eta}}^{(k+1)} - \dot{\mathbf{\eta}}^{(k)}}{\beta^{(k)}} }\rangle - \frac{\beta^{(k)}}{2}|| \frac{\dot{\mathbf{\eta}}^{(k+1)} - \dot{\mathbf{\eta}}^{(k)}}{\beta^{(k)}} ||_{F}^{2} \\
&= \mathcal{L}(\dot{\mathbf{X}}^{(k+1)}, \dot{\mathbf{Z}}^{(k+1)}, \dot{\mathbf{\eta}}^{(k)}, \beta^{(k)}) + \frac{||\dot{\mathbf{\eta}}^{(k)}||_{F}^{2} - ||\dot{\mathbf{\eta}}^{(k+1)}||_{F}^{2} }{2\beta^{(k)}}.
\label{XZ}
\end{aligned}
\end{equation}
From this, we get that $\{\dot{\mathbf{A}}\dot{\mathbf{X}}^{(k)} - \dot{\mathbf{Y}} \}$ and $\{ \dot{\mathbf{Z}}^{(k)} \}$ are upper bounded. Since $\dot{\mathbf{\eta}}^{(k+1)} = \dot{\mathbf{\eta}}^{(k)} + \beta^{(k)}(\dot{\mathbf{X}}^{(k+1)} - \dot{\mathbf{Z}}^{(k+1)}) $, $\{ \dot{\mathbf{X}}^{(k)} \}$ is also upper bounded. Thus, there exists at least one accumulation point for $\{ \dot{\mathbf{X}}^{(k)}, \dot{\mathbf{Z}}^{(k)} \}$. Further, we can get
\begin{equation}
\begin{aligned}
\lim_{k\rightarrow +\infty} || \dot{\mathbf{X}}^{(k+1)}-\dot{\mathbf{Z}}^{(k+1)} ||_{F} &= \lim_{k\rightarrow +\infty} \frac{1}{\beta^{(k)}}|| \dot{\mathbf{\eta}}^{(k+1)}-\dot{\mathbf{\eta}}^{(k)} ||_{F} \\
&=0,
\end{aligned}
\end{equation}
and that the accumulation point is a feasible solution to the objective function. Thus, the first equation in (\ref{converge1}) is proved.

Finally, we demonstrate that the change in the sequences $\{\dot{\mathbf{X}}^{(k)}\}$ and $\{\dot{\mathbf{Z}}^{(k)}\}$ converges to 0 with iteration. By $\dot{\mathbf{X}}^{(k)} = \frac{1}{\beta^{(k-1)}}(\dot{\mathbf{\eta}}^{(k)} - \dot{\mathbf{\eta}}^{(k-1)} ) + \dot{\mathbf{Z}}^{(k)}$, we get
\begin{equation}
\begin{aligned}
&\lim_{k\rightarrow +\infty} || \dot{\mathbf{X}}^{(k+1)}- \dot{\mathbf{X}}^{(k)} ||_{F} \\
&~= \lim_{k\rightarrow +\infty} || (\gamma\dot{\mathbf{A}}^{*}\dot{\mathbf{A}} + \beta^{(k)}\dot{\mathbf{I}})^{-1}(\gamma\dot{\mathbf{A}}^{*}\dot{\mathbf{Y}} + \beta^{(k)}\dot{\mathbf{Z}}^{(k)} - \dot{\mathbf{\eta}}^{(k)}) - \frac{1}{\beta^{(k-1)}}(\dot{\mathbf{\eta}}^{(k)} - \dot{\mathbf{\eta}}^{(k-1)} ) - \dot{\mathbf{Z}}^{(k)} ||_{F} \\
&~= \lim_{k\rightarrow +\infty} || (\gamma\dot{\mathbf{A}}^{*}\dot{\mathbf{A}} + \beta^{(k)}\dot{\mathbf{I}})^{-1}(\gamma\dot{\mathbf{A}}^{*}\dot{\mathbf{Y}} - \gamma\dot{\mathbf{A}}^{*}\dot{\mathbf{A}}\dot{\mathbf{Z}}^{(k)} - \dot{\mathbf{\eta}}^{(k)}) - \frac{1}{\beta^{(k-1)}}(\dot{\mathbf{\eta}}^{(k)} - \dot{\mathbf{\eta}}^{(k-1)} ) ||_{F} \\
&~\leq \lim_{k\rightarrow +\infty} || (\gamma\dot{\mathbf{A}}^{*}\dot{\mathbf{A}} + \beta^{(k)}\dot{\mathbf{I}})^{-1}(\gamma\dot{\mathbf{A}}^{*}\dot{\mathbf{Y}} - \gamma\dot{\mathbf{A}}^{*}\dot{\mathbf{A}}\dot{\mathbf{Z}}^{(k)} - \dot{\mathbf{\eta}}^{(k)})||_{F} + \frac{1}{\beta^{(k-1)}}||\dot{\mathbf{\eta}}^{(k)} - \dot{\mathbf{\eta}}^{(k-1)}||_{F} \\
&~=0.
\end{aligned}
\end{equation}

Similarly, we get
\begin{equation}
\begin{aligned}
&\lim_{k\rightarrow +\infty} || \dot{\mathbf{Z}}^{(k+1)}- \dot{\mathbf{Z}}^{(k)} ||_{F} \\
&~= \lim_{k\rightarrow +\infty} || \frac{1}{\beta^{(k)}}(\dot{\mathbf{\eta}}^{(k)} - \dot{\mathbf{\eta}}^{(k+1)} ) +  \dot{\mathbf{X}}^{(k+1)} - \dot{\mathbf{Z}}^{(k)} ||_{F} \\
&~= \lim_{k\rightarrow +\infty} || \dot{\mathbf{X}}^{(k)} + \frac{1}{\beta^{(k-1)}}\dot{\mathbf{\eta}}^{(k-1)} - \dot{\mathbf{Z}}^{(k)} +  \dot{\mathbf{X}}^{(k+1)} -  \dot{\mathbf{X}}^{(k)} + \frac{1}{\beta^{(k)}}(\dot{\mathbf{\eta}}^{(k)} - \dot{\mathbf{\eta}}^{(k+1)} )  - \frac{1}{\beta^{(k-1)}}\dot{\mathbf{\eta}}^{(k-1)} ||_{F} \\
&~\leq \lim_{k\rightarrow +\infty} || \dot{\mathbf{X}}^{(k)} + \frac{1}{\beta^{(k-1)}}\dot{\mathbf{\eta}}^{(k-1)} - \dot{\mathbf{Z}}^{(k)} ||_{F} + || \dot{\mathbf{X}}^{(k+1)} -  \dot{\mathbf{X}}^{(k)} ||_{F} + || \frac{1}{\beta^{(k)}}(\dot{\mathbf{\eta}}^{(k)} - \dot{\mathbf{\eta}}^{(k+1)} )  - \frac{1}{\beta^{(k-1)}}\dot{\mathbf{\eta}}^{(k-1)} ||_{F} \\
&~= \lim_{k\rightarrow +\infty} || \dot{\mathbf{X}}^{(k)} + \frac{1}{\beta^{(k-1)}}\dot{\mathbf{\eta}}^{(k-1)} - \dot{\mathbf{Z}}^{(k)} ||_{F} + || \dot{\mathbf{X}}^{(k+1)} -  \dot{\mathbf{X}}^{(k)} ||_{F} + || \frac{1}{\beta^{(k)}}(\dot{\mathbf{\eta}}^{(k+1)} - \dot{\mathbf{\eta}}^{(k)} ) ||_{F} + || \frac{1}{\beta^{(k-1)}}\dot{\mathbf{\eta}}^{(k-1)} ||_{F} \\
&~= \lim_{k\rightarrow +\infty} ||  \frac{1}{\beta^{(k-1)}}\dot{\mathbf{\eta}}^{(k)}  ||_{F} + || \dot{\mathbf{X}}^{(k+1)} -  \dot{\mathbf{X}}^{(k)} ||_{F} + ||  \dot{\mathbf{X}}^{(k+1)} - \dot{\mathbf{Z}}^{(k+1)} ||_{F} + || \frac{1}{\beta^{(k-1)}}\dot{\mathbf{\eta}}^{(k-1)} ||_{F} \\
&~=0.
\end{aligned}
\end{equation}

\end{proof}

\bibliographystyle{elsarticle-num-names}
\bibliography{ref}

\begin{thebibliography}{42}
\expandafter\ifx\csname natexlab\endcsname\relax\def\natexlab#1{#1}\fi
\providecommand{\url}[1]{\texttt{#1}}
\providecommand{\href}[2]{#2}
\providecommand{\path}[1]{#1}
\providecommand{\DOIprefix}{doi:}
\providecommand{\ArXivprefix}{arXiv:}
\providecommand{\URLprefix}{URL: }
\providecommand{\Pubmedprefix}{pmid:}
\providecommand{\doi}[1]{\href{http://dx.doi.org/#1}{\path{#1}}}
\providecommand{\Pubmed}[1]{\href{pmid:#1}{\path{#1}}}
\providecommand{\bibinfo}[2]{#2}
\ifx\xfnm\relax \def\xfnm[#1]{\unskip,\space#1}\fi
%Type = Article
\bibitem[{Jia et~al.(2019)Jia, Ng, and Wang}]{jia2019color}
\bibinfo{author}{Z.~Jia}, \bibinfo{author}{M.~K. Ng},
  \bibinfo{author}{W.~Wang},
\newblock \bibinfo{title}{Color image restoration by saturation-value total
  variation},
\newblock \bibinfo{journal}{SIAM J. Imaging Sci.} \bibinfo{volume}{12}
  (\bibinfo{year}{2019}) \bibinfo{pages}{972--1000}.
  \DOIprefix\doi{10.1137/18M1230451}.
%Type = Article
\bibitem[{Dabov et~al.(2007{\natexlab{a}})Dabov, Foi, Katkovnik, and
  Egiazarian}]{DF07BM3D}
\bibinfo{author}{K.~Dabov}, \bibinfo{author}{A.~Foi},
  \bibinfo{author}{V.~Katkovnik}, \bibinfo{author}{K.~Egiazarian},
\newblock \bibinfo{title}{Image denoising by sparse 3-d transform-domain
  collaborative filtering},
\newblock \bibinfo{journal}{IEEE Trans. Image Process.} \bibinfo{volume}{16}
  (\bibinfo{year}{2007}{\natexlab{a}}) \bibinfo{pages}{2080--2095}.
  \DOIprefix\doi{10.1109/TIP.2007.901238}.
%Type = Inproceedings
\bibitem[{Dabov et~al.(2007{\natexlab{b}})Dabov, Foi, Katkovnik, and
  Egiazarian}]{dabov2007color}
\bibinfo{author}{K.~Dabov}, \bibinfo{author}{A.~Foi},
  \bibinfo{author}{V.~Katkovnik}, \bibinfo{author}{K.~Egiazarian},
\newblock \bibinfo{title}{Color image denoising via sparse 3d collaborative
  filtering with grouping constraint in luminance-chrominance space},
\newblock in: \bibinfo{booktitle}{Proc. IEEE Int. Conf. Image Process.},
  volume~\bibinfo{volume}{1}, \bibinfo{organization}{IEEE},
  \bibinfo{year}{2007}{\natexlab{b}}, pp. \bibinfo{pages}{I--313}.
  \DOIprefix\doi{10.1109/ICIP.2007.4378954}.
%Type = Article
\bibitem[{Huang et~al.(2021)Huang, Ng, Wu, and Zeng}]{huang2021quaternion}
\bibinfo{author}{C.~Huang}, \bibinfo{author}{M.~K. Ng},
  \bibinfo{author}{T.~Wu}, \bibinfo{author}{T.~Zeng},
\newblock \bibinfo{title}{Quaternion-based dictionary learning and
  saturation-value total variation regularization for color image restoration},
\newblock \bibinfo{journal}{IEEE Trans. Multimedia} \bibinfo{volume}{24}
  (\bibinfo{year}{2021}) \bibinfo{pages}{3769--3781}.
  \DOIprefix\doi{10.1109/TMM.2021.3107162}.
%Type = Article
\bibitem[{Zha et~al.(2020)Zha, Yuan, Zhou, Zhu, and Wen}]{zha2020image}
\bibinfo{author}{Z.~Zha}, \bibinfo{author}{X.~Yuan}, \bibinfo{author}{J.~Zhou},
  \bibinfo{author}{C.~Zhu}, \bibinfo{author}{B.~Wen},
\newblock \bibinfo{title}{Image restoration via simultaneous nonlocal
  self-similarity priors},
\newblock \bibinfo{journal}{IEEE Trans. Image Process.} \bibinfo{volume}{29}
  (\bibinfo{year}{2020}) \bibinfo{pages}{8561--8576}.
  \DOIprefix\doi{10.1109/TIP.2020.3015545}.
%Type = Inproceedings
\bibitem[{Gu et~al.(2014)Gu, Zhang, Zuo, and Feng}]{gu2014weighted}
\bibinfo{author}{S.~Gu}, \bibinfo{author}{L.~Zhang}, \bibinfo{author}{W.~Zuo},
  \bibinfo{author}{X.~Feng},
\newblock \bibinfo{title}{Weighted nuclear norm minimization with application
  to image denoising},
\newblock in: \bibinfo{booktitle}{Proc. IEEE Conf. Comput. Vis. Pattern
  Recognit.}, \bibinfo{year}{2014}, pp. \bibinfo{pages}{2862--2869}.
  \DOIprefix\doi{10.1109/CVPR.2014.366}.
%Type = Article
\bibitem[{Guo et~al.(2022)Guo, Guo, Jin, Kwok-Po~Ng, and
  Wang}]{guo2022gaussian}
\bibinfo{author}{J.~Guo}, \bibinfo{author}{Y.~Guo}, \bibinfo{author}{Q.~Jin},
  \bibinfo{author}{M.~Kwok-Po~Ng}, \bibinfo{author}{S.~Wang},
\newblock \bibinfo{title}{Gaussian patch mixture model guided low-rank
  covariance matrix minimization for image denoising},
\newblock \bibinfo{journal}{SIAM J. Imaging Sci.} \bibinfo{volume}{15}
  (\bibinfo{year}{2022}) \bibinfo{pages}{1601--1622}.
  \DOIprefix\doi{10.1137/21M1454262}.
%Type = Article
\bibitem[{Jiang et~al.(2023)Jiang, Zhang, Zhang, Wang, and
  Qi}]{jiang2023robust}
\bibinfo{author}{W.~Jiang}, \bibinfo{author}{J.~Zhang},
  \bibinfo{author}{C.~Zhang}, \bibinfo{author}{L.~Wang},
  \bibinfo{author}{H.~Qi},
\newblock \bibinfo{title}{Robust low tubal rank tensor completion via factor
  tensor norm minimization},
\newblock \bibinfo{journal}{Pattern Recognit.} \bibinfo{volume}{135}
  (\bibinfo{year}{2023}) \bibinfo{pages}{109169}.
  \DOIprefix\doi{10.1016/j.patcog.2022.109169}.
%Type = Article
\bibitem[{Yang et~al.(2020)Yang, Zhao, Ji, Ma, and Huang}]{yang2020low}
\bibinfo{author}{J.-H. Yang}, \bibinfo{author}{X.-L. Zhao},
  \bibinfo{author}{T.-Y. Ji}, \bibinfo{author}{T.-H. Ma},
  \bibinfo{author}{T.-Z. Huang},
\newblock \bibinfo{title}{Low-rank tensor train for tensor robust principal
  component analysis},
\newblock \bibinfo{journal}{Appl. Math. Comput.} \bibinfo{volume}{367}
  (\bibinfo{year}{2020}) \bibinfo{pages}{124783}.
  \DOIprefix\doi{10.1016/j.amc.2019.124783}.
%Type = Article
\bibitem[{{Zhang} et~al.(2017){Zhang}, {Zuo}, {Chen}, {Meng}, and
  {Zhang}}]{Zhang2017DnCNN}
\bibinfo{author}{K.~{Zhang}}, \bibinfo{author}{W.~{Zuo}},
  \bibinfo{author}{Y.~{Chen}}, \bibinfo{author}{D.~{Meng}},
  \bibinfo{author}{L.~{Zhang}},
\newblock \bibinfo{title}{Beyond a gaussian denoiser: Residual learning of deep
  cnn for image denoising},
\newblock \bibinfo{journal}{IEEE Trans. Image Process.} \bibinfo{volume}{26}
  (\bibinfo{year}{2017}) \bibinfo{pages}{3142--3155}.
  \DOIprefix\doi{10.1109/TIP.2017.2662206}.
%Type = Article
\bibitem[{Guo et~al.(2021)Guo, Davy, Facciolo, Morel, and Jin}]{guo2021fast}
\bibinfo{author}{Y.~Guo}, \bibinfo{author}{A.~Davy},
  \bibinfo{author}{G.~Facciolo}, \bibinfo{author}{J.-M. Morel},
  \bibinfo{author}{Q.~Jin},
\newblock \bibinfo{title}{Fast, nonlocal and neural: A lightweight high quality
  solution to image denoising},
\newblock \bibinfo{journal}{IEEE Signal Processing Letters}
  \bibinfo{volume}{28} (\bibinfo{year}{2021}) \bibinfo{pages}{1515--1519}.
  \DOIprefix\doi{10.1109/LSP.2021.3099963}.
%Type = Article
\bibitem[{Hurault et~al.(2018)Hurault, Ehret, and Arias}]{hurault2018epll}
\bibinfo{author}{S.~Hurault}, \bibinfo{author}{T.~Ehret},
  \bibinfo{author}{P.~Arias},
\newblock \bibinfo{title}{Epll: an image denoising method using a gaussian
  mixture model learned on a large set of patches},
\newblock \bibinfo{journal}{Image Processing On Line} \bibinfo{volume}{8}
  (\bibinfo{year}{2018}) \bibinfo{pages}{465--489}.
  \DOIprefix\doi{10.5201/ipol.2018.242}.
%Type = Article
\bibitem[{Guo et~al.(2024)Guo, Jin, Morel, and Facciolo}]{guo2023best}
\bibinfo{author}{Y.~Guo}, \bibinfo{author}{Q.~Jin}, \bibinfo{author}{J.-M.
  Morel}, \bibinfo{author}{G.~Facciolo},
\newblock \bibinfo{title}{How to best combine demosaicing and denoising?},
\newblock \bibinfo{journal}{Inverse Probl. Imaging} \bibinfo{volume}{18}
  (\bibinfo{year}{2024}) \bibinfo{pages}{571--599}.
  \DOIprefix\doi{10.3934/ipi.2023044}.
%Type = Article
\bibitem[{Kong and Yang(2019)}]{kong2019color}
\bibinfo{author}{Z.~Kong}, \bibinfo{author}{X.~Yang},
\newblock \bibinfo{title}{Color image and multispectral image denoising using
  block diagonal representation},
\newblock \bibinfo{journal}{IEEE Trans. Image Process.} \bibinfo{volume}{28}
  (\bibinfo{year}{2019}) \bibinfo{pages}{4247--4259}.
  \DOIprefix\doi{10.1109/TIP.2019.2907478}.
%Type = Article
\bibitem[{Miao et~al.(2020)Miao, Kou, and Liu}]{miao2020low}
\bibinfo{author}{J.~Miao}, \bibinfo{author}{K.~I. Kou},
  \bibinfo{author}{W.~Liu},
\newblock \bibinfo{title}{Low-rank quaternion tensor completion for recovering
  color videos and images},
\newblock \bibinfo{journal}{Pattern Recognit.} \bibinfo{volume}{107}
  (\bibinfo{year}{2020}) \bibinfo{pages}{107505}.
  \DOIprefix\doi{10.1016/j.patcog.2020.107505}.
%Type = Inproceedings
\bibitem[{Xu et~al.(2017)Xu, Zhang, Zhang, and Feng}]{xu2017multi}
\bibinfo{author}{J.~Xu}, \bibinfo{author}{L.~Zhang},
  \bibinfo{author}{D.~Zhang}, \bibinfo{author}{X.~Feng},
\newblock \bibinfo{title}{Multi-channel weighted nuclear norm minimization for
  real color image denoising},
\newblock in: \bibinfo{booktitle}{Proc. IEEE Int. Conf. Comput. Vis.},
  \bibinfo{year}{2017}, pp. \bibinfo{pages}{1096--1104}.
  \DOIprefix\doi{10.1109/ICCV.2017.125}.
%Type = Article
\bibitem[{Shan et~al.(2023)Shan, Hu, Wang, and Jia}]{shan2023multi}
\bibinfo{author}{Y.~Shan}, \bibinfo{author}{D.~Hu}, \bibinfo{author}{Z.~Wang},
  \bibinfo{author}{T.~Jia},
\newblock \bibinfo{title}{Multi-channel nuclear norm minus frobenius norm
  minimization for color image denoising},
\newblock \bibinfo{journal}{Signal Processing} \bibinfo{volume}{207}
  (\bibinfo{year}{2023}) \bibinfo{pages}{108959}.
  \DOIprefix\doi{10.1016/j.sigpro.2023.108959}.
%Type = Article
\bibitem[{Tian et~al.(2023)Tian, Zheng, Zuo, Zhang, Zhang, and
  Zhang}]{tian2023multi}
\bibinfo{author}{C.~Tian}, \bibinfo{author}{M.~Zheng},
  \bibinfo{author}{W.~Zuo}, \bibinfo{author}{B.~Zhang},
  \bibinfo{author}{Y.~Zhang}, \bibinfo{author}{D.~Zhang},
\newblock \bibinfo{title}{Multi-stage image denoising with the wavelet
  transform},
\newblock \bibinfo{journal}{Pattern Recognit.} \bibinfo{volume}{134}
  (\bibinfo{year}{2023}) \bibinfo{pages}{109050}.
  \DOIprefix\doi{10.1016/j.patcog.2022.109050}.
%Type = Article
\bibitem[{Shamsolmoali et~al.(2024)Shamsolmoali, Zareapoor, Zhou, Li, and
  Lu}]{shamsolmoali2024distance}
\bibinfo{author}{P.~Shamsolmoali}, \bibinfo{author}{M.~Zareapoor},
  \bibinfo{author}{H.~Zhou}, \bibinfo{author}{X.~Li}, \bibinfo{author}{Y.~Lu},
\newblock \bibinfo{title}{Distance-based weighted transformer network for image
  completion},
\newblock \bibinfo{journal}{Pattern Recognit.} \bibinfo{volume}{147}
  (\bibinfo{year}{2024}) \bibinfo{pages}{110120}.
  \DOIprefix\doi{10.1016/j.patcog.2023.110120}.
%Type = Inproceedings
\bibitem[{Xia et~al.(2023)Xia, Zhang, Wang, Wang, Wu, Tian, Yang, and
  Van~Gool}]{xia2023diffir}
\bibinfo{author}{B.~Xia}, \bibinfo{author}{Y.~Zhang},
  \bibinfo{author}{S.~Wang}, \bibinfo{author}{Y.~Wang},
  \bibinfo{author}{X.~Wu}, \bibinfo{author}{Y.~Tian},
  \bibinfo{author}{W.~Yang}, \bibinfo{author}{L.~Van~Gool},
\newblock \bibinfo{title}{Diffir: Efficient diffusion model for image
  restoration},
\newblock in: \bibinfo{booktitle}{Proc. IEEE/CVF Int. Conf. Comput. Vis.},
  \bibinfo{year}{2023}, pp. \bibinfo{pages}{13095--13105}.
  \DOIprefix\doi{10.1109/ICCV51070.2023.01204}.
%Type = Article
\bibitem[{Chen et~al.(2019)Chen, Xiao, and Zhou}]{chen2019low}
\bibinfo{author}{Y.~Chen}, \bibinfo{author}{X.~Xiao},
  \bibinfo{author}{Y.~Zhou},
\newblock \bibinfo{title}{Low-rank quaternion approximation for color image
  processing},
\newblock \bibinfo{journal}{IEEE Trans. Image Process.} \bibinfo{volume}{29}
  (\bibinfo{year}{2019}) \bibinfo{pages}{1426--1439}.
  \DOIprefix\doi{10.1109/TIP.2019.2941319}.
%Type = Article
\bibitem[{Yu et~al.(2019)Yu, Zhang, and Yuan}]{yu2019quaternion}
\bibinfo{author}{Y.~Yu}, \bibinfo{author}{Y.~Zhang}, \bibinfo{author}{S.~Yuan},
\newblock \bibinfo{title}{Quaternion-based weighted nuclear norm minimization
  for color image denoising},
\newblock \bibinfo{journal}{Neurocomputing} \bibinfo{volume}{332}
  (\bibinfo{year}{2019}) \bibinfo{pages}{283--297}.
  \DOIprefix\doi{10.1016/j.neucom.2018.12.034}.
%Type = Article
\bibitem[{Zhang(1997)}]{zhang1997quaternions}
\bibinfo{author}{F.~Zhang},
\newblock \bibinfo{title}{Quaternions and matrices of quaternions},
\newblock \bibinfo{journal}{Linear Alg. Appl.} \bibinfo{volume}{251}
  (\bibinfo{year}{1997}) \bibinfo{pages}{21--57}.
  \DOIprefix\doi{10.1016/0024-3795(95)00543-9}.
%Type = Article
\bibitem[{Gu et~al.(2017)Gu, Xie, Meng, Zuo, Feng, and Zhang}]{gu2017weighted}
\bibinfo{author}{S.~Gu}, \bibinfo{author}{Q.~Xie}, \bibinfo{author}{D.~Meng},
  \bibinfo{author}{W.~Zuo}, \bibinfo{author}{X.~Feng},
  \bibinfo{author}{L.~Zhang},
\newblock \bibinfo{title}{Weighted nuclear norm minimization and its
  applications to low level vision},
\newblock \bibinfo{journal}{Int. J. Comput. Vis.} \bibinfo{volume}{121}
  (\bibinfo{year}{2017}) \bibinfo{pages}{183--208}.
  \DOIprefix\doi{10.1007/s11263-016-0930-5}.
%Type = Article
\bibitem[{Huang et~al.(2022)Huang, Li, Liu, Wu, and Zeng}]{huang2022quaternion}
\bibinfo{author}{C.~Huang}, \bibinfo{author}{Z.~Li}, \bibinfo{author}{Y.~Liu},
  \bibinfo{author}{T.~Wu}, \bibinfo{author}{T.~Zeng},
\newblock \bibinfo{title}{Quaternion-based weighted nuclear norm minimization
  for color image restoration},
\newblock \bibinfo{journal}{Pattern Recognit.} \bibinfo{volume}{128}
  (\bibinfo{year}{2022}) \bibinfo{pages}{108665}.
  \DOIprefix\doi{10.1016/j.patcog.2022.108665}.
%Type = Article
\bibitem[{Xie et~al.(2016)Xie, Gu, Liu, Zuo, Zhang, and
  Zhang}]{xie2016weighted}
\bibinfo{author}{Y.~Xie}, \bibinfo{author}{S.~Gu}, \bibinfo{author}{Y.~Liu},
  \bibinfo{author}{W.~Zuo}, \bibinfo{author}{W.~Zhang},
  \bibinfo{author}{L.~Zhang},
\newblock \bibinfo{title}{Weighted schatten $ p $-norm minimization for image
  denoising and background subtraction},
\newblock \bibinfo{journal}{IEEE Trans. Image Process.} \bibinfo{volume}{25}
  (\bibinfo{year}{2016}) \bibinfo{pages}{4842--4857}.
  \DOIprefix\doi{10.1109/TIP.2016.2599290}.
%Type = Article
\bibitem[{Zhang et~al.(2024)Zhang, He, Wang, Deng, and
  Liu}]{zhang2024quaternion}
\bibinfo{author}{Q.~Zhang}, \bibinfo{author}{L.~He}, \bibinfo{author}{Y.~Wang},
  \bibinfo{author}{L.-J. Deng}, \bibinfo{author}{J.~Liu},
\newblock \bibinfo{title}{Quaternion weighted schatten p-norm minimization for
  color image restoration with convergence guarantee},
\newblock \bibinfo{journal}{Signal Processing} \bibinfo{volume}{218}
  (\bibinfo{year}{2024}) \bibinfo{pages}{109382}.
  \DOIprefix\doi{10.1016/j.sigpro.2024.109382}.
%Type = Inproceedings
\bibitem[{Wang et~al.(2021)Wang, Yao, and Kwok}]{wang2021scalable}
\bibinfo{author}{Y.~Wang}, \bibinfo{author}{Q.~Yao}, \bibinfo{author}{J.~Kwok},
\newblock \bibinfo{title}{A scalable, adaptive and sound nonconvex regularizer
  for low-rank matrix learning},
\newblock in: \bibinfo{booktitle}{Proceedings of the Web Conference 2021},
  \bibinfo{year}{2021}, pp. \bibinfo{pages}{1798--1808}.
  \DOIprefix\doi{10.1145/3442381.3450142}.
%Type = Article
\bibitem[{Yu and Yang(2023)}]{yu2023low}
\bibinfo{author}{Q.~Yu}, \bibinfo{author}{M.~Yang},
\newblock \bibinfo{title}{Low-rank tensor recovery via non-convex
  regularization, structured factorization and spatio-temporal
  characteristics},
\newblock \bibinfo{journal}{Pattern Recognit.} \bibinfo{volume}{137}
  (\bibinfo{year}{2023}) \bibinfo{pages}{109343}.
  \DOIprefix\doi{10.1016/j.patcog.2023.109343}.
%Type = Article
\bibitem[{Yang et~al.(2022)Yang, Miao, and Kou}]{yang2022quaternion}
\bibinfo{author}{L.~Yang}, \bibinfo{author}{J.~Miao}, \bibinfo{author}{K.~I.
  Kou},
\newblock \bibinfo{title}{Quaternion-based color image completion via
  logarithmic approximation},
\newblock \bibinfo{journal}{Information Sciences} \bibinfo{volume}{588}
  (\bibinfo{year}{2022}) \bibinfo{pages}{82--105}.
  \DOIprefix\doi{10.1016/j.ins.2021.12.055}.
%Type = Article
\bibitem[{Lou and Yan(2018)}]{lou2018fast}
\bibinfo{author}{Y.~Lou}, \bibinfo{author}{M.~Yan},
\newblock \bibinfo{title}{Fast l1--l2 minimization via a proximal operator},
\newblock \bibinfo{journal}{J. Sci. Comput.} \bibinfo{volume}{74}
  (\bibinfo{year}{2018}) \bibinfo{pages}{767--785}.
  \DOIprefix\doi{10.1007/s10915-017-0463-2}.
%Type = Article
\bibitem[{Yang et~al.(2021)Yang, Kou, and Miao}]{yang2021weighted}
\bibinfo{author}{L.~Yang}, \bibinfo{author}{K.~I. Kou},
  \bibinfo{author}{J.~Miao},
\newblock \bibinfo{title}{Weighted truncated nuclear norm regularization for
  low-rank quaternion matrix completion},
\newblock \bibinfo{journal}{J. Vis. Commun. Image Represent.}
  \bibinfo{volume}{81} (\bibinfo{year}{2021}) \bibinfo{pages}{103335}.
  \DOIprefix\doi{10.1016/j.jvcir.2021.103335}.
%Type = Article
\bibitem[{Xu and Mandic(2015)}]{Xu2015Theory}
\bibinfo{author}{D.~Xu}, \bibinfo{author}{D.~P. Mandic},
\newblock \bibinfo{title}{The theory of quaternion matrix derivatives},
\newblock \bibinfo{journal}{IEEE Trans. Signal Process.} \bibinfo{volume}{63}
  (\bibinfo{year}{2015}) \bibinfo{pages}{1543--1556}.
  \DOIprefix\doi{10.1109/TSP.2015.2399865}.
%Type = Article
\bibitem[{Sangwine(1996)}]{sangwine1996fourier}
\bibinfo{author}{S.~J. Sangwine},
\newblock \bibinfo{title}{Fourier transforms of colour images using quaternion
  or hypercomplex, numbers},
\newblock \bibinfo{journal}{Electron. Lett.} \bibinfo{volume}{32}
  (\bibinfo{year}{1996}) \bibinfo{pages}{1979--1980}.
  \DOIprefix\doi{10.1049/el:19961331}.
%Type = Article
\bibitem[{Cand{\`e}s et~al.(2011)Cand{\`e}s, Li, Ma, and
  Wright}]{candes2011robust}
\bibinfo{author}{E.~J. Cand{\`e}s}, \bibinfo{author}{X.~Li},
  \bibinfo{author}{Y.~Ma}, \bibinfo{author}{J.~Wright},
\newblock \bibinfo{title}{Robust principal component analysis?},
\newblock \bibinfo{journal}{Journal of the ACM (JACM)} \bibinfo{volume}{58}
  (\bibinfo{year}{2011}) \bibinfo{pages}{1--37}.
  \DOIprefix\doi{10.1145/1970392.1970395}.
%Type = Article
\bibitem[{{Dubois}(2005)}]{Dubois2005Frequency}
\bibinfo{author}{E.~{Dubois}},
\newblock \bibinfo{title}{Frequency-domain methods for demosaicking of
  bayer-sampled color images},
\newblock \bibinfo{journal}{IEEE Signal Process. Lett.} \bibinfo{volume}{12}
  (\bibinfo{year}{2005}) \bibinfo{pages}{847--850}.
  \DOIprefix\doi{10.1109/LSP.2005.859503}.
%Type = Article
\bibitem[{Zhang et~al.(2011)Zhang, Wu, Buades, and Li}]{zhang2011color}
\bibinfo{author}{L.~Zhang}, \bibinfo{author}{X.~Wu},
  \bibinfo{author}{A.~Buades}, \bibinfo{author}{X.~Li},
\newblock \bibinfo{title}{Color demosaicking by local directional interpolation
  and nonlocal adaptive thresholding},
\newblock \bibinfo{journal}{J. Electron. Imaging} \bibinfo{volume}{20}
  (\bibinfo{year}{2011}) \bibinfo{pages}{023016}.
  \DOIprefix\doi{10.1117/1.3600632}.
%Type = Inproceedings
\bibitem[{Nam et~al.(2016)Nam, Hwang, Matsushita, and Kim}]{nam2016holistic}
\bibinfo{author}{S.~Nam}, \bibinfo{author}{Y.~Hwang},
  \bibinfo{author}{Y.~Matsushita}, \bibinfo{author}{S.~J. Kim},
\newblock \bibinfo{title}{A holistic approach to cross-channel image noise
  modeling and its application to image denoising},
\newblock in: \bibinfo{booktitle}{Proc. IEEE Conf. Comput. Vis. Pattern
  Recognit.}, \bibinfo{year}{2016}, pp. \bibinfo{pages}{1683--1691}.
  \DOIprefix\doi{10.1109/CVPR.2016.186}.
%Type = Article
\bibitem[{Xu et~al.(2018)Xu, Li, Liang, Zhang, and Zhang}]{xu2018real}
\bibinfo{author}{J.~Xu}, \bibinfo{author}{H.~Li}, \bibinfo{author}{Z.~Liang},
  \bibinfo{author}{D.~Zhang}, \bibinfo{author}{L.~Zhang},
\newblock \bibinfo{title}{Real-world noisy image denoising: A new benchmark},
\newblock \bibinfo{journal}{arXiv preprint arXiv:1804.02603}
  (\bibinfo{year}{2018}).
%Type = Inproceedings
\bibitem[{{Abdelhamed} et~al.(2018){Abdelhamed}, {Lin}, and
  {Brown}}]{Abdelhamed2018}
\bibinfo{author}{A.~{Abdelhamed}}, \bibinfo{author}{S.~{Lin}},
  \bibinfo{author}{M.~S. {Brown}},
\newblock \bibinfo{title}{A high-quality denoising dataset for smartphone
  cameras},
\newblock in: \bibinfo{booktitle}{Proc. IEEE/CVF Conf. Comput. Vis. Pattern
  Recognit.}, \bibinfo{year}{2018}, pp. \bibinfo{pages}{1692--1700}.
  \DOIprefix\doi{10.1109/CVPR.2018.00182}.
%Type = Inproceedings
\bibitem[{Chen et~al.(2015)Chen, Zhu, and Heng}]{Chen2015Efficient}
\bibinfo{author}{G.~Chen}, \bibinfo{author}{F.~Zhu}, \bibinfo{author}{P.~A.
  Heng},
\newblock \bibinfo{title}{An efficient statistical method for image noise level
  estimation},
\newblock in: \bibinfo{booktitle}{Proc. IEEE Int. Conf. Comput. Vis.},
  \bibinfo{year}{2015}, pp. \bibinfo{pages}{477--485}.
  \DOIprefix\doi{10.1109/ICCV.2015.62}.
%Type = Article
\bibitem[{Jia et~al.(2019)Jia, Ng, and Song}]{jia2019robust}
\bibinfo{author}{Z.~Jia}, \bibinfo{author}{M.~K. Ng}, \bibinfo{author}{G.-J.
  Song},
\newblock \bibinfo{title}{Robust quaternion matrix completion with applications
  to image inpainting},
\newblock \bibinfo{journal}{Numer. Linear Algebr. Appl.} \bibinfo{volume}{26}
  (\bibinfo{year}{2019}) \bibinfo{pages}{e2245}.
  \DOIprefix\doi{10.1002/nla.2245}.

\end{thebibliography}

\newpage
{\huge{Supplementary Material}}

\subsection*{B. Proof of Theorem 3.4}
\label{Appendix 3}
\begin{proof}
	First, we prove that the sequence $\{\dot{\mathbf{\eta}}^{(k)}\}$ generated by Algorithm 3 is upper bounded.
	\begin{equation}
		\begin{aligned}
			||\dot{\mathbf{\eta}}^{(k+1)} ||_{F}^{2} &= || \dot{\mathbf{\eta}}^{(k)} + \beta^{(k)}(\dot{\mathbf{Y}} - \dot{\mathbf{X}}^{(k+1)}-\dot{\mathbf{Z}}^{(k+1)} )  ||_{F}^{2} \\
			&= {\beta^{(k)}}^{2} || ({\beta^{(k)}}^{-1}\dot{\mathbf{\eta}}^{(k)} + \dot{\mathbf{Y}} - \dot{\mathbf{Z}}^{(k+1)}) -\dot{\mathbf{X}}^{(k+1)}  ||_{F}^{2} \\
			&= {\beta^{(k)}}^{2} ||\dot{\mathbf{U}}^{(k)}\Sigma^{(k)} \mathbf{V}^{(k)\top} -\dot{\mathbf{U}}^{(k)}\Sigma^{(k)}_{\dot{\mathbf{X}}} \mathbf{V}^{(k)\top} ||_{F}^{2} \\
			&= {\beta^{(k)}}^{2} \sum^{M}_{i=1}(\sigma^{(k)}_{i} - \sigma^{(k)}_{\dot{\mathbf{X}},i})^2 \\
			&\leq {\beta^{(k)}}^{2} \sum_{i=1}^{M} (\lambda/\beta^{(k)})^{2} \\
			&= \lambda^{2}M.
			\label{eta_3}
		\end{aligned}
	\end{equation}
	For the last inequality, if $\sigma^{(k)}_{i} \geq \lambda/\beta^{(k)}$, then $\sigma^{(k)}_{\dot{\mathbf{X}},i} = \frac{||x||_{2}+\alpha \lambda/\beta^{(k)}}{||x||_{2}}\max(\sigma^{(k)}_{i}-\lambda/\beta^{(k)}, 0)$.  Obviously $\sigma^{(k)}_{i} - \sigma^{(k)}_{\dot{\mathbf{X}},i} \leq \lambda/\beta^{(k)}$ holds. If $\sigma^{(k)}_{i} < \lambda/\beta^{(k)}$, $\sigma^{(k)}_{i} - \sigma^{(k)}_{\dot{\mathbf{X}},i} = \sigma^{(k)}_{i} < \lambda/\beta^{(k)}$. Hence, the sequence $\{\dot{\mathbf{\eta}}^{(k)}\}$ generated by Algorithm 3 is upper bounded.

	Then, we prove that the sequence of Lagrange function $\{ \mathcal{L}(\dot{\mathbf{X}}^{(k+1)}, \dot{\mathbf{Z}}^{(k+1)}, \dot{\mathbf{\eta}}^{(k+1)}, \beta^{(k+1)}) \}$ is also upper bounded. According to the update rule for $\dot{\mathbf{\eta}}^{(k)}$, it can be obtained 
	\begin{equation}
		\begin{aligned}
			&\mathcal{L}(\dot{\mathbf{X}}^{(k+1)}, \dot{\mathbf{Z}}^{(k+1)}, \dot{\mathbf{\eta}}^{(k+1)}, \beta^{(k+1)}) \\
			&~=\mathcal{L}(\dot{\mathbf{X}}^{(k+1)}, \dot{\mathbf{Z}}^{(k+1)}, \dot{\mathbf{\eta}}^{(k)}, \beta^{(k)}) + \langle{\dot{\mathbf{\eta}}^{(k+1)} - \dot{\mathbf{\eta}}^{(k)}, \dot{\mathbf{Y}} - \dot{\mathbf{X}}^{(k+1)} - \dot{\mathbf{Z}}^{(k+1)} }\rangle  +\frac{\beta^{(k+1)}-\beta^{(k)}}{2}||\dot{\mathbf{Y}} - \dot{\mathbf{X}}^{(k+1)} - \dot{\mathbf{Z}}^{(k+1)} ||_{F}^{2} \\
			&~=\mathcal{L}(\dot{\mathbf{X}}^{(k+1)}, \dot{\mathbf{Z}}^{(k+1)}, \dot{\mathbf{\eta}}^{(k)}, \beta^{(k)})  + \langle{\dot{\mathbf{\eta}}^{(k+1)} - \dot{\mathbf{\eta}}^{(k)}, \frac{\dot{\mathbf{\eta}}^{(k+1)} - \dot{\mathbf{\eta}}^{(k)}}{\beta^{(k)}} }\rangle +\frac{\beta^{(k+1)}-\beta^{(k)}}{2}||\frac{\dot{\mathbf{\eta}}^{(k+1)} - \dot{\mathbf{\eta}}^{(k)}}{\beta^{(k)}} ||_{F}^{2} \\
			&~= \mathcal{L}(\dot{\mathbf{X}}^{(k+1)}, \dot{\mathbf{Z}}^{(k+1)}, \dot{\mathbf{\eta}}^{(k)}, \beta^{(k)})  +\frac{\beta^{(k+1)} +\beta^{(k)}}{2{\beta^{(k)}}^2}||\dot{\mathbf{\eta}}^{(k+1)} - \dot{\mathbf{\eta}}^{(k)} ||_{F}^{2}. 
			\label{L_3}
		\end{aligned}
	\end{equation}
	Since $\{\dot{\mathbf{\eta}}^{(k)}\}$ is bounded, the sequence $\{\dot{\mathbf{\eta}}^{(k+1)} - \dot{\mathbf{\eta}}^{(k)}\}$ is also bounded. Suppose that the upper bound of the sequence $\{\dot{\mathbf{\eta}}^{(k+1)} - \dot{\mathbf{\eta}}^{(k)}\}$ is $M_1$, i.e., $ \forall k \geq 0, || \dot{\mathbf{\eta}}^{(k+1)} - \dot{\mathbf{\eta}}^{(k)} ||_F \leq M_1$. Meanwhile, the inequality $\mathcal{L}(\dot{\mathbf{X}}^{(k+1)}, \dot{\mathbf{Z}}^{(k+1)}, \dot{\mathbf{\eta}}^{(k)}, \beta^{(k)}) \leq \mathcal{L}(\dot{\mathbf{X}}^{(k)}, \dot{\mathbf{Z}}^{(k)}, \dot{\mathbf{\eta}}^{(k)}, \beta^{(k)})$ always holds because $\dot{\mathbf{X}}$ and $\dot{\mathbf{Z}}$ are globally optimal solutions to the corresponding subproblems. Therefore, we have
	\begin{equation*}
		\begin{aligned}
			&\mathcal{L}(\dot{\mathbf{X}}^{(k+1)}, \dot{\mathbf{Z}}^{(k+1)}, \dot{\mathbf{\eta}}^{(k+1)}, \beta^{(k+1)}) \\
			&~\leq \mathcal{L}(\dot{\mathbf{X}}^{(k+1)}, \dot{\mathbf{Z}}^{(k+1)}, \dot{\mathbf{\eta}}^{(k)}, \beta^{(k)}) + \frac{\beta^{(k+1)}+\beta^{(k)}}{2{\beta^{(k)}}^2}M_{1}^{2} \\
		\end{aligned}
	\end{equation*}
	\begin{equation}
		\begin{aligned}
			&~\leq \mathcal{L}(\dot{\mathbf{X}}^{(1)}, \dot{\mathbf{Z}}^{(1)}, \dot{\mathbf{\eta}}^{(0)}, \beta^{(0)}) + M_{1}^{2}\sum_{k=0}^{\infty}\frac{1+\mu}{2\beta^{(0)}\mu^{k}} \\
			&~\leq \mathcal{L}(\dot{\mathbf{X}}^{(1)}, \dot{\mathbf{Z}}^{(1)}, \dot{\mathbf{\eta}}^{(0)}, \beta^{(0)}) + \frac{M_{1}^{2}}{\beta^{(0)}}\sum_{k=0}^{\infty}\frac{1}{\mu^{k-1}} \\
			&~< -\infty.
		\end{aligned}
		\label{L2_3}
	\end{equation}
	Hence, $\{\mathcal{L}(\dot{\mathbf{X}}^{(k+1)}, \dot{\mathbf{Z}}^{(k+1)}, \dot{\mathbf{\eta}}^{(k+1)}, \beta^{(k+1)}) \}$ is upper bounded.

	Next, we prove that the sequences $\{\dot{\mathbf{X}}^{(k)}$\} and $\{\dot{\mathbf{Z}}^{(k)}\}$ are upper bounded. 
	\begin{equation}
		\begin{aligned}
			&\lambda(||\dot{\mathbf{X}}^{(k+1)}||_{*} - \alpha||\dot{\mathbf{X}}^{(k+1)}||_{F}) + \rho||\dot{\mathbf{Z}}^{(k+1)}||_{1} \\
			&=\mathcal{L}(\dot{\mathbf{X}}^{(k+1)}, \dot{\mathbf{Z}}^{(k+1)}, \dot{\mathbf{\eta}}^{(k)}, \beta^{(k)}) - \langle{\dot{\mathbf{\eta}}^{(k)}, \dot{\mathbf{Y}} - \dot{\mathbf{X}}^{(k+1)} - \dot{\mathbf{Z}}^{(k+1)}}\rangle - \frac{\beta^{(k)}}{2}|| \dot{\mathbf{Y}} - \dot{\mathbf{X}}^{(k+1)} - \dot{\mathbf{Z}}^{(k+1)} ||_{F}^{2} \\
			&=\mathcal{L}(\dot{\mathbf{X}}^{(k+1)}, \dot{\mathbf{Z}}^{(k+1)}, \dot{\mathbf{\eta}}^{(k)}, \beta^{(k)}) - \langle{\dot{\mathbf{\eta}}^{(k)}, \frac{\dot{\mathbf{\eta}}^{(k+1)} - \dot{\mathbf{\eta}}^{(k)}}{\beta^{(k)}} }\rangle - \frac{\beta^{(k)}}{2}|| \frac{\dot{\mathbf{\eta}}^{(k+1)} - \dot{\mathbf{\eta}}^{(k)}}{\beta^{(k)}} ||_{F}^{2} \\
			&= \mathcal{L}(\dot{\mathbf{X}}^{(k+1)}, \dot{\mathbf{Z}}^{(k+1)}, \dot{\mathbf{\eta}}^{(k)}, \beta^{(k)}) + \frac{||\dot{\mathbf{\eta}}^{(k)}||_{F}^{2} - ||\dot{\mathbf{\eta}}^{(k+1)}||_{F}^{2} }{2\beta^{(k)}}.
			\label{XZ_3}
		\end{aligned}
	\end{equation}
	From this, we get that $\{\dot{\mathbf{X}}^{(k)}\}$ and $\{ \dot{\mathbf{Z}}^{(k)} \}$ are upper bounded. Thus, there exists at least one accumulation point for $\{ \dot{\mathbf{X}}^{(k)}, \dot{\mathbf{Z}}^{(k)} \}$. Further, we can get
	\begin{equation}
		\begin{aligned}
			\lim_{k\rightarrow +\infty} || \dot{\mathbf{Y}} - \dot{\mathbf{X}}^{(k+1)}-\dot{\mathbf{Z}}^{(k+1)} ||_{F} &= \lim_{k\rightarrow +\infty} \frac{1}{\beta^{(k)}}|| \dot{\mathbf{\eta}}^{(k+1)}-\dot{\mathbf{\eta}}^{(k)} ||_{F} \\
			&=0,
		\end{aligned}
	\end{equation}
	and that the accumulation point is a feasible solution to the objective function. Thus, the first equation in (32) is proved.

	Finally, we demonstrate that the change in the sequences $\{\dot{\mathbf{X}}^{(k)}\}$ and $\{\dot{\mathbf{Z}}^{(k)}\}$ converges to 0 with iteration. For $\dot{\mathbf{Z}}$ subproblem, by $\dot{\mathbf{Z}}^{(k)} = \dot{\mathbf{Y}} - \dot{\mathbf{X}}^{(k)} + \frac{\dot{\mathbf{\eta}}^{(k-1)}-\dot{\mathbf{\eta}}^{(k)}}{\beta^{(k-1)}}$, we get
	\begin{equation}
		\begin{aligned}
			&\lim_{k\rightarrow +\infty} ||  \dot{\mathbf{Z}}^{(k+1)} - \dot{\mathbf{Z}}^{(k)} ||_{F} \\
			&~= \lim_{k\rightarrow +\infty} || \mathcal{S}_{\frac{\rho}{\beta^{(k)}}}(\dot{\mathbf{Y}} - \dot{\mathbf{X}}^{(k)}  + \frac{\dot{\mathbf{\eta}}^{(k)}}{\beta^{(k)}})  - (\dot{\mathbf{Y}} - \dot{\mathbf{X}}^{(k)}  + \frac{\dot{\mathbf{\eta}}^{(k)}}{\beta^{(k)}})  + (\frac{\dot{\mathbf{\eta}}^{(k)}}{\beta^{(k-1)}} - \frac{\dot{\mathbf{\eta}}^{(k-1)}}{\beta^{(k-1)}} + \frac{\dot{\mathbf{\eta}}^{(k)}}{\beta^{(k)}})||_{F} \\
			&~\leq \lim_{k\rightarrow +\infty} || \mathcal{S}_{\frac{\rho}{\beta^{(k)}}}(\dot{\mathbf{Y}} - \dot{\mathbf{X}}^{(k)}  + \frac{\dot{\mathbf{\eta}}^{(k)}}{\beta^{(k)}})  - (\dot{\mathbf{Y}} - \dot{\mathbf{X}}^{(k)}  + \frac{\dot{\mathbf{\eta}}^{(k)}}{\beta^{(k)}}) ||_{F}  + ||\frac{\dot{\mathbf{\eta}}^{(k)}}{\beta^{(k-1)}} - \frac{\dot{\mathbf{\eta}}^{(k-1)}}{\beta^{(k-1)}} + \frac{\dot{\mathbf{\eta}}^{(k)}}{\beta^{(k)}}||_{F} \\
			&~\leq \lim_{k\rightarrow +\infty} \frac{\rho \sqrt{mn}}{\beta^{(k)}} + ||\frac{(\mu+1)\dot{\mathbf{\eta}}^{(k)} - \mu \dot{\mathbf{\eta}}^{(k-1)}}{\beta^{(k)}}||_{F}\\
			&~=0,
		\end{aligned}
	\end{equation}
	where $\mathcal{S}_{\frac{\rho}{\beta^{(k)}}}$ is the soft-threshold operation with parameter $\frac{\rho}{\beta^{(k)}}$, $m$ and $n$ are the size of $\dot{\mathbf{Y}}$.
	
	Similarly, we prove that $\lim_{k\rightarrow +\infty} || \dot{\mathbf{X}}^{(k+1)}-\dot{\mathbf{X}}^{(k)} ||_{F} = 0$. By Algorithm 3, we get
	\begin{equation*}
		\dot{\mathbf{X}}^{(k+1)} = \dot{\mathbf{Y}} - \dot{\mathbf{Z}}^{(k+1)} + \frac{\dot{\mathbf{\eta}}^{(k)} - \dot{\mathbf{\eta}}^{(k+1)}}{\beta^{(k)}},
	\end{equation*}
	\begin{equation*}
		\dot{\mathbf{X}}^{(k)} = \dot{\mathbf{U}}^{(k-1)}\mathcal{S}^{\mathrm{QNMF}}_{\frac{\lambda}{\beta^{(k-1)}}}(\Sigma^{(k-1)})\dot{\mathbf{V}}^{(k-1)\top},
	\end{equation*}
	where $\dot{\mathbf{U}}^{(k-1)}\Sigma^{(k-1)}\dot{\mathbf{V}}^{(k-1)\top}$ is the QSVD of $\dot{\mathbf{Y}} -  \dot{\mathbf{Z}}^{(k)} + \frac{ \dot{\mathbf{\eta}^{(k-1)}}}{\beta^{(k-1)}}$, $\mathcal{S}^{\mathrm{QNMF}}_{\frac{\lambda}{\beta^{(k-1)}}}$  is the QNMF closed solution operator provided by Theorem III.1 and Eq. (8). Then, we have
	\begin{equation}
		\begin{aligned}
			&\lim_{k\rightarrow +\infty} ||  \dot{\mathbf{X}}^{(k+1)} - \dot{\mathbf{X}}^{(k)} ||_{F} \\
			&~= \lim_{k\rightarrow +\infty} || \dot{\mathbf{Y}} - \dot{\mathbf{Z}}^{(k+1)} + \frac{\dot{\mathbf{\eta}}^{(k)} - \dot{\mathbf{\eta}}^{(k+1)}}{\beta^{(k)}} - \dot{\mathbf{X}}^{(k)}||_{F}\\
			&~= \lim_{k\rightarrow +\infty} || \dot{\mathbf{Y}} - \dot{\mathbf{Z}}^{(k+1)} + \frac{\dot{\mathbf{\eta}}^{(k)} - \dot{\mathbf{\eta}}^{(k+1)}}{\beta^{(k)}} - \dot{\mathbf{X}}^{(k)} + (\dot{\mathbf{Z}}^{(k)} + \frac{\dot{\mathbf{\eta}}^{(k-1)}}{\beta^{(k-1)}}) - (\dot{\mathbf{Z}}^{(k)} + \frac{\dot{\mathbf{\eta}}^{(k-1)}}{\beta^{(k-1)}})||_{F}\\
			&~\leq \lim_{k\rightarrow +\infty} || (\dot{\mathbf{Y}} - \dot{\mathbf{Z}}^{(k)} +\frac{\dot{\mathbf{\eta}}^{(k-1)}}{\beta^{(k-1)}}) - \dot{\mathbf{X}}^{(k)} ||_{F} + || \dot{\mathbf{Z}}^{(k)} - \dot{\mathbf{Z}}^{(k+1)}||_{F} + || \frac{\dot{\mathbf{\eta}}^{(k)} - \dot{\mathbf{\eta}}^{(k+1)}}{\beta^{(k)}} - \frac{\dot{\mathbf{\eta}}^{(k-1)}}{\beta^{(k-1)}}||_{F}\\
			&~= \lim_{k\rightarrow +\infty} ||\Sigma^{(k-1)} - \mathcal{S}^{\mathrm{QNMF}}_{\frac{\lambda}{\beta^{(k-1)}}}(\Sigma^{(k-1)}) ||_{F} + || \dot{\mathbf{Z}}^{(k)} - \dot{\mathbf{Z}}^{(k+1)}||_{F} + || \frac{\dot{\mathbf{\eta}}^{(k)} - (\mu+1)\dot{\mathbf{\eta}}^{(k+1)}}{\beta^{(k)}} ||_{F} \\
			&~\leq \lim_{k\rightarrow +\infty}  \frac{\lambda \sqrt{mn}}{\beta^{(k-1)}} + || \dot{\mathbf{Z}}^{(k)} - \dot{\mathbf{Z}}^{(k+1)}||_{F} + || \frac{\dot{\mathbf{\eta}}^{(k)} - (\mu+1)\dot{\mathbf{\eta}}^{(k+1)}}{\beta^{(k)}} ||_{F} \\
			&~=0
		\end{aligned}
	\end{equation}
	
\end{proof}

\end{document}